\documentclass[10pt]{article}

\usepackage[letterpaper, left=1in, top=1in, right=1in, bottom=1in, verbose, ignoremp]{geometry}

\usepackage{url}\RequirePackage[colorlinks,citecolor=blue, linkcolor=blue,urlcolor = blue]{hyperref}
\usepackage{latexsym,amssymb,amsmath,amsfonts,graphicx,color,fancyvrb,amsthm,enumerate,mathrsfs}
\usepackage[authoryear,round]{natbib}
\usepackage[dvipsnames]{xcolor}
\usepackage{xy}\xyoption{all} \xyoption{poly} \xyoption{knot}
\usepackage{float}
\usepackage{bm}
\usepackage{bbm}
\usepackage{multicol,multirow}
\usepackage{array}
\usepackage{relsize}
\usepackage{chngcntr}
\usepackage{etoolbox}
\usepackage{caption,subcaption}
\usepackage{tabularx}
\usepackage{booktabs}
\usepackage{tikz}
\usepackage{tikz-cd}
\usetikzlibrary{patterns}
\usepackage{amsopn}
\usepackage{calc}
\usepackage{titletoc}
\usepackage{hyperref}
\usepackage{mathtools}

\newtheorem{theorem}{Theorem}[]
\newtheorem*{theorem*}{Theorem}
\newtheorem{corollary}[theorem]{Corollary}
\newtheorem{lemma}[theorem]{Lemma}
\newtheorem{proposition}[theorem]{Proposition}

\newtheorem*{claim*}{Claim}
\theoremstyle{definition}
\newtheorem{definition}[theorem]{Definition}
\newtheorem*{definition*}{Definition}

\theoremstyle{remark}
\newtheorem{remark}[theorem]{Remark}

\newtheorem*{example*}{Example}

\def\log{{\rm log}}

\newcommand\myeq{\stackrel{\mathclap{\normalfont\mbox{$t := r^p$}}}{=}}
\newcommand*\diff{\mathop{}\!\mathrm{d}}

\makeatletter
\newcommand*{\op}{%
  \DOTSB
  \mathop{\vphantom{\bigoplus}\mathpalette\matt@op\relax}%
  \slimits@
}
\newcommand\matt@op[2]{%
  \vcenter{\m@th\hbox{\resizebox{\widthof{$#1\bigoplus$}}{!}{$\boxplus$}}}%
}
\makeatother

\newcommand{\ot}{\mathrm{OT}}

\makeatletter
\def\@biblabel#1{}
\makeatother

\makeatletter
\patchcmd{\NAT@citex}
  {\@citea\NAT@hyper@{%
     \NAT@nmfmt{\NAT@nm}%
     \hyper@natlinkbreak{\NAT@aysep\NAT@spacechar}{\@citeb\@extra@b@citeb}%
     \NAT@date}}
  {\@citea\NAT@nmfmt{\NAT@nm}%
   \NAT@aysep\NAT@spacechar\NAT@hyper@{\NAT@date}}{}{}

\patchcmd{\NAT@citex}
  {\@citea\NAT@hyper@{%
     \NAT@nmfmt{\NAT@nm}%
     \hyper@natlinkbreak{\NAT@spacechar\NAT@@open\if*#1*\else#1\NAT@spacechar\fi}%
       {\@citeb\@extra@b@citeb}%
     \NAT@date}}
  {\@citea\NAT@nmfmt{\NAT@nm}%
   \NAT@spacechar\NAT@@open\if*#1*\else#1\NAT@spacechar\fi\NAT@hyper@{\NAT@date}}
  {}{}

\makeatother

\begin{document}
\def\spacingset#1{\renewcommand{\baselinestretch}%
{#1}\small\normalsize} \spacingset{1}
\begin{flushleft}
{\Large{\textbf{Approximating Persistent Homology for Large Datasets}}}
\newline
\\
Yueqi Cao$^{1,\dagger}$ and Anthea Monod$^{2,\dagger}$
\\
\bigskip
\bf{1} Department of Mathematics, KTH Royal Institute of Technology, Sweden\\
\bf{2} Department of Mathematics, Imperial College London, UK
\\
\bigskip
$\dagger$ Corresponding e-mail: yueqic@kth.se, a.monod@imperial.ac.uk
\end{flushleft}


\section*{Abstract}
Persistent homology is an important methodology in topological data analysis which adapts theory from algebraic topology to data settings. Computing persistent homology produces persistence diagrams, which have been successfully used in diverse domains.  Despite its widespread use, persistent homology is simply impossible to compute when a dataset is very large.  We study a statistical approach to the problem of computing persistent homology for massive datasets using a multiple subsampling framework and extend it to three summaries of persistent homology: H\"{o}lder continuous vectorizations of persistence diagrams; the alternative representation as persistence measures; and standard persistence diagrams.  
Specifically, we derive finite sample convergence rates for empirical means for persistent homology and practical guidance on interpreting and tuning parameters.  We validate our approach through extensive experiments on both synthetic and real-world data. We demonstrate the performance of multiple subsampling in a permutation test to analyze the topological structure of Poincar\'{e} embeddings of large lexical databases.

\paragraph{Keywords:} Fréchet means; persistence measures; multiple subsampling; Wasserstein stability.



\section{Introduction}
\label{sec:intro}

Topological data analysis (TDA) is a recently emerged field that harnesses theory from algebraic topology to address modern data challenges, such as high dimensionality and structural complexity in a wide variety of contexts.  A particularly successful TDA tool---\emph{persistent homology}---adapts the theory of homology to data to produce interpretable summaries of the dataset capturing its ``shape'' and ``size.''  Computing persistent homology results in a {\em persistence diagram}, which summarizes the topological information within the data.  Persistent homology has been extensively implemented in various applications, including biomedical imaging \citep{SECT}; information retrieval and machine learning \citep{vlontzos2021topological}; materials science \citep{hirata2020structural}; neuroscience \citep{anderson2018topological}; sensor networks \citep{adams2015evasion}; and many others.  

A significant challenge of applying persistent homology to data, however, is its computational expense, which is known to be intensive; see \cite{otter2017roadmap} for a detailed discussion on computing persistent homology.  Existing approaches and recent advances have greatly increased the burden with geometric solutions being a popular workaround \citep{sheehy2012linear, graf2025flood}.  This paper focuses on the power of \emph{statistical} approaches for approximating the persistent homology of a dataset when exact computation is intractable.  

The technique of multiple subsampling for persistent homology was first studied by \citet{chazal-2015-subsampling} and focused on \emph{persistence landscapes}, which are the earliest vector representations of persistence diagrams \citep{bubenik2015statistical}.  The main idea is to randomly sample multiple smaller subsets from a large dataset and compute the average persistence landscape from these subsets as an approximation of the persistence landscape of the entire dataset.  However, persistence landscapes are only one of many vector representations of persistence diagrams, and vectorizations in general may not be the most desirable representation for all applications. For example, one limitation is the lack of interpretability of vectorized persistence diagrams: although vectorization methods provide the benefit of representing the topological information of datasets in a usable vector form for classical statistical theory as well as current machine learning algorithms, it is often difficult to extract the intuition that persistence diagrams themselves carry from their vector representation.  Similarly, it is difficult to ascertain the global topological behavior of a larger dataset from a collection of smaller persistence diagrams when working with complex data.

\subsection{Related Work}

\citet{chazal-2015-subsampling} have studied multiple subsampling for persistence landscapes and concluded that the empirical average landscape computed from subsamples approximates the true mean landscape.  More recently, \cite{solomon2021geometry} propose the notion of {\em distributed persistence} to describe the topology of a dataset, which relies on subsampling.  Distributed persistence produces a collection of persistence diagrams of smaller subsets, rather than a single one computed on the large dataset; this collection has been shown to be stable to outliers and possess desirable inverse properties.  

Though vectorizations of persistence diagrams can be useful in real-world applications of TDA, especially in machine learning and deep neural networks, it is often desirable to study persistence diagrams rather than their vectorizations for reasons of interpretability and information retention. 
Persistence diagrams encode the geometry of the underlying spaces: For shape data, \cite{song2023generalized} showed that the points in persistence diagrams correspond to different types of critical points of the distance functions. Such information might be lost via vectorizations: as shown by \cite{bubenik-2020-embeddings}, there does not exist isometric embeddings from the space of persistence diagrams to any Hilbert space.

Other related work by \cite{reani2021cycle} adapts a subsampling approach for topological inference, where the goal is to distinguish topological signal from noise in point cloud data.  \cite{10.2307/26542470} also study the asymptotic behavior of persistent homology and prove a strong law of large numbers for persistence diagrams, however, not in the context of subsampling.
   
\subsection{Contributions}
In this paper, we extend the multiple subsampling framework to three different summaries of persistent homology: a general class of H\"{o}lder continuous vectorizations (HCVs) of persistence diagrams; persistence measures; and persistence diagrams themselves.

\begin{itemize}
    \item We introduce the class of HCVs of persistence diagrams, which consists of many existing and popular vectorizations of persistence diagrams.  We prove convergence rates for
    HCVs which generalizes results in \citet{chazal-2015-subsampling}. In particular, we use a new method to control the variance term using the Law of the Iterated Logarithm (LIL) for Banach space-valued random variables.
    \item We generalize the multiple subsampling method to persistence measures and persistence diagrams, for which we
    take the mean persistence measure and the Fréchet mean as the corresponding quantity of centrality. We prove convergence rates for the mean persistence measure for finite and continuous underlying spaces. 
    \item We provide practical guidance on interpreting and tuning parameters and validate our approach through extensive experiments on both synthetic and real-world data.  In applications, we use multiple subsampling and permutation testing to analyze the topological structure of Poincar\'{e} embeddings of large lexical databases. We also apply our method to shape topology inference and shape clustering.
\end{itemize}

\subsection{Outline} The remainder of this paper is structured as follows.  Section \ref{sec:ph_rep} provides an overview of persistent homology and its representations by persistence diagrams, persistence measures, and vectorized persistence diagrams.  Section \ref{sec:subsampling} focuses on theoretical derivations of multiple subsampling method for different representations of persistent homology.  In Section \ref{sec:prac-para-turning}, we present practical guidance on interpreting and tuning parameters. Section \ref{sec:experiments} presents simulations to verify theoretical results and Section \ref{sec:real-data} presents applications to Poincar\'e embeddings of lexical databases.  Finally, we close the paper with a discussion in Section \ref{sec:disussion}.


\section{Persistent Homology and its Representations}
\label{sec:ph_rep}

In this section, we provide background on different representations on the output of persistent homology.

\subsection{Persistent Homology}

Persistent homology adapts classical homology from algebraic topology to a dynamical setting, making homology applicable to real-world data \citep{frosini1999size,edelsbrunner2000topological,zomorodian2005computing}.   A detailed discussion on persistent homology can be found in the supplementary materials. 

\subsection{Persistence Diagrams}
\label{sec:pd}

A canonical representation of the output of persistent homology is the \emph{persistence diagram}. A persistence diagram $D$ is a locally finite multiset of points in the half-plane $\Omega = \{(x,y)\in\mathbb{R}^2 \mid x<y\}$ together with points on the diagonal $\partial\Omega=\{(x,x)\in\mathbb{R}^2\}$ counted with infinite multiplicity. Points in $\Omega$ are called {\em off-diagonal points}. The persistence diagram with no off-diagonal points is called the {\em empty persistence diagram}, denoted by $D_\emptyset$. In this paper, we will use $D[\mathcal{X}]$ to denote the persistence diagram of the \v{C}ech or Vietoris--Rips (VR) filtration of $\mathcal{X}$, which are two of the most commonly used filtrations, especially in applications; we will use $D[f]$ to denote the persistence diagram of the (sub)level set filtration of $f$.

The collection of all persistence diagrams constitutes a well-defined metric space with properties amenable to statistical and probabilistic analysis.
Let $\|\cdot\|_q$ denote the $q$-norm on $\mathbb{R}^2$ for $1\le q \le\infty$. Let $D_1,D_2$ be any two persistence diagrams. For $1\le p<\infty$, the {\em $p$-Wasserstein distance} is 
\begin{equation}\label{eq:p-wasserstein-distance}
    \mathrm{W}_{p,q}(D_1,D_2) = \inf_{\gamma}\left(\sum_{x\in D_1}\|x-\gamma(x)\|_q^p\right)^{\frac{1}{p}}
\end{equation}
where $\gamma$ ranges over all bijections between $D_1$ and $D_2$. For $p=\infty$, \eqref{eq:p-wasserstein-distance} becomes the {\em bottleneck distance},
\begin{equation}
    \mathrm{W}_{\infty,q}(D_1,D_2) = \inf_{\gamma}\sup_{x\in D_1}\|x-\gamma(x)\|_q.
\end{equation}
The {\em $p$-total persistence} of $D$ is defined as $\mathrm{W}_{p,q}(D,D_\emptyset)$; the space of all persistence diagrams with finite $p$-total persistence is denoted by $\mathcal{D}_p$.

The metric space $(\mathcal{D}_p,\mathrm{W}_{p,q})$ is central in TDA. For $q=\infty$, $(\mathcal{D}_p,\mathrm{W}_{p,\infty})$ is a complete and separable metric space, i.e., a Polish space for any $1\le p<\infty$, which means that statistical and probabilistic quantities such as probability measures, expectations, and variances are well-defined on $(\mathcal{D}_p,\mathrm{W}_{p,q})$ \citep{mileyko-2011-probability}. Note that since all $q$-norms are equivalent on $\mathbb{R}^2$, $(\mathcal{D}_p,\mathrm{W}_{p,q})$ is Polish regardless of the choice of $1\le q\le \infty$. However, for $p=\infty$, $(\mathcal{D}_\infty,\mathrm{W}_{\infty,q})$ is not Polish as it is not separable \citep{divol-2019-understanding}. Furthermore, the following geometric characterizations under Wasserstein distances are known for the space of persistence diagrams: For $\|\cdot\|_2$ the $2$-norm on $\mathbb{R}^2$, the space $(\mathcal{D}_2,\mathrm{W}_{2,2})$ is a non-negatively curved Alexandrov space \citep{turner-2014-frechet}. Whenever $p\neq 2$, the space $(\mathcal{D}_p,\mathrm{W}_{p,p})$ fails to be non-negatively curved \citep{turner2013means}. In this paper, unless otherwise stated, we always set $p=q$ and use the notation $\mathrm{W}_p$ (rather than $\mathrm{W}_{p,p}$) for simplicity.

\subsection{Persistence Measures}
\label{sec:pm}

An alternative, equivalent representation for the output of persistent homology is to define persistence diagrams as measures on $\Omega$ of the form $\sum_{x\in D\cap\Omega}n_x\delta_x$ where $x$ ranges all off-diagonal points in a persistence diagram $D$, $n_x$ is the multiplicity of $x$, and $\delta_x$ is the Dirac measure at $x$ \citep{divol2019density}. Motivated by this measure-based perspective, \cite{divol-2019-understanding} considered all Radon measures with finite $p$-total persistence supported on $\Omega$. The space of all persistence measures is denoted by $\mathcal{M}_p$.  As with the case of persistence diagrams, the collection of all persistence measures is also a well-defined metric space under the {\em optimal partial transport distance}.
\begin{equation}
\label{eq:OT_dist}
    \mathrm{OT}_{p,q}(\mu,\nu) = \inf_{\pi}\Bigg(\int_{\overline{\Omega}\times\overline{\Omega}}\|x-y\|_q^p\diff{\Pi(x,y)}\Bigg)^{1/p},
\end{equation}
where $\overline{\Omega}=\Omega\cup\partial\Omega$, and $\Pi$ ranges all Radon measures on $\overline{\Omega}\times\overline{\Omega}$ such that for all Borel sets $P,Q\subseteq\Omega$,
    $\Pi(P\times\overline{\Omega})=\mu(P)$ and $\Pi(\overline{\Omega}\times Q) = \nu(Q).$
When $p=q$, we write $\ot_p$ for simplicity of notation. 


$(\mathcal{M}_p,\mathrm{OT}_{p,q})$ is also a viable space for statistics and probability since it is also a Polish space for all $1\le p<\infty$ and $1\le q\le \infty$ \citep{divol-2019-understanding}. Furthermore, $\mathcal{D}_p$ is a closed subspace of $\mathcal{M}_p$ and $\mathrm{OT}_{p,q}$ coincides with $\mathrm{W}_{p,q}$ on $\mathcal{D}_p$. These results show that persistence measures together with the optimal partial transport distance are proper generalizations of persistence diagrams and $p$-Wasserstein distance.

\subsection{Vectorized Persistence Diagrams}

Due to the nonlinear nature of persistence diagrams, it is difficult to directly apply classical statistical methodology. This challenge motivated an active area of research on vectorizations of persistence diagrams, which are essentially embeddings into Euclidean or Hilbert spaces, which can then be directly integrated into statistical and machine learning pipelines.

Popular examples of vectorizations for persistence diagrams include persistence landscapes \citep{bubenik2015statistical}, which are linear, functional vectorizations of persistence diagrams into Hilbert space, and persistence images \citep{adams-2017-persistence}.  There also exist machine learning-motivated kernel-based embeddings such as the persistence scale-space kernel \citep{reininghaus2015stable}, the persistence weighted Gaussian kernel \citep{kusano2016persistence}, and the sliced Wasserstein kernel for persistence diagrams \citep{carriere-2017-sliced}. Each of these constructions maps diagrams into linear spaces, typically 
$\mathbb{R}^n$ or function spaces, where natural metrics such as Euclidean or 
$L^p$ norms can be used to quantify distances between vectorized diagrams. 


We introduce a class of \emph{H\"{o}lder continuous vectorizations (HCVs)} of persistence diagrams, which incorporate many existing vectorizations. 

\begin{definition*}
\label{def:holder_cont}
    Let $\mathcal{V}=\mathcal{C}(\mathcal{Y})$ be the space of continuous functions over a compact metric space $(\mathcal{Y},d_{\mathcal{Y}})$. A \emph{H\"{o}lder continuous vectorization (HCV)} of a persistence diagram is a map $\Phi: (\mathcal{D},\mathrm{W}_\infty)\to(\mathcal{V},\|\cdot\|_\infty)$ such that
    \begin{itemize}
        \item[(A1)] For each $D\in\mathcal{D}$, the function $\Phi(D)\in\mathcal{V}$ is H\"{o}lder continuous, i.e., for $y,y'\in \mathcal{Y}$,
        $$
        |\Phi(D)(y)-\Phi(D)(y')|\le C_{\mathcal{Y}} \cdot d^{\alpha}_{\mathcal{Y}}(y,y')
        $$
        where $C_{\mathcal{Y}}>0$ and $0<\alpha\le 1$ are constants only depending on $\mathcal{Y}$;
        \item[(A2)] The map $\Phi$ itself is H\"{o}lder continuous, i.e., for $D,D'\in\mathcal{D}$, 
    $$
    \|\Phi(D)-\Phi(D')\|_{\infty}\le C_{\Phi} \cdot \mathrm{W}_\infty^\beta(D,D') ,
    $$
    where $C_{\Phi}>0$ and $0<\beta\le 1$ are constants only depending on $\Phi$.
    \end{itemize} 
\end{definition*} 

In summary, we require H\"older continuity for both the function $\Phi(D)$ on $\mathcal{Y}$ and the vectorization map $\Phi$ on $\mathcal{D}$. As will be shown further on  in Section \ref{sec:multiple-subsampling-holder}, the H\"older exponent $\alpha$ in (A1) controls the entropy of $\mathcal{V}$, which essentially allows us to use the Law of the Iterated Logarithms (LIL) in Banach spaces, and the H\"older exponent $\beta$ in (A2) controls the bias rate of multiple subsampling. When $\alpha=\beta=1$, such a class corresponds to \emph{Lipschitz continuous vectorizations (LCVs)} of persistance diagrams. A typical example is the persistence landscape.  We verify the conditions of HCVs for a comprehensive list of vectorizations in supplementary materials.



\section{Multiple Subsampling for Persistent Homology}
\label{sec:subsampling}

Let $\mathcal{X}$ be a compact metric space with a probability distribution $\pi$. To estimate the persistent homology of $\mathcal{X}$, \citet{chazal-2015-subsampling} initially studied the following multiple subsampling method for persistent homology: First, sample $B$ subsets $S_n^{(1)},\ldots,S_n^{(B)}$, each consisting of $n$ i.i.d.~samples from $\pi$. The persistent homology of $\mathcal{X}$ may then be approximated using the ``average'' persistent homology of $S_n^{(j)},j=1,\ldots,B$. Due to the intrinsic nonlinearity of persistence diagrams, it is technically much easier to compute averages of linear representations of persistence diagrams rather than persistence diagrams themselves. This is the approach taken by \citet{chazal-2015-subsampling}, where they studied the convergence of the average of persistence landscapes. To prove consistency, they impose the following technical assumption on $\pi$, which we also adopt in our work: The measure $\pi$ is said the satisfy the \emph{$(a,b,r_0)$-standard assumption} if there exists $a,b>0$ and $r_0\ge 0$ such that for any $x\in\mathcal{X}$ and $r>r_0$,
\begin{equation}
    \pi(\mathcal{B}(x,r))\ge \min(1,ar^b),
\end{equation}
where $\mathcal{B}(x,r)$ is the metric ball centered at $x$ with radius $r$. For $r_0=0$, the expression reduces to the \emph{$(a,b)$-standard assumption}. The parameter $r_0$ characterizes the ``denseness'' of the underlying space $\mathcal{X}$. When $r_0>0$ is small, $\mathcal{X}$ is often interpreted as a large, dense, discrete dataset that each point has a neighbor within distance $r_0$; when $r_0=0$, $\mathcal{X}$ is no longer discrete dataset and often interpreted as a compact, continuous space. We explain the insights of this assumption and overview the main convergence results of \citet{chazal-2015-subsampling} in supplementary materials. 





\subsection{Multiple Subsampling for HCVs}\label{sec:multiple-subsampling-holder}

For a sample $S_n^{(1)},\ldots,S_n^{(B)}$, let $D[S_n^{(i)}]$ be the persistence diagram from the VR or \v{C}ech filtration of $S_n^{(i)}$ and let $\phi^{(i)} = \Phi(D[S_n^{(i)}])$ be its HCV. Let $\bm{\phi}$ be the $\mathcal{V}$-valued random variable representing the population of $\{\phi^{(i)}\},\, i=1, \ldots, B$. The empirical mean is given by $\bar{\phi}= \frac{1}{B}(\phi^{(1)}+\cdots+\phi^{(B)})$. Let $\phi_\mathcal{X} = \Phi(D[\mathcal{X}])$ be the HCV of the true persistence diagram of $\mathcal{X}$. By the triangle inequality, we have the following variance--bias decomposition
\begin{equation}
\label{eq:triangle_holder}
 \mathbb{E}[\|\bar{\phi}-\phi_\mathcal{X}\|_\infty]\le \underbrace{\mathbb{E}[\|\bar{\phi}-\mathbb{E}[\bm{\phi}]\|_\infty]}_{\text{variance}}+\underbrace{\|\mathbb{E}[\bm{\phi}]-\phi_\mathcal{X}\|_\infty}_{\text{bias}}.
\end{equation}



\paragraph{Variance.}  
We outline our strategy to derive our theorem on the convergence rate for the variance component.  The derivation for a rate estimation of the variance term of \eqref{eq:triangle_holder} relies on the  CLT and the LIL for Banach spaces.

The difficulty in computing a rate estimation for the variance lies in establishing moment and entropy conditions on random variables in Banach spaces so that the CLT holds.  It turns out that certain moment and entropy conditions hold on a subspace of Lipschitz functions over metric spaces $\mathcal{Y}$ with finite entropy, defined to be the log of the covering number (i.e., the fewest balls of radius $r$ needed to cover $\mathcal{Y}$), $\mathrm{Ent}(\mathcal{Y},r) := \log(\mathrm{cv}(\mathcal{Y}, r))$.  Our first requirement towards a variance rate estimation is that the assumption (A1) implies finiteness of the moments of HCVs so that the CLT and LIL hold.

\begin{proposition}\label{prop:moment-control}
    Suppose $\Phi$ is HCV. Assume for some $y_0\in\mathcal{Y}$, $\mathbb{E}[\bm{\phi}(y_0)^2]<\infty$, then $\mathbb{E}[\|\bm{\phi}\|^2_\infty]<\infty$.
\end{proposition}

Given this fact, we can then derive the variance rate estimation in the following result. Let $\mathrm{Log}(x) := \max(1, \log x)$ and $\mathrm{LLog}(x) := \mathrm{Log}(\mathrm{Log}(x))$.

\begin{lemma}\label{lemma:var-rate-vectorization}
    Suppose $\Phi$ is HCV with exponents $\alpha$ and $\beta$, and for some $y_0\in\mathcal{Y}$, $\mathbb{E}[\bm{\phi}(y_0)^2]<\infty$. Suppose the metric entropy function of $\mathcal{Y}$ is such that
    $\displaystyle \int_0^{\infty}\mathrm{Ent}^{\frac{1}{2}}(\mathcal{Y},r^{\frac{1}{\alpha}})<\infty$.
    Then with probability one, there exists an integer $B_0>0$ and a constant $C_0>0$ such that for all $B>B_0$,
    $$
    \mathbb{E}[\|\bar{\phi}-\mathbb{E}[\bm{\phi}]\|_\infty]\le C_0\frac{\sqrt{2\mathrm{LLog}(B)}}{\sqrt{B}}.
    $$
\end{lemma}

Note that the assumption of finiteness of the metric entropy in Lemma \ref{lemma:var-rate-vectorization} is not restrictive: in most cases, the point clouds or the underlying space are compact, thus the induced distributions of HCVs have compact supports, which always have finite metric entropy. 

\paragraph{Bias.}  
We now turn to studying the bias term of \eqref{eq:triangle_holder} and derive its rate estimation.  Essentially, it involves upper bounding the bias term in \eqref{eq:triangle_holder} by an estimate of the $\beta$-th moment of the Hausdorff distance $\mathrm{H}_\infty^\beta(\bm{S}_n,\mathcal{X})$ and studying the convergence of this estimate.  The expression for the estimate is given in the following result.

\begin{lemma}\label{thm:moment-hausdorff-estimation}
    Let $\bm{S}_n$ be a random set of size $n$ from a measure $\pi$ on $\mathcal{X}$ that satisfies the $(a,b,r_0)$-standard assumption. Let $\displaystyle r_n:=2\bigg(\frac{\log n}{an}\bigg)^{1/b}$. Then
    $$
    \mathbb{E}[\mathrm{H}_\infty^\alpha(\bm{S}_n,\mathcal{X})]\le r_0^\beta + r_n^\beta \bm{1}_{\{r_n>r_0\}} + C_1\frac{r_n^\beta}{(\log n)^2}, 
    $$
    where $C_1$ is a constant depending on $a$ and $b$.
\end{lemma}

 The bias rate for HCVs is summarized by the following result.

 \begin{lemma}\label{lemma:vec-bias-rate}
     Suppose $\mathcal{X}$ is a metric space and $\pi$ is a probability measure satisfying the $(a,b,r_0)$-standard assumption. Let $\displaystyle r_n=2\bigg(\frac{\log n}{an}\bigg)^{1/b}$. Then
     $$
     \|\mathbb{E}[\bm{\phi}]-\phi_\mathcal{X}\|_\infty\le C_{\Phi}\left(r_0^\beta + r_n^\beta \bm{1}_{\{r_n>r_0\}} + C_1\frac{r_n^\beta}{(\log n)^2}\right) .
     $$
 \end{lemma}

\paragraph{Convergence Rate.}
Combining Lemma \ref{lemma:var-rate-vectorization} and Lemma \ref{lemma:vec-bias-rate}, we summarize the nonasymptotic convergence rate as follows.

\begin{theorem}\label{thm:vectorization-rate}
    Suppose $\mathcal{X}$ is a metric space and $\pi$ is a probability measure satisfying the $(a,b,r_0)$-standard assumption. Suppose $\Phi$ is HCV with exponents $\alpha$ and $\beta$, and for some $y_0\in\mathcal{Y}$, $\mathbb{E}[\bm{\phi}(y_0)^2]<\infty$. Suppose the metric entropy function of $\mathcal{Y}$ is such that
    $\int_0^{\infty}\mathrm{Ent}^{\frac{1}{2}}(\mathcal{Y},r^{\frac{1}{\alpha}})<\infty$.
    Let $\displaystyle r_n:=2\bigg(\frac{\log n}{an}\bigg)^{1/b}$. Then with probability one, there exists an integer $B_0>0$ and constants $C_0,C_1>0$ such that for all $B>B_0$,
    $$
    \mathbb{E}[\|\bar{\phi}-\phi_{\mathcal{X}}\|_\infty]\le  C_0\frac{\sqrt{2\mathrm{LLog}(B)}}{\sqrt{B}}+ C_{\Phi}\left(r_0^\beta + r_n^\beta \bm{1}_{\{r_n>r_0\}} + C_1\frac{r_n^\beta}{(\log n)^2}\right) .
    $$
\end{theorem}

Compared to the results in \citet{chazal-2015-subsampling}, our bias estimation has an additional parameter $0<\beta\le 1$, characterizing the H\"older regularity of the vectorization map $\Phi$. For persistence landscapes where $\beta=1$, our bias estimation is consistent with \citet{chazal-2015-subsampling} while our variance estimation is new, and is stated in the sense of strong convergence with respect to the supremum norm. 




\subsection{Multiple Subsampling for Persistence Measures}\label{sec:multiple-subsampling-pm}

In the same statistical setting, let $S_n^{(1)},\ldots,S_n^{(B)}$ be realizations of $\bm{S}_n$ and let $D[S_n^{(i)}]$ be the persistence diagram for each $S_n^{(i)}$. In the space of persistence measures, we consider the mean persistence measure
    \begin{equation}
    \label{eq:mpm}
    \bar{D} = \frac{1}{B}\sum_{i=1}^B D[S_n^{(i)}].
    \end{equation}
Recall from Section \ref{sec:pm} that the natural metric on the space of persistence measures is the optimal partial transport distance $\ot_p$. Let $D[\mathcal{X}]$ denote the true persistence diagram of $\mathcal{X}$,
and $D_\mu$ denote the population mean persistence measure. 
Then the approximation error $\mathbb{E}[\ot^p_p(\bar D, D[\mathcal{X}])]$ has the following variance--bias  decomposition:
\begin{align}
\mathbb{E}[\ot_p^p(\bar D, D[\mathcal{X}])] & \leq \mathbb{E}\big[ \big( \ot_p(\bar D, D_\mu) + \ot_p(D_\mu, D[\mathcal{X}]) \big)^p \big] \nonumber\\
& \le 2^{p-1} \big(\underbrace{\mathbb{E}\big[\ot_p^p(\bar{D}, D_\mu)\big]}_{\text{variance}}+\underbrace{\ot_p^p( D_\mu, D[\mathcal{X}])}_{\text{bias}}\big). \label{eq:bias_var}
\end{align}

The critical reason to use $\ot_p^p$ instead of $\ot_p$ is that the variance estimation only holds for $\ot_p^p$. Therefore we are restricted to the regime of finite Wasserstein index $p$. 



\paragraph{Variance.}
\cite{divol-2021-estimation} established the following result on the convergence rate for the variance of the mean persistence measure.

\begin{theorem}{\cite[Theorem 1]{divol-2021-estimation}}\label{thm:variance-original}
Let $1 \leq p < \infty$ and $0 \leq k < p$.  Let $\pi$ be a probability distribution supported on $\mathcal{M}_{k,L}$ and $\nu_1, \ldots, \nu_B$ be i.i.d.~samples drawn from $\pi$.  Let $\bar \nu = \frac{1}{B}\sum_{i=1}^B \nu_i$ be the mean persistence measure and $\mathbb{E}_\pi[\nu]$ be the expected persistence diagram.  Then
$$
\mathbb{E}\big[\ot_p^p (\bar\nu,\, \mathbb{E}_\pi[\nu] ) \big] \leq C_{p,k,L}\bigg( \frac{1}{\sqrt{B}} + \frac{a_p(B)}{B^{p-k}} \bigg) ,
$$
where $C_{p,k,L}$ depends only on $p,k,L$ and $a_p(B) = 1$ if $p>1$ and $a_p(B) = \log(B)$ if $p=1$.
\end{theorem}

\paragraph{Bias.} To estimate the bias, a key step is to use \emph{stability theorems} to bound the distance between persistence measures. Stability refers to the property that a bounded perturbation of input data will result in a bounded perturbation of the statistical summary computed from them---in this case, representations of persistent homology. For finite $p$, the Wasserstein stability theorems are not unified and depend on specific constructions of filtrations. We split our discussions for the following two cases: (i) the VR filtration, where we we assume $\mathcal{X}\subseteq \mathbb{R}^m$ is a finite set of cardinality $N$; and (ii) the \v{C}ech filtration, where we assume $\mathcal{X}\subseteq \mathbb{R}^m$ is any compact set.

\subsubsection{Case I: VR Filtration, $\mathcal{X}$ Finite}

Assume $\pi$ is a probability measure supported on $\mathcal{X}$ satisfying the $(a,b,r_0)$-standard assumption. Let $\bm{S}_n$ be the random variable with distribution $\pi^{\otimes n}$. The persistence diagram of VR filtration induces the random variable $\bm{D}_n$ with pushforward distribution $(\pi^{\otimes n})_*$. 

Let $\mathcal{A},\mathcal{B}\subseteq\mathbb{R}^m$ be two finite sets in Euclidean space. For $1\le p<\infty$, the $p$-Hausdorff distance between $\mathcal{A}$ and $\mathcal{B}$ is  
$$
    \mathrm{H}_p(\mathcal{A},\mathcal{B}) = \inf_{\mathbf{C}}\Bigg(\sum_{(x,y)\in\mathbf{C}}\|x-y\|_2^p\Bigg)^{\frac{1}{p}},
$$
where $\mathbf{C}$ ranges over all correspondences in $\mathbf{C}(\mathcal{A}, \mathcal{B})$. Using the convexity of $\ot_p^p$ and the Wasserstein stability theorem for VR filtrations of point clouds, we have the following inequalities.

\begin{proposition}\label{prop:bias-control}
Suppose $\mathcal{X}\subseteq\mathbb{R}^m$ is of size $N$, then
   \begin{equation}\label{eq:ineqs}
    \mathrm{OT}^p_p(D_\mu, D[\mathcal{X}])  \le \mathbb{E}[\mathrm{OT}^p_p(\bm{D}_n, D[\mathcal{X}])]  \le C_N\mathbb{E}[\mathrm{H}_p^p(\bm{S}_n, \mathcal{X})], 
\end{equation} 
where $C_N>0$ is a constant depending on $N$.
\end{proposition}


Thus, in order to bound the bias, we only need to bound the $p$-Hausdorff distance between the point cloud-valued random variable $\bm{S}_n$  and $\mathcal{X}$. We achieve this with the following tail bound for the $p$-Hausdorff distance.

\begin{lemma}
\label{lem:tail_prob}
For any $r > 2r_0N^{1/p} > 0$, we have
$$
\mathbb{P}(\mathrm{H}_p(\bm{S}_n, \mathcal{X}) > r) \leq \frac{4^b N^{b/p}}{ar^b}\exp\bigg( -\frac{ar^bn}{2^b N^{b/p}} \bigg).
$$
\end{lemma}

Under the $(a,b,r_0)$-standard assumption on $\mathcal{X}$, we have the following bias estimation for persistence measures.

\begin{lemma}
\label{lemma:hausdorff_bound}
Let $\theta := \frac{p}{b}-1$.  For any $p>b$, there exists a constant $C_2>0$ such that
\begin{equation}
\label{eq:haus_bound_1}
\mathbb{E}\big[ \mathrm{OT}^p_p(D_\mu, D[\mathcal{X}]) \big] \leq C_2\left( r_0^p + \frac{\Gamma(\theta)}{n^{\theta}}\right),
\end{equation}
where $\Gamma(\cdot)$ here is the gamma function.  For $p \leq b$, set $\displaystyle r_n = \bigg(\frac{\log n}{an}\bigg)^{1/b}$, there exists a constant $C_3>0$ such that
\begin{equation}
\label{eq:haus_bound_2}
\mathbb{E}[\mathrm{OT}^p_p(D_\mu, D[\mathcal{X}])] \leq C_3\left(r_0^p +  r_n^p \mathbf{1}_{\{ r_n > r_0\}} +\bigg( \frac{\log n}{n} \bigg)^{p/b} \frac{1}{(\log n)^2} \right).
\end{equation}
\end{lemma}

\paragraph{Convergence Rate.}
For $\mathcal{X}$ a point cloud in $\mathbb{R}^m$, the number of points in the random persistence diagram $\bm{D}_n$ is bounded.  Therefore, there exists $L > 0$ such that the pushforward measure $(\pi^{\otimes n})_*$ is supported on $\mathcal{M}_{0,L}$.  This allows us to apply Theorem \ref{thm:variance-original}.

\begin{corollary}
Let $\mathcal{X} \subset \mathbb{R}^m$ with finitely many points; let $1 \leq p < \infty$.  Then there exists a constant $C_4 > 0$ such that
\begin{equation}
\label{eq:var_bound}
\mathbb{E}[\ot_p^p(\bar D, D_\mu)] \leq \frac{C_4}{\sqrt{B}}.
\end{equation}
\end{corollary}

From the above discussions on the variance and bias, we obtain the following rate estimates for mean persistence measures.

\begin{theorem}\label{thm:main-mpm}
Let $\mathcal{X}\subset \mathbb{R}^m$ be a finite set of points and $\pi$ be a probability measure on $\mathcal{X}$ satisfying the $(a,b,r_0)$-standard assumption. Suppose $S_n^{(1)},\ldots,S_n^{(B)}$ are $B$ i.i.d. samples from the distribution $\pi^{\otimes n}$. Let $\bar{D}$ be the empirical persistence measure; denote $\theta := \frac{p}{b}-1$. Then the empirical persistence measure approaches the true persistence measure $D[\mathcal{X}]$ at the following rates:
\begin{equation}
\label{eq:main-mpm}
    \mathbb{E}[\ot_p^p(\bar{D},D[\mathcal{X}])]\leq
    \begin{cases}
    O(B^{-1/2})+O(r_0^p)+  O\big(n^{-\theta}\big)& \mbox{~if~} p>b,\\
    \displaystyle O(B^{-1/2})+O(r_0^p) + O\bigg(\bigg(\frac{\log n}{n}\bigg)^{\theta+1}\bigg)& \mbox{~if~} \displaystyle p\le b.
    
    \end{cases}
\end{equation}
\end{theorem}

\begin{remark}
\label{rem:randomness}
Notice that as opposed to consistency results in classical statistics, here we do not consider the limiting case where $n$ tends to infinity since the total space is finite. In this setting the problem of interest is to obtain a valid approximation for a persistence diagram representing the true persistence diagram for a large yet finite dataset in the form of a point cloud $\mathcal{X}$.  This has important practical implications when $n$ and $B$ are both small with respect to $\mathcal{X}$, which will induce randomness.
\end{remark}

\subsubsection{Case II: \v{C}ech Filtration, $\mathcal{X}$ Compact}
\label{sec:continuous-estimation}


One way to extend $\mathcal{X}$ from a finite space to any compact space is via finite approximation, i.e., by letting the cardinality $N$ tend to infinity. However, Theorem \ref{thm:main-mpm} fails if $N$ is not fixed, since the bound on the $p$-Hausdorff distance is proportional to $N$. Another obstacle lies in the assumption of a uniform bound on the cardinality of point clouds in order to apply results on Wasserstein stability. Replacing boundedness by other assumptions would entail studying conjectures on finiteness of the constants \citep{skraba-2020-wasserstein}.  

To bypass such obstacles, we return to the original Wasserstein stability for Lipschitz functions \citep{cohen-2010-lipschitz}. This observation depends on the fact that \v{C}ech filtrations yield the same persistent homology as level set filtrations of the distance function $\mathrm{dist}_\mathcal{X}(x) = \displaystyle\min_{y\in\mathcal{X}}\|x-y\|$ for finite $\mathcal{X}\subseteq\mathbb{R}^m$.

\begin{theorem}{\citep[Lemma 3.4]{chazal2008towards}}\label{thm:persistent-nerve}
    Let $\mathcal{X}\subseteq\mathbb{R}^m$ be a finite set of points and let $0<\epsilon<\epsilon'$. Then there exists homotopy equivalence $\mathcal{X}^\epsilon\simeq \mathrm{\check{C}ech}(\mathcal{X},\epsilon)$ and $\mathcal{X}^{\epsilon'}\simeq \mathrm{\check{C}ech}(\mathcal{X},\epsilon')$, where $\mathcal{X}^\epsilon=\displaystyle\bigcup_{x\in\mathcal{X}}\mathcal{B}(x,\epsilon)$, such that the following diagram commutes:
    $$
    \begin{tikzcd}
        H_\bullet(\mathcal{X}^{\epsilon}) \arrow{r} &H_\bullet(\mathcal{X}^{\epsilon'})\\
       H_\bullet(\mathrm{\check{C}ech}(\mathcal{X},\epsilon)) \arrow{r} \arrow{u}& H_\bullet(\mathrm{\check{C}ech}(\mathcal{X},\epsilon')) \arrow{u}
    \end{tikzcd}
    $$
    where the vertical arrows are isomorphisms.
\end{theorem}

Since $\mathrm{dist}_\mathcal{X}((-\infty,\epsilon]) = \bigcup_{x\in\mathcal{X}}\mathcal{B}(x,\epsilon)$, we can identify persistence diagrams from \v{C}ech filtrations and level set filtrations.

\begin{corollary}
    Let $\mathcal{X}\subseteq\mathbb{R}^m$ be a finite set. Then $D[\mathrm{\check{C}ech}(\mathcal{X})]=D[\mathrm{dist}_\mathcal{X}]$.
\end{corollary}

Now let $\mathcal{X}\subseteq\mathbb{R}^m$ be a compact set such that the distance function $\mathrm{dist}_\mathcal{X}$ is \emph{tame}; intuitively, this means that the distance function is sufficiently well-behaved and has finitely many change points. Assume $\pi$ is a probability measure supported on $\mathcal{X}$ satisfying the $(a,b)$-standard assumption. Let $\bm{S}_n$ be the random variables with distribution $\pi^{\otimes n}$ and $S_n^{(1)},\ldots,S_n^{(B)}$ be i.i.d.~samples of $\bm{S}_n$. For each $S_n^{(i)}$, take the \v{C}ech filtration, or equivalently the level set filtration of $\mathrm{dst}_{S_n^{i}}$, and compute its persistence diagram. The mean persistence measure is given by
$$
\bar{D} = \frac{1}{B}\sum_{i=1}^BD[\mathrm{dist}_{S_n^{i}}].
$$
We would like to control the error
\begin{equation}\label{eq:continuous-error}
    \mathbb{E}[\mathrm{OT}_p^p(\bar{D},D[\mathrm{dist}_\mathcal{X}])].
\end{equation}
The variance--bias decomposition \eqref{eq:bias_var} also holds for \eqref{eq:continuous-error}. Using the same strategy, but a different stability theorem, we are able to derive the bias rate.

\begin{lemma}\label{lemma:continuous-bias}
    Let $\mathcal{X}\subseteq\mathbb{R}^m$ be a compact set and $0<m<k<p$. Let $\pi$ be a probability measure satisfying the $(a,b)$-standard assumption. Then there exists a constant $C_5>0$ such that
    \begin{equation}\label{eq:continuous-bias}
    \mathbb{E}[\mathrm{OT}_p^p(D_\mu,D[\mathrm{dst}_\mathcal{X}])]\le  C_5\bigg(1+\frac{1}{(\log n)^2}\bigg)\bigg(\frac{\log n}{an}\bigg)^{\frac{p-k}{b}} .
    \end{equation}
\end{lemma}

The parameter $k$ involved in Lemma \ref{lemma:continuous-bias} comes from the Wasserstein stability theorem for Lipschitz functions. A metric space $\mathcal{X}$ has bounded $k$-total persistence if there exists a constant $C_\mathcal{X}$ such that $\mathrm{W}_k(D[f],D_\emptyset)\le C_\mathcal{X}$ for every tame function $f: \mathcal{X}\to\mathbb{R}$ with Lipschitz norm $\|f\|_{\mathrm{Lip}}\le 1$. Since the function of interest is $\mathrm{dist}_{\mathcal{X}}$, which is defined on $\mathbb{R}^m$, our strategy is to place $\mathcal{X}$ inside a ball $\mathcal{B}(R)\subseteq\mathbb{R}^m$ that is known to have bounded $k$-total persistence for any $k>m$.

\paragraph{Convergence Rate.}
We can now directly apply Theorem \ref{thm:variance-original} to obtain a variance estimate.

\begin{lemma}\label{lemma:var-rate-continuous}
    Let $\mathcal{X}\subseteq\mathbb{R}^m$ be a compact set and $0<m< k< p$. Then there exists a constant $C_6>0$ such that 
    \begin{equation}\label{eq:var-rate-continuous}
            \mathbb{E}[\mathrm{OT}_p^p(\bar{D},D_\mu)]\le C_6\left(\frac{1}{\sqrt{B}}+\frac{1}{B^{p-k}}\right).
    \end{equation}
    In particular, for $p>m+\frac{1}{2}$, the rate is dominated by $O(\frac{1}{\sqrt{B}})$.
\end{lemma}

Summarizing estimations for variance and bias, we obtain the following rate estimates for the approximation error for mean persistence measures.

\begin{theorem}\label{thm:cont-mpm}
Let $\mathcal{X}\subset \mathbb{R}^m$ be a compact subset and $\pi$ be a probability measure on $\mathcal{X}$ satisfying the $(a,b)$-standard assumption. Suppose $S_n^{(1)},\ldots,S_n^{(B)}$ are $B$ i.i.d.~samples from the distribution $\pi^{\otimes n}$. Let $\bar{D}$ be the empirical persistence measure from the \v{C}ech filtration or the level set filtration of distance function. Let $0<m<k<p$. Then the empirical persistence measure approaches the true persistence measure $D[\mathrm{dist}_\mathcal{X}]$ at the following rate
\begin{equation}
\label{eq:cont-mpm}
    \mathbb{E}[\ot_p^p(\bar{D},\, D[\mathrm{dist}_\mathcal{X}])]\leq O\left(\frac{1}{\sqrt{B}}+\frac{1}{B^{p-k}}\right)+O\left(\bigg(1+\frac{1}{(\log n)^2}\bigg)\bigg(\frac{\log n}{an}\bigg)^{\frac{p-k}{b}}\right).
\end{equation}
\end{theorem}


\subsection{Multiple Subsampling for Persistence Diagrams}
\label{sec:multiple-subsampling-pd}

In previous discussions for HCVs and persistence measures, the empirical and population means are not elements in $\mathcal{D}_p$, i.e., they are not ordinary persistence diagrams. The multiple subsampling method can be applied directly to the space of persistence diagrams $\mathcal{D}_p$, if we use Fr\'echet means to define mean persistence diagrams.

In the same setting as in Section \ref{sec:multiple-subsampling-pm}, we define the empirical Fr\'echet mean as 
\begin{equation*}
    \widehat{\mathrm{Fr}} =\mathop{\arg\min}_{D\in\mathcal{D}_p} \frac{1}{B}\sum_{i=1}^B\mathrm{W}_{p}^{p}(D,D[S_n^{i}]).
\end{equation*}
Suppose $(\pi^{\otimes{n}})_*$ is the pushforward probability measure on $\mathcal{D}_p$, the corresponding population Fr\'echet mean is defined as 
\begin{equation*}
    \mathbf{Fr} =\mathop{\arg\min}_{D\in\mathcal{D}_p} \int_{\mathcal{D}_p}\mathrm{W}_{p}^{p}(D,D')\,\mathrm{d} (\pi^{\otimes{n}})_*(D').
\end{equation*}
The same variance--bias decomposition holds for Fr\'echet means:
\begin{equation}
\label{eq:bias_var_fr}
\mathbb{E}\big[ \mathrm{W}^p_p(\widehat{\mathrm{Fr}}, D[\mathcal{X}]) \big] \leq 2^{p-1}\Big(\underbrace{\mathbb{E}\big[ \mathrm{W}_p^p(\widehat{\mathrm{Fr}}, \mathbf{Fr}) \big]}_{\text{variance}} + \underbrace{\mathrm{W}_p^p(\mathbf{Fr}, D[\mathcal{X}])}_{\text{bias}}\Big).
\end{equation}

As opposed to the setting with persistence measures, there are limitations to achieving theoretical results due to the intrinsic geometry of $\mathcal{D}_p$, which prohibits a  convergence rate on the variance.

\paragraph{Variance.}
Under the assumption of a unique population Fr\'echet mean, \cite{turner-2014-frechet} proved an asymptotic result on the consistency of empirical Fr\'echet means on $\mathcal{D}_2$. The derivation of a finite sample convergence rate for Fr\'echet means under practical conditions  remains an open problem. In general, for metric spaces with nonnegative curvature in the sense of Alexandrov,  due to a lack of convexity of the objective function, deriving nonasymptotic rates can be prohibitive and often requires additional assumptions on the population distribution.

The state-of-the-art result on finite sample convergence rates for empirical Wasserstein means/barycenters in nonnegatively curved spaces is given by \cite{gouic-2019-fast}. A key assumption in the framework of \cite{gouic-2019-fast} is known as the bi-extendability of geodesics. When applied to the space of persistence diagrams, this bi-extendability condition is interpreted by \cite{cao2022geometric} as  
\emph{flatness} for a set of persistence diagrams. Under the flatness assumption, the empricial Fr\'echet mean converges at a rate of $O(1/\sqrt{B})$ as expected. However, such a flatness condition is not applicable here since it requires that all persistence diagrams have the same number of off-diagonal points. Other than theories from Alexandrov geometry, there are other approaches to derive finite sample convergence rates of empirical Fr\'echet means, such as quadruple inequalities \citep{schotz2019convergence} and relaxation \citep{blanchard2022fr}. However, the assumptions needed for these means to be applicable are difficult to verify in our setting.

\paragraph{Bias.} For bias estimation, we face the same challenges as for variance estimation. However, using a similar approach as in Section \ref{sec:multiple-subsampling-pm} we can provide a loose bound for the bias term. Denote the Fr\'echet variance by
\begin{equation*}
\sigma^2 = \int_{\mathcal{D}_p}\mathrm{W}_p^p(\mathbf{Fr},D')\,\mathrm{d}(\pi^{\otimes n})_*(D').
\end{equation*}
Then we have the following bound for the bias.

\begin{theorem}\label{thm:bias-fr}
Let $\mathcal{X}\subset \mathbb{R}^m$ be a finite set of points, and $\pi$ be a probability measure on $\mathcal{X}$ satisfying the $(a,b,r_0)$-standard assumption. Suppose $S_n^{(1)},\ldots,S_n^{(B)}$ are $B$ i.i.d. samples from the distribution $\pi^{\otimes n}$. Let $\widehat{\mathrm{Fr}}$/$\mathbf{Fr}$ be the empirical/population Fr\'echet mean and $\sigma^2$ be the Fr\'echet variance. Denote $\theta := \frac{p}{b}-1$. 
\begin{equation}
    \mathrm{W}^p_p(\mathbf{Fr},D[\mathcal{X}])\leq
    \begin{cases}
    O(\sigma^2)+O(r_0^p)+  O\big(n^{-\theta}\big)& \mbox{~if~} p>b,\\
    \displaystyle O(\sigma^2)+O(r_0^p) + O\bigg(\bigg(\frac{\log n}{n}\bigg)^{\theta+1}\bigg)& \mbox{~if~} \displaystyle p\le b. 
    \end{cases}
\end{equation}
\end{theorem}

\paragraph{Utility in Applications.}
Though there exists no complete expression for the approximation error for Fr\'echet means of persistence diagrams, by using the computational algorithm of \cite{turner-2014-frechet}, the multiple subsampling method is fully applicable for persistence diagrams in practice. In simulations and applications, we will demonstrate the utility of this approach through many examples.


\section{Tuning Parameters}
\label{sec:prac-para-turning}

In this section we explain how to interpret and tune the parameters appearing in the theoretical analysis of finite sample convergence rates. We focus on Theorem \ref{thm:main-mpm} where the setting is multiple subsampling for persistence measures for a finite underlying space $\mathcal{X}$.  

\subsection{Interpreting Parameters} 

We analyze the parameters in each of the three components in our derived rate estimations.

\paragraph{Variance.} The variance term is only affected by the number of subsampled sets $B$. Increasing $B$ to infinity will eliminate the variance term, however, in practice when computational efficiency is a priority, we would like for $B$ to be as small as possible. Theorem \ref{thm:main-mpm} allows us to tune $B$ and thus select the number of subsamples to draw so that the variance and bias are of the same rate. In this case, we have
\begin{equation}\label{eq:optimal-rate}
    \mathbb{E}[\ot_p^p(\bar{D},\, D[\mathcal{X}])]\le \begin{cases}
 O(r_0^p)+O(n^{-\theta}) & \mbox{~if~} p>b,\,  B= O(n^{2\theta})\\ \displaystyle O(r_0^p)+O\bigg(\bigg(\frac{\log n}{n}\bigg)^{\theta+1}\bigg)& \mbox{~if~} \displaystyle p\le b,\,  B=O\bigg(\bigg(\frac{n}{\log n}\bigg)^{2(\theta+1)}\bigg).\\
    \end{cases}
\end{equation}
Note that in the case when $\mathcal{X}$ is finite, the variance estimation is valid for any Wasserstein index $1\le p<\infty$. 

\paragraph{Intrinsic Error.} The constant term of $O(r_0^p)$ comes from the $(a,b,r_0)$-standard assumption, where $r_0>0$ is intrinsic when $\mathcal{X}$ is finite. In practice, this constant term is negligible for two reasons: First, when $\mathcal{X}$ is a dense dataset, $r_0$ is controlled by the smallest mutual distance in $\mathcal{X}$. A typical example in this case is a large dataset uniformly sampled from a compact manifold. Second, in practical computational settings, $B$ and $n$ are small, and the approximation error is dominated by the variance and bias terms. 
    
\paragraph{Bias.} The parameters we can tune to control the bias are $n$ and $p$. Note that the $(a,b,r_0)$-standard assumption applies to the distribution, which in practice is rarely known except for geometric data. Ideally, we would like for $n$ to be as large as possible. However, as previously discussed, the goal is to perform estimation with a relatively small $n$, so that randomness is induced (see Remark \ref{rem:randomness}). The convergence rate is determined by Wasserstein index $p$: the larger $p$ is, the faster the convergence is of the bias term. Intuitively, $p$ can be seen as a weight for ``topological noise'' where for larger $p$ less information is attributed to points near the diagonal in the persistence diagrams. However, when $p$ becomes too large the numerics will become small and eventually cause floating problems. 



\subsection{Practical Tuning}

In practice, there are three parameters to tune: the number of subsampled sets $B$, the size of a subsample set $n$, and the index of Wasserstein distance $p$. 

\begin{table}[htbp]
\scriptsize
\setlength{\extrarowheight}{5pt}
\newcolumntype{Y}{>{\centering\arraybackslash}X}

\begin{tabularx}{\textwidth}{|c|c|c|Y|Y|c|c|c|Y|Y|}
\hline
$n$ & $p$  & $B$   & $\mathrm{Err}_B$   & Time  & $n$ & $p$  & $B$   & $\mathrm{Err}_B$   & Time\\
\hline
\multirow{4}{*}{100}  & 2  & 40  & 0.6069 & 8.0156   & \multirow{4}{*}{1000} & 2  & 22 & 1.5655 & 378.6562 \\
\cline{2-5} \cline{7-10}
                      & 5  & 17  & 0.0307 & 1.8593   &                       & 5  & 30 & 0.1876 & 693.4218 \\
                      \cline{2-5} \cline{7-10}
                      & 7  & 74  & 0.0081 & 24.5156  &                       & 7  & 7  & 0.0998 & 40.5312  \\
                      \cline{2-5} \cline{7-10}
                      & 10 & 131 & 0.0017 & 72.5468  &                       & 10 & 28 & 0.0394 & 602.8281 \\
                      \hline
\multirow{4}{*}{200}  & 2  & 21  & 1.0365 & 9.9218   & \multirow{4}{*}{1500} & 2  & 8  & 1.6468 & 101.0937 \\
\cline{2-5} \cline{7-10}
                      & 5  & 32  & 0.0627 & 22.0781  &                       & 5  & 21 & 0.2245 & 672.9218 \\
                      \cline{2-5} \cline{7-10}
                      & 7  & 25  & 0.0221 & 13.5937  &                       & 7  & 41 & 0.1221 & 2537.0165 \\
                      \cline{2-5} \cline{7-10}
                      & 10 & 27  & 0.0056 & 15.8593  &                       & 10 & 30 & 0.0531 & 1378.1093 \\
                      \hline
\multirow{4}{*}{500}  & 2  & 9   & 1.4075 & 16.8751  & \multirow{4}{*}{2000} & 2  & 9  & 1.6848 & 196.2031 \\
\cline{2-5} \cline{7-10}
                      & 5  & 25  & 0.1288 & 121.7656 &                       & 5  & 3  & 0.2487 & 22.2343  \\
                      \cline{2-5} \cline{7-10}
                      & 7  & 20  & 0.0566 & 78.2187  &                       & 7  & 31 & 0.1399 & 2200.2812 \\
                      \cline{2-5} \cline{7-10}
                      & 10 & 33  & 0.0196 & 212.8281 &                       & 10 & 39 & 0.0644 & 3464.1718\\
                      \hline
\end{tabularx}
\caption{Exhaustive tests for parameters on torus data.}
\label{tab:torus}
\end{table}

For $B$ and $n$, these can both be large as long as the computation can be done in ``reasonable'' time, which is normally user-specified. Moreover, the computation of persistent homology for each of the $B$ sets is independent, which allows the use of parallel computing.  We now further explore the implications of tuning $B$ and $n$ on computational runtime: Though the variance rate is sublinear, we find that practically, there is no need to take very large $B$ to achieve convergence. Specifically, we sampled 5,000,000 points from the torus/sphere as the underlying data set $\mathcal{X}$. For each fixed $n$ and $p$, we iteratively increased $B$ until convergence. The convergence is normally measured by the quotient of distances to the limit point \citep[Chapter 9]{ortega2000iterative},
$$
Q=\left|\dfrac{\mathrm{OT}_p^p(\bar{D}_{B+1},D[\mathcal{X}])}{\mathrm{OT}_p^p(\bar{D}_{B},D[\mathcal{X}])}-1\right|,
$$
where $\bar{D}_B$ denotes the mean persistence measure computed from $B$ sets. Since the true persistence diagram $D[\mathcal{X}]$ is not available, we consider the quotient of consecutive distances $\mathrm{Err}_B=\mathrm{OT}_p^p(\bar{D}_{B+1},\bar{D}_B)$ instead, which is more useful in practice \citep{van1994acceleration}:
$$
Q^*=\left|\dfrac{\mathrm{Err}_B}{\mathrm{Err}_{B-1}}-1\right|.
$$
We carry out an exhaustive test for different parameters. An experiment is considered to be convergent if $Q^*<10^{-3}$. 

For the choice of $p$, according to Theorem \ref{thm:main-mpm}, it is safe to choose large $p$ when the underlying dimension is not known to guarantee a fast convergence for the bias. However, the absolute error becomes small as $p$ increases, which is not helpful for computations. The effect is more apparent when the true persistence diagram has few points with long persistence and most points close to the diagonal. For example, for the 3D sphere Table \ref{tab:sphere} shows that the consecutive distance $\mathrm{Err}_B$ is not observable for large $p$.

\begin{table}[htbp]
\scriptsize
\setlength{\extrarowheight}{5pt}
\newcolumntype{Y}{>{\centering\arraybackslash}X}
\begin{tabularx}{\textwidth}{|c|c|c|Y|Y|c|c|c|Y|Y|}
\hline
$n$ & $p$  & $B$   & $\mathrm{Err}_B$   & Time  & $n$ & $p$  & $B$   & $\mathrm{Err}_B$   & Time\\
\hline
\multirow{4}{*}{100}  & 2  & 37  & 0.1814 & 9.7187   & \multirow{4}{*}{1000} & 2  & 21 & 0.2387 & 161.2986 \\
\cline{2-5} \cline{7-10}
                      & 5  & 97  & 0.0008 & 59.8754   &                       & 5  & 61 & 3.6$\times 10^{-5}$ & 1327.0312 \\
                      \cline{2-5} \cline{7-10}
                      & 7  & 82  & 4.6$\times 10^{-5}$ & 42.9375  &                       & 7  & 68  & N & 1655.1416  \\
                      \cline{2-5} \cline{7-10}
                      & 10 & 119 & 1.7$\times 10^{-6}$ & 88.0781  &                       & 10 & 60 & N & 1287.8134 \\
                      \hline
\multirow{4}{*}{200}  & 2  & 12  & 0.2264 & 2.2812   & \multirow{4}{*}{1500} & 2  & 18  & 0.2561 & 261.2656 \\
\cline{2-5} \cline{7-10}
                      & 5  & 54  & 0.0004 & 36.8906  &                       & 5  & 32 & 2.2$\times 10^{-5}$ & 809.0625 \\
                      \cline{2-5} \cline{7-10}
                      & 7  & 220  & 1.1$\times 10^{-5}$ & 577.8437  &                       & 7  & 28 & N & 620.8437 \\
                      \cline{2-5} \cline{7-10}
                      & 10 & 165  & N & 329.2343  &                       & 10 & 10 & N & 82.7031 \\
                      \hline
\multirow{4}{*}{500}  & 2  & 50   & 0.2274 & 195.6251  & \multirow{4}{*}{2000} & 2  & 3  & 0.2541 & 12.5937 \\
\cline{2-5} \cline{7-10}
                      & 5  & 4  & 0.0003 & 1.5625 &                       & 5  & 56  & 1.4$\times 10^{-5}$ & 3994.9691  \\
                      \cline{2-5} \cline{7-10}
                      & 7  & 207  & 1.2$\times 10^{-5}$ & 3311.5471  &                       & 7  & 40 & N & 2041.3912 \\
                      \cline{2-5} \cline{7-10}
                      & 10 & 67  & N & 353.2523 &                       & 10 & 8 & N & 87.3437\\
                      \hline
\end{tabularx}
\caption{Exhaustive tests for parameters on sphere data. If $\mathrm{Err}_B<10^{-6}$ it is recorded as N.}
\label{tab:sphere}
\end{table}


\section{Simulations}
\label{sec:experiments}


In this section, we provide numerical validations of our derived theory and discuss robustness of the multiple subsampling method to noise.

\subsection{Validating the Convergence Rate}
\label{sec:rate-experiment}

\paragraph{2D Torus.} We take $\mathcal{X}$ to be a sample set consisting of $N = 50,000$ points from a torus with outer radius 0.8 and inner radius 0.3. We then subsample $B=0.1n$ subsets each consisting of $n$ points from $\mathcal{X}$ to compute the empirical mean persistence measure $\bar{D}$. We repeat the subsampling procedure so that $n$ ranges from $400$ to $3,800$. Since we are uniformly sampling from a dense data set from a 2D manifold, $b$ is assumed to be 2. If we take $p=3$, by \eqref{eq:optimal-rate}, the optimal rate is $O(n^{-\frac{1}{2}})$ and the loss (i.e., approximation error) curve takes the form
\begin{equation}\label{eq:torus-rate}
    \mathbb{E}[\ot_3^3(\bar{D},D[\mathcal{X}])] \approx a_0 + a_1 n^{-\frac{1}{2}}.
\end{equation}
For $p=8$, the bias decreases at a rate of $3$, which is much faster than the variance. In this case the loss curve is dominated by variance and should be of the same form as \eqref{eq:torus-rate}. We fit the loss curves using the empirical losses. The result shown in Figure \ref{fig:torus-rate} presents convergence rates of 0.50 and 0.52 with respect to different $p$, which are consistent with our derived theory. 

\begin{figure}[htbp]
\centering
    \begin{subfigure}[t]{0.45\textwidth}
    \centering
    \includegraphics[width=\textwidth]{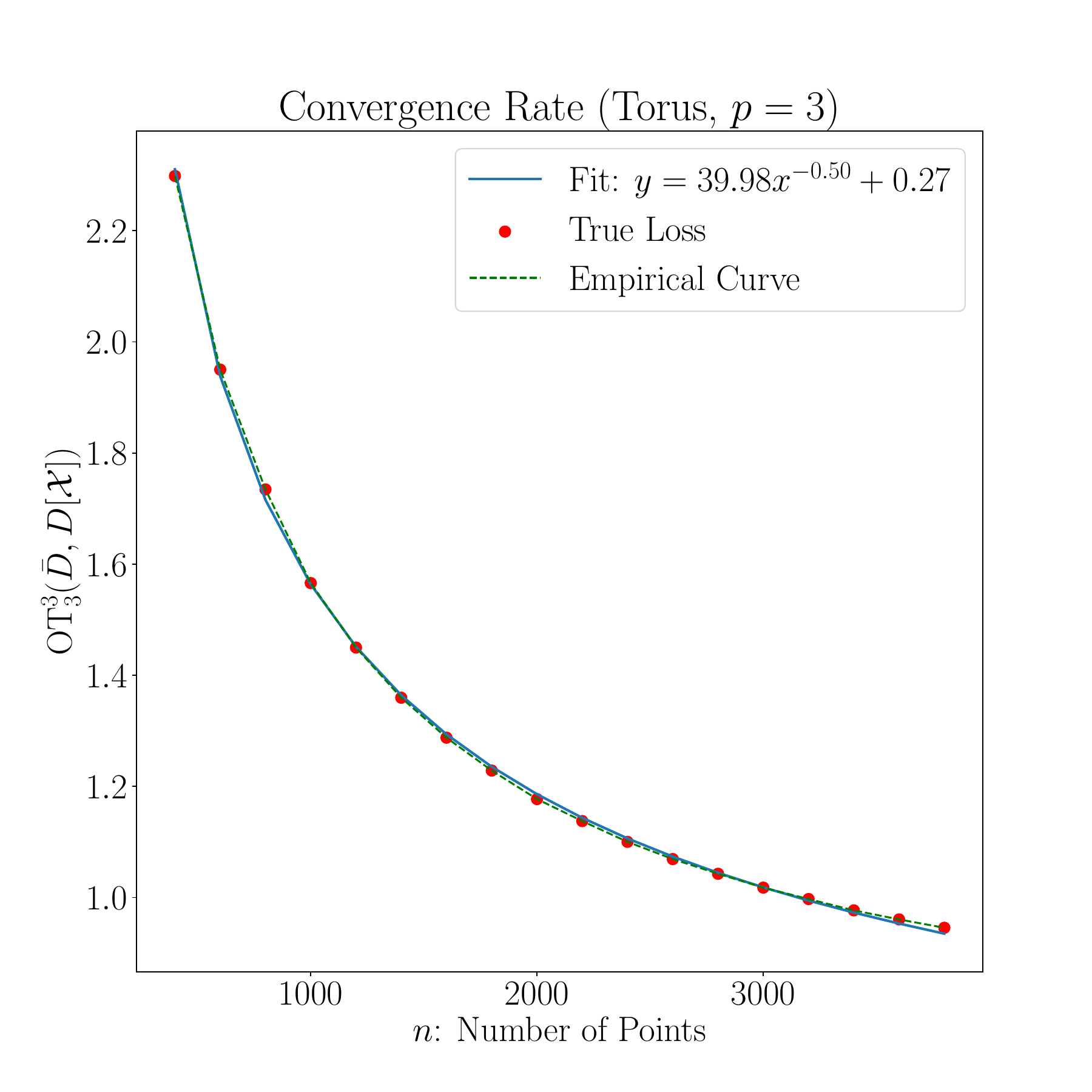}
    \caption{The loss curve for $p=3$}
    \end{subfigure}
    \begin{subfigure}[t]{0.45\textwidth}
    \centering
    \includegraphics[width=\textwidth]{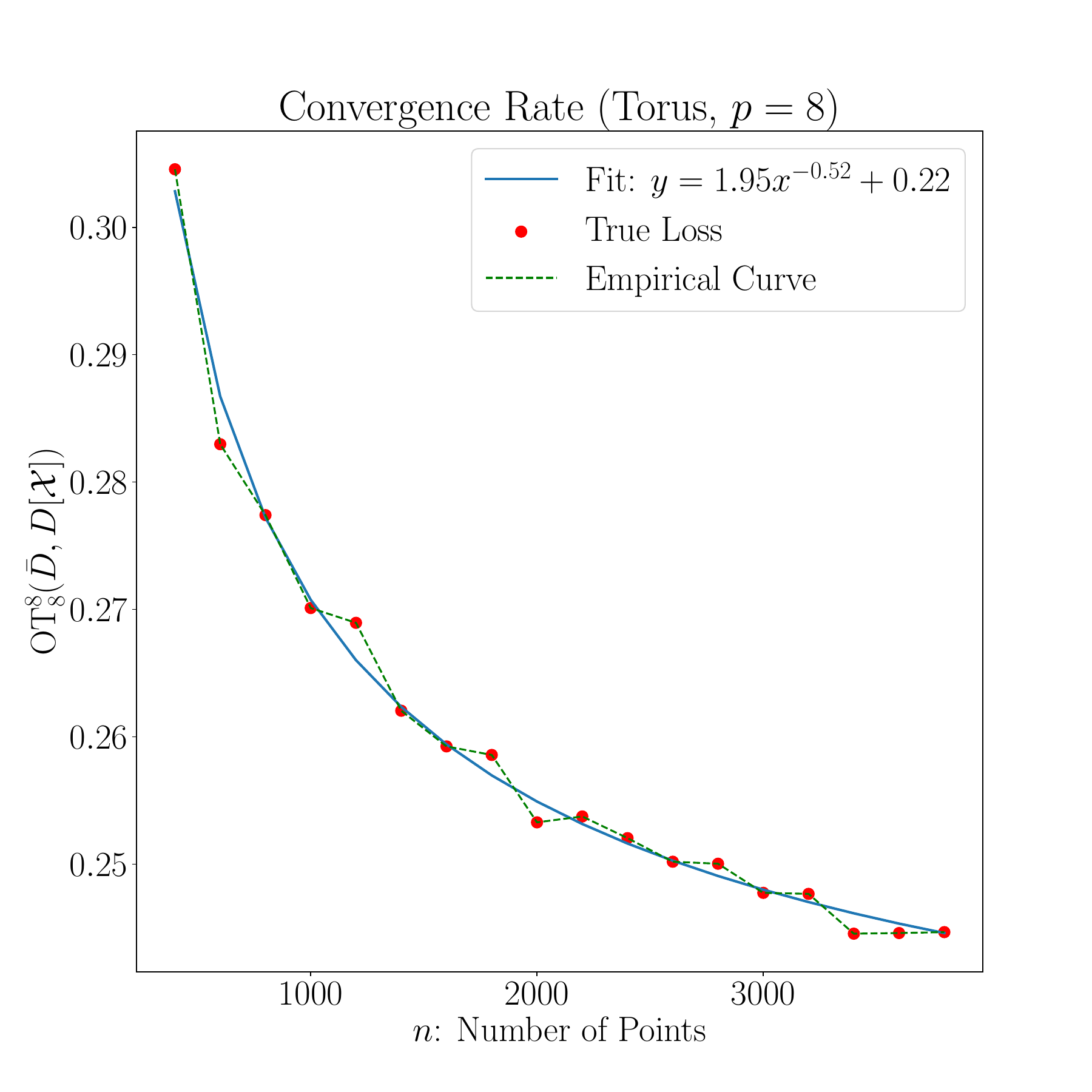}
    \caption{The loss curve for $p=8$}
    \end{subfigure}
    \caption{Convergence rate verification for mean persistence measure on the Torus (point cloud of $N = 50,000$ points).}
    \label{fig:torus-rate}
\end{figure}

\paragraph{3D Sphere.} We take $\mathcal{X}$ to be the sample set consisting of $N = 20,000$ points from a 3-dimensional sphere with radius 0.5. Then we sample $B=\lfloor n^{2/3}\rfloor$ subsets each consisting of $n$ points to compute the empirical mean persistence measure. We repeat the subsampling procedure so that $n$ ranges from 600 to 4,000. In this case $b$ is assumed to be 3. If we take $p=2$, notice that we are in the setting $p<b$ in Theorem \ref{thm:main-mpm}, the loss curve takes the form
\begin{equation}\label{eq:sphere-rate}
    \mathbb{E}[\ot_2^2(\bar{D},D[\mathcal{X}])] \approx a_0 + a_1 n^{-\frac{1}{3}}.
\end{equation}
If we take $p=9$, the bias vanishes in a rate $\theta=2$. Again the loss curve is dominated by variance and should be in the same form as \eqref{eq:sphere-rate}.  We fit the loss curves using the empirical losses. The result shown in Figure \ref{fig:sphere-rate} presents convergence rates of 0.32 and 0.36, which are consistent with our theory.


\subsection{Comparing Fr\'echet Means and Studying Robustness to Noise}

\begin{figure}[htbp]
\centering
\begin{subfigure}[t]{0.45\textwidth}
    \centering
    \includegraphics[width=\textwidth]{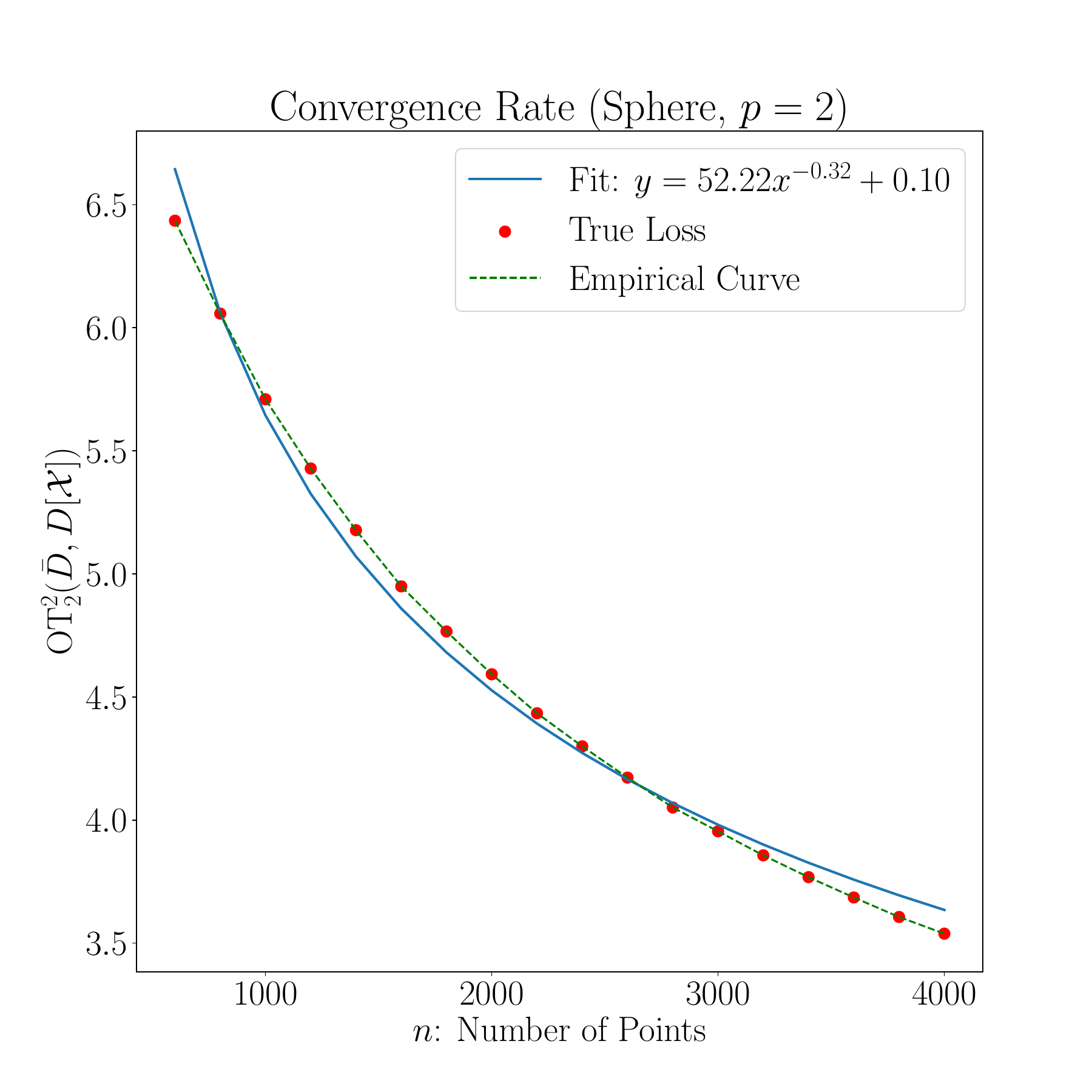}
    \caption{The loss curve for $p=2$}
\end{subfigure}
\begin{subfigure}[t]{0.45\textwidth}
    \centering
    \includegraphics[width=\textwidth]{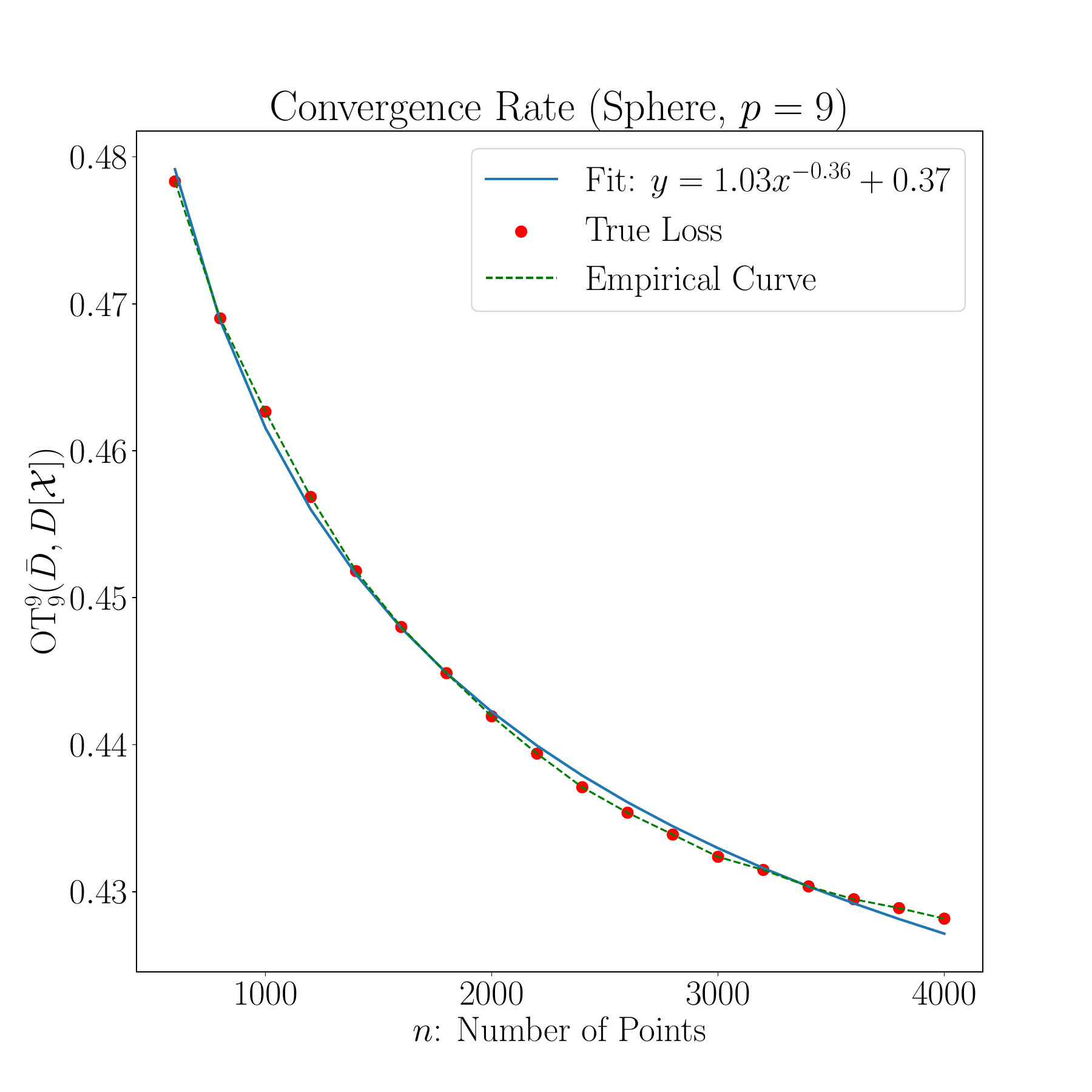}
    \caption{The loss curve for $p=9$}
\end{subfigure}
    \caption{Convergence rate verification for mean persistence measure on the Sphere (point cloud of $N=20,000$ points).}
    \label{fig:sphere-rate}
\end{figure}

We provide additional simulation results in the supplementary materials. In particular, we provide experimental comparisons of the Fr\'echet means and mean persistence measures as measures of central tendency in the multiple subsampling method. The comparison is done on three types of data: Euclidean point clouds, abstract finite metric spaces, and networks. Although in our theoretical convergence analysis we assumed point cloud data in Euclidean space, these experiments demonstrate the applicability of the multiple subsampling method to other types of data. We also show robustness of multiple subsampling to noise in input data.


\section{Applications}
\label{sec:real-data}





In this section, we demonstrate the multiple subsampling approach on real-world applications.


\subsection{Inferencing on Lexical Databases}
\label{sec:database-embedding}

In representation learning, the Poincar\'{e} embedding is one of the most widely-used embedding methods to learn representations of symbolic data \citep{nickel2017poincare}. Compared to the Euclidean embedding, the Poincar\'{e} embedding is able to simultaneously capture hierarchy and similarity due to the underlying hyperbolic geometry. For example, for lexical databases where different words have hypernymity relations, the Poincar\'{e} embedding is able to preserve the tree-like structure. 

Let $\mathbb{D}^m = \{u:\|u\|<1\}\subseteq\mathbb{R}^m$ be the unit open ball in $\mathbb{R}^m$. The Poincar\'{e} distance is defined by
\begin{equation}\label{eq:poincare-distance}
d_{\mathbb{D}}(u,v) = \mathrm{cosh}^{-1}\left(1+\frac{2\|u-v\|^2}{(1-\|x\|^2)(1-\|v\|^2)}\right),    
\end{equation}
where $\|\cdot\|$ is the Euclidean norm. Let $L_V$ be the dataset and $L_E$ be its hypernymity relations. Poincar\'{e} embeddings aim to minimize the following loss function:
$$
\mathcal{L} = \sum_{(u,v)\in L_E}\log\dfrac{e^{-d_{\mathbb{D}(u,v)}}}{\displaystyle\sum_{v'\in\mathcal{N}(u)}e^{-d_{\mathbb{D}(u,v')}}},
$$
where $u$ denotes both the raw data in $L_V$ and its representation in $\mathbb{D}^m$, and $\mathcal{N}(u)=\{v':(u,v')\notin L_E\}\cup\{v\}$ denotes the set of negative examples for $u$ \citep{nickel2017poincare}. Since the Poincar\'{e} ball has a Riemannian manifold structure, we can optimize the loss function and update the representations via Riemannian stochastic gradient descent \citep{bonnabel2013stochastic}. 

By placing the root node of a tree at the origin of the Poincar\'{e} ball, the hierarchical structure of the database is encoded in the distances of points to the origin. By \eqref{eq:poincare-distance}, leaf nodes are trained to be close to the boundary since the distance grows exponentially fast when the Euclidean norm is close to 1. We present snapshots of training at different epochs in Figure \ref{fig:poincare-embedding}, where we can see points are pushed to boundary as the number of epochs increases.

\begin{figure}[htbp]
\centering
\begin{subfigure}[t]{0.32\textwidth}
    \includegraphics[width=\textwidth]{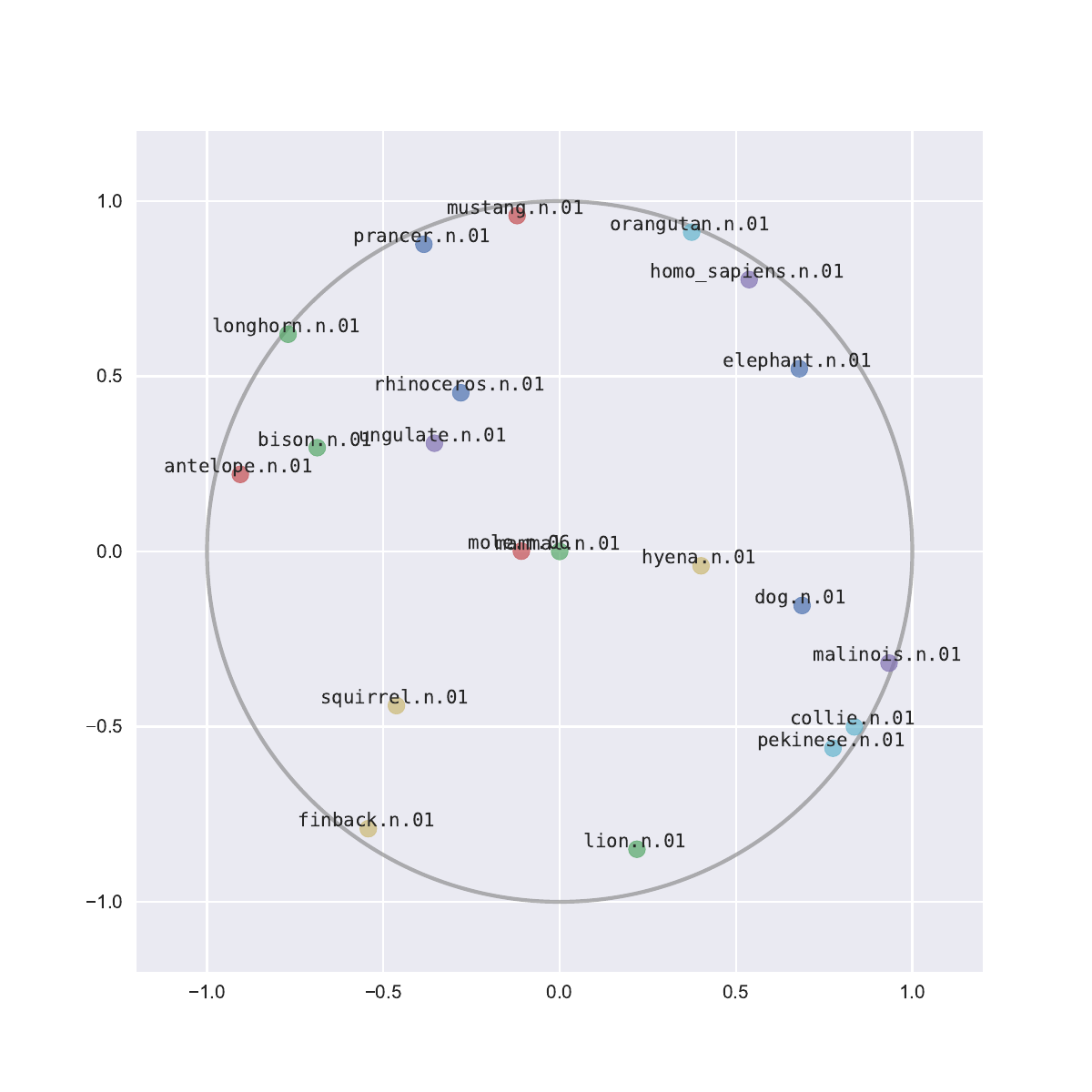}
    \caption{Epoch 20}
\end{subfigure}
\begin{subfigure}[t]{0.32\textwidth}
    \includegraphics[width=\textwidth]{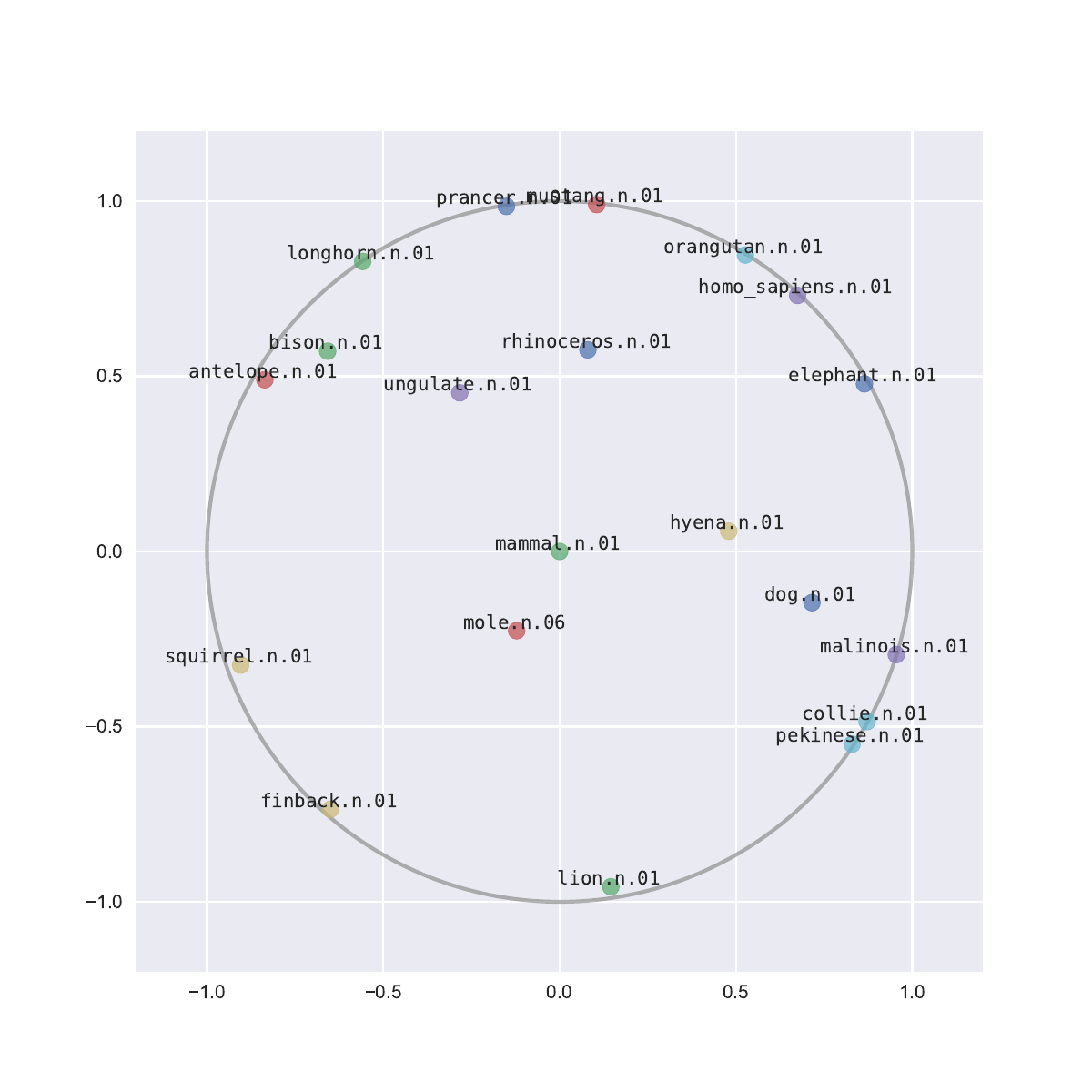}
    \caption{Epoch 50}
\end{subfigure}
\begin{subfigure}[t]{0.32\textwidth}
    \includegraphics[width=\textwidth]{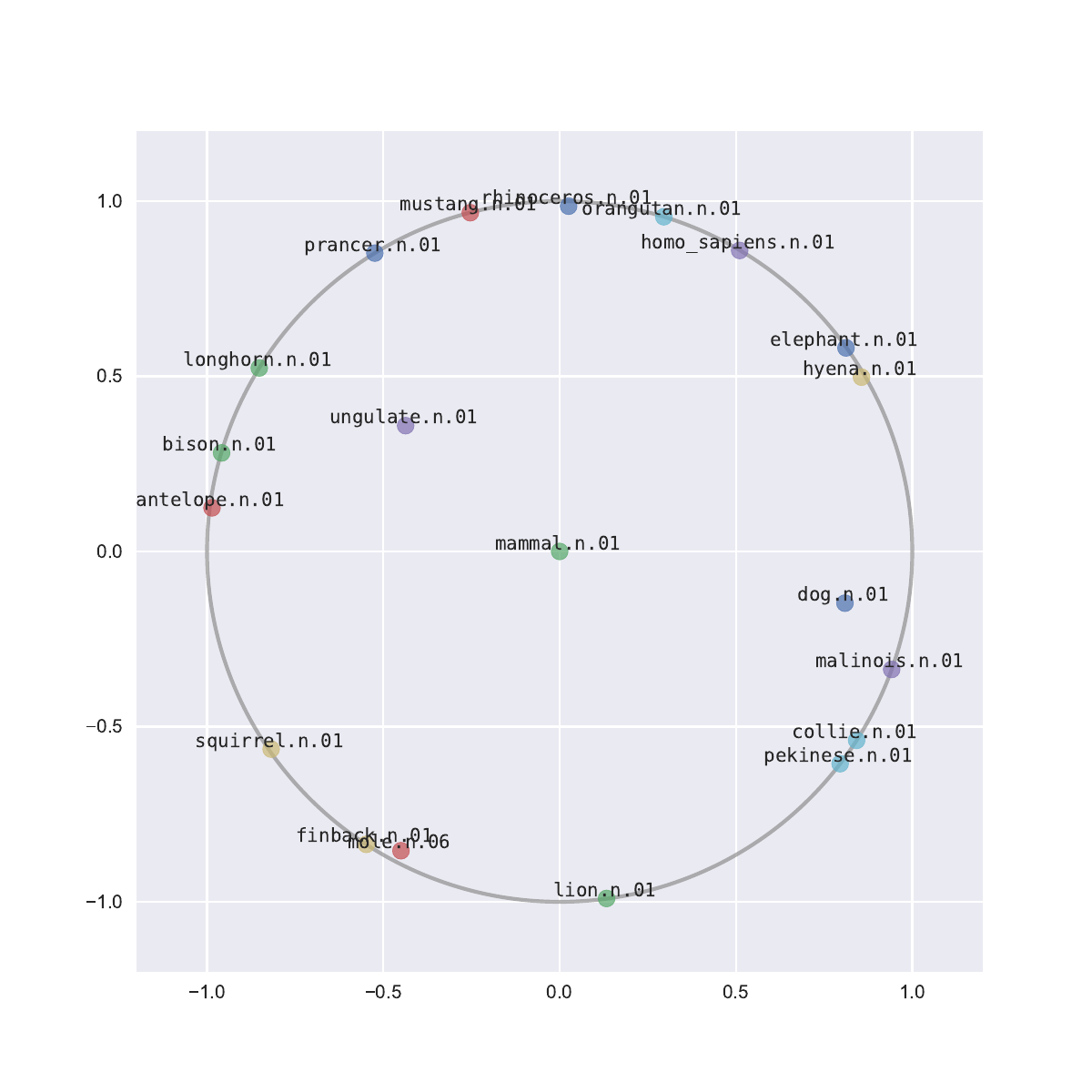}
    \caption{Epoch 150}
\end{subfigure}
    \caption{Training of Poincare embedding of \texttt{Mammal} subtree of \texttt{WordNet} at different epochs. We select 20 samples for visualization.}
    \label{fig:poincare-embedding}
\end{figure}

\cite{vlontzos2021topological} studied the persistent homology of Poincar\'{e} embeddings. It turns out that the persistence diagrams for VR filtrations capture limited information since the persistence of points are very small, indicating they are noise generated from filtrations. This behavior is explained by a phase transition for the connectivity of VR complexes built on points in the Poincar\'{e} ball. Roughly speaking, there is a critical threshold $r_c$ such that when $r<r_c$, the VR complex $\mathrm{VR}(\mathcal{X},r)$ mainly contains discrete 0-simplices, and when $r>r_c$ the VR complex $\mathrm{VR}(\mathcal{X},r)$ approaches a full dimensional simplex. The phase transition is sharper when the trained points are closer to the boundary of Poincar\'{e} ball. In fact, let $\mathbb{S}^{m-1}(r)\subseteq\mathbb{R}^m$ be the circle of radius $0<r<1$ centered at the origin. Measured with the Poincar\'{e} distance, the circle $\mathbb{S}^{m-1}(r)$ has diameter $2\mathrm{sinh}^{-1}\big(\frac{2r}{1-r^2}\big)$. Let $\bm{u}$ and $\bm{v}$ be two independent random $\mathbb{S}^{m-1}(r)$-valued random variables with uniform distribution. \cite{vlontzos2021topological} showed the cumulative distribution function (cdf) of the normalized Poincar\'{e} distance $\dfrac{1}{M}d_\mathbb{D}(\bm{u},\bm{v})$ is
$$
F_r(t) = \frac{2}{\pi}\sin^{-1}\bigg(\frac{1}{2}\mathrm{sinh}\bigg(\frac{tM}{2}\bigg)\bigg(\frac{1}{r}-r\bigg)\bigg).
$$
The cdf depends on the radius $r$. As $r\to 1$, the cdf converges to the Dirac distribution at 1. Thus, as the training process pushes points to the boundary of Poincar\'{e} ball, the points in the corresponding persistence diagram are also pushed to the boundary of the half-plane $\Omega$. As a result, we will only see topological noise in the persistence diagram.

We verify the above observation by testing the distribution of persistence diagrams of representation points and the distribution of persistence diagrams of random points near the boundary of $\mathbb{D}^m$. Let $\mathcal{X}_0$ be the set of representation points after performing the Poincar\'{e} embedding. Denote $\mu_0$ to be the underlying probability distribution of persistence diagrams from multiple subsampling of $\mathcal{X}_0$. Let $\mathcal{X}_1=\{u\in\mathbb{D}^m:r_1\leq \|u\|\leq r_2\}$ be a ``thin'' shell of $\mathbb{D}^m$ near the boundary such that $0<1-r_2<1-r_1\ll 1$. Denote $\mu_1$ to be the distribution of persistence diagrams from multiple subsampling of $\mathcal{X}_1$. We wish to test the null hypothesis $H_0:\mu_0=\mu_1$ that two distributions are equal against the alternative hypothesis $H_1:\mu_0\neq \mu_1$. Under the null hypothesis, the sampling of persistence diagrams does not depend on order, thus we may use a permutation test \citep{kim2022minimax}.

We sample $B_0$ persistence diagrams $\{D_1,\ldots,D_{B_0}\}$ from $\mu_0$ with mean persistence measure $\bar{D}_0=\frac{1}{B_0}\sum_{i=1}^{B_0}D_i$, and $B_1$ persistence diagrams $\{D_1^*,\ldots,D^*_{B_1}\}$ from $\mu_1$ with mean persistence measure $\bar{D}_1=\frac{1}{B_1}\sum_{i=1}^{B_1}D^*_i$. The observed distance between the two sample means is $\mathrm{OT}_{ob} = \mathrm{OT}_p(\bar{D}_0,\bar{D}_1)$. We then pool the persistence diagrams together $\{D_1,\ldots,D_{B_0},D_1^*,\ldots,D^*_{B_1}\}$ and shuffle the pooled data. If we permute the pooled data $L$ times, in the $\ell$th permutation, we compute the mean persistence measure of the first $B_0$ persistence diagrams and the mean persistence measure of the remaining $B_1$ persistence diagrams, and compute the distance $\mathrm{OT}_{\ell}$ between sample means. The $p$-value is the defined as
\begin{equation}\label{eq:pvalue}
 p = \frac{1+\# \{\mathrm{OT}_\ell\geq \mathrm{OT}_{ob}\} }{L}.   
\end{equation}

A small $p$-value indicates that the positions of data are not exchangeable and the original configuration is relevant to data, thus, we should reject the null hypothesis. If the $p$-value is greater than the significance level, then the evidence is not strong enough and we fail to reject the null hypothesis.

We apply the permutation test to the \texttt{Mammal} subtree of \texttt{WordNet} database. We train the Poincar\'{e} embedding of the database with different training epochs in dimension 2. After training, we take subsamples of the vector representation set, each consisting of $n=200$ points, and compute the persistence diagrams for VR filtrations based on the Poincar\'{e} distance \eqref{eq:poincare-distance}. Then we take samples of random points in the annulus $\mathcal{X}_1\subseteq\mathbb{D}^2$ with $r_1=0.98$ and $r_2=0.999$ and compute their persistence diagrams. We pool the persistence diagrams together and permute the pooled persistence diagrams 10,000 times. We compute 2-Wasserstein distances between mean persistence measures and the $p$-values according to \eqref{eq:pvalue}. The result is shown in Table \ref{tab:p-values}. Comparing each $p$-value with the significance level $\alpha=5\%$, we see that for small number of epochs of training, we obtain low $p$-values and thus reject the null hypothesis. However, after sufficiently many epochs of training and with representation points being pushed to the boundary of $\mathbb{D}^2$, the $p$-values are large and we fail to reject the null hypothesis. This means that statistically, it is difficult to distinguish persistence diagrams from the representation point sets with persistence diagrams from the random point sets in the annulus near the boundary of $\mathbb{D}^2$.  

\begin{table}[htbp]
\setlength{\extrarowheight}{5pt}
\newcolumntype{Y}{>{\centering\arraybackslash}X}
\centering
\begin{tabularx}{\textwidth}{|c|c|Y|Y|Y|Y|Y|Y|Y|Y|Y|}
\hline
\multicolumn{1}{|c|}{\multirow{2}{*}{Epoch}} & $B_0$ & 10       & 15       & 20       & 10       & 15       & 20       & 10       & 15       & 20 \\
\cline{2-11}
\multicolumn{1}{|c|}{}                       & $B_1$ & 10       & 10       & 10       & 15       & 15       & 15       & 20       & 20       & 20 \\
\hline
50                                         & $p$   & \textit{0.0001} & 0.0565         & \textit{0.0411}         & \textit{0.0001}         & \textit{0.0001}         &  \textit{0.0105}        &  \textit{0.0010}        &  \textit{0.0001}        & \textit{0.0001}   \\
\hline
150                                        &  $p$  & 0.3702 &  0.8752        &  0.7741        & 0.0784         &  0.2594        & 0.8233         & 0.1397         & 0.0847         & 0.0809   \\
\hline
200                                        & $p$   & 0.4761 & 0.8440 & 0.6313 & 0.1188 & 0.4448 & 0.9402 & 0.1913 & 0.1039 &  0.2175\\
\hline
\end{tabularx}
\caption{The $p$ values of permutation test. For each epoch we take $B_0$ subsample sets from the representation vectors and $B_1$ subsample sets from the annulus near the boundary of $\mathbb{D}^2$. The numbers in \emph{Italic} are lower than the significance level $5\%$.}
\label{tab:p-values}
\end{table}

\subsection{More Applications}

We further perform two additional applications, given in the supplementary materials. First, we use multiple subsampling to approximate the persistence diagrams of two massive point clouds with complex topology consisting of over 400,000 points. We argue that the mean persistence diagrams/measures reveal the underlying topological structures of the point clouds. Second, we apply multiple subsampling to a shape clustering task on a selected dataset from the Mechanical Components Benchmark (MCB). The dataset consists of 3D point clouds of size ranging from 30,000 points to 250,000 points, where a direct computation of persistent homology is not feasible. We show that mean persistence diagrams/measures are effective shape descriptors.


\section{Discussion}
\label{sec:disussion}
We studied a statistical approach to approximating the persistent homology for massive datasets based on a multiple subsampling approach.  We studied a broad class of vectorizations of persistence diagrams, persistence measures, and persistence diagrams themselves.  We derived nonasymptotic convergence rates for the empirical means of HCVs and persistence measures.  We provided guidelines for parameter selection as well as their interpretation.  Experimentally, we provided a numerical verification of our derived theoretical rates and applied multiple subsampling to study Poincar\'{e} embeddings of complex lexical databases.



\section*{Acknowledgments}

The authors wish to thank Omer Bobrowski, Th\'eo Lacombe, and Primo\v{z} \v{S}kraba for helpful conversations.  Y.C.~is funded by Digital Futures Postdoctoral Fellowship,  A.M.~is supported by the the UKRI EPSRC grant [EP/Y028872/1], Mathematical Foundations of Intelligence: An ``Erlangen Programme'' for AI.



\bibliographystyle{apalike}  
\bibliography{subsampling_ref} 

\vfill\eject



\appendix

\startcontents[app]
\printcontents[app]{l}{1}{\section*{Supplementary Materials}}

\section{Background on Persistent Homology}
\label{app:PH}

The progression of persistent homology starts with a \emph{filtration}, which is a nested sequence of topological spaces: $X_0 \subseteq X_1 \subseteq \cdots \subseteq X_n = X$.  In particular, filtrations can be applied to both point clouds and functions.

For point clouds, we focus on \emph{\v{C}ech} filtrations and \emph{Vietoris--Rips} (VR) filtrations. Let $\epsilon_1\le\epsilon_2\le\cdots\le \epsilon_n$ be an increasing sequence of parameters; let $\mathcal{X}\subseteq \mathbb{R}^m$ be a point cloud. The \emph{\v{C}ech complex} $\mathrm{\check{C}ech}(\mathcal{X},\epsilon_i)$ at scale $\epsilon_i$ is constructed by adding a $k$-simplex for each nonempty intersection of $k+1$ balls of radius $\epsilon_i$. For $\epsilon_i\le \epsilon_j$, we have $\mathrm{\check{C}ech}(\mathcal{X},\epsilon_i)\subseteq\mathrm{\check{C}ech}(\mathcal{X},\epsilon_j)$, thus giving the \emph{\v{C}ech filtration}. The \emph{VR complex} is a relaxed version of \v{C}ech complex: At scale $\epsilon_i$, the VR complex $\mathrm{VR} (\mathcal{X},\epsilon_i)$ is constructed by adding a node for each $x_j\in\mathcal{X}$ and a $k$-simplex for each set $\{x_{j_1},x_{j_2},\ldots,x_{j_{k+1}}\}$ with diameter less than $\epsilon_i$. Note that VR filtrations can be applied to any finite metric spaces $(\mathcal{X},d_{\mathcal{X}})$ without embedding into Euclidean space. 

For functions, we focus on (sub)level set filtrations. A function $f:\mathcal{X}\to\mathbb{R}$ defined on some topological space $\mathcal{X}$ is called \emph{tame} if it has finitely many topological critical values, i.e., the level set $f^{-1}((-\infty,\epsilon])$ changes topology at only finitely many $\epsilon$. For critical values $\epsilon_i\le\epsilon_j$, we have $f^{-1}((-\infty,\epsilon_i])\subseteq f^{-1}((-\infty,\epsilon_j])$, which gives the \emph{level set filtration}. 

A filtration of topological spaces induces a filtration of homology groups in the following manner. For any fixed dimension $\bullet$, let $H_\bullet(\mathcal{X},\epsilon_i)$ be the homology group of the \v{C}ech or VR complex or level set at scale $\epsilon_i$ with coefficients in a field. Then we have the following sequence of vector spaces:
$$
    H_\bullet(\mathcal{X},\epsilon_1)\to H_\bullet(\mathcal{X},\epsilon_2)\to\cdots\to H_\bullet(\mathcal{X},\epsilon_n).
$$
The collection of vector spaces $H_\bullet(\mathcal{X},\epsilon_i)$, together with vector space homomorphisms $H_\bullet(\mathcal{X},\epsilon_i)\to H_\bullet(\mathcal{X},\epsilon_j)$, is called a \emph{persistence module}. 

\begin{figure}[htbp]
    \begin{subfigure}{\linewidth}
        \includegraphics[width=\linewidth]{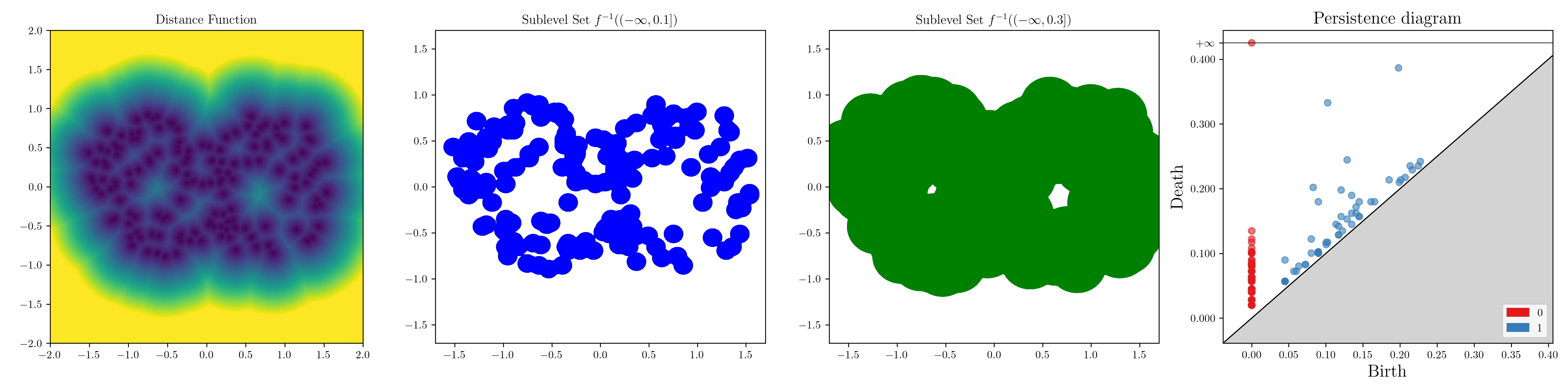}
        \caption{Sublevel set filtration}
    \end{subfigure}
    \begin{subfigure}{\linewidth}
        \includegraphics[width=\linewidth]{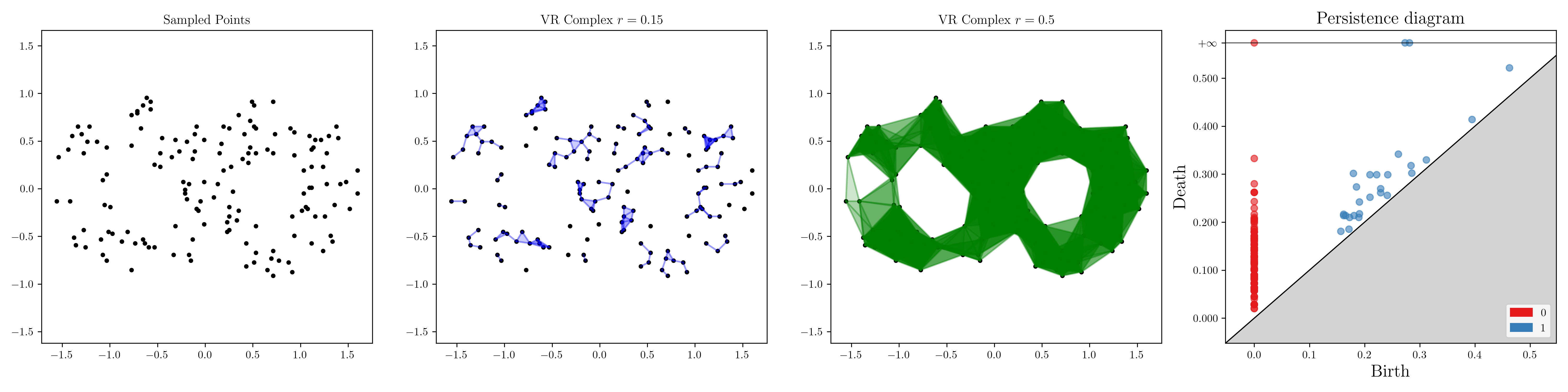}
        \caption{Vietoris--Rips filtration}
    \end{subfigure}
    \caption{Illustration of filtrations and persistence diagrams. (a) Sublevel set filtration of the distance function to a figure eight with Gaussian noise. (b) Vietoris--Rips filtration of samples from a figure eight with Gaussian noise.}
    \label{fig:pipeline}
\end{figure}

When each $H_\bullet(\mathcal{X},\epsilon_i)$ is finite dimensional, the persistence module can be decomposed into rank one summands which correspond to birth and death times of homology classes \citep{chazal2016structure}: Let $\alpha\in H_\bullet(\mathcal{X},\epsilon_i)$ be a nontrivial homology class; $\alpha$ is \emph{born} at $\epsilon_i$ if it is not in the image of $H_\bullet(\mathcal{X},\epsilon_{i-1})\to H_\bullet(\mathcal{X},\epsilon_i)$; it is \emph{dead} entering $\epsilon_j$ if the image of $\alpha$ via $H_\bullet(\mathcal{X},\epsilon_i)\to H_\bullet(\mathcal{X},\epsilon_{j-1})$ is not in the image $H_\bullet(\mathcal{X},\epsilon_{i-1})\to H_\bullet(\mathcal{X},\epsilon_{j-1})$, but the image of $\alpha$ via $H_\bullet(\mathcal{X},\epsilon_i)\to H_\bullet(\mathcal{X},\epsilon_{j})$ is in the image $H_\bullet(\mathcal{X},\epsilon_{i-1})\to H_\bullet(\mathcal{X},\epsilon_{j})$. The collection of birth--death intervals $[\epsilon_i,\epsilon_j)$ is called a \emph{barcode} and it represents the persistent homology of VR filtration of $\mathcal{X}$. Equivalently, we can regard each interval as an ordered pair of birth--death as coordinates and plot each in a plane $\mathbb{R}^2$, which provides an alternate representation of a barcode as a \emph{persistence diagram}.

\section{Background on Multiple Subsampling}
\label{app:stat_setup}


Let $\mathcal{X}$ be an unknown metric space with a predefined probability distribution $\pi$. We would like to approximate the persistent homology of $\mathcal{X}$ via the persistent homology of sample sets $S_N=\{X_1,\ldots,X_N\}$ where $X_i,i=1,\ldots,N,$ are independent and identically distributed (i.i.d.) samples drawn from $\pi$ on $\mathcal{X}$. 

\cite{chazal-2014-convergence} prove that under a certain assumption on the distribution $\pi$, the persistent homology of $S_N$ consistently approximates the persistent homology of $\mathcal{X}$. This is assumption is the $(a,b,r_0)$-{\em standard assumption} and provides a guarantee that the probability space is ``sufficiently well-behaved.''

\begin{definition}
Let $\mathcal{X}$ be a compact metric space and $\pi$ be a probability distribution. The measure $\pi$ is said the satisfy the $(a,b,r_0)$-{\em standard assumption} if there exists $a,b>0$ and $r_0\ge 0$ such that for any $x\in\mathcal{X}$ and $r>r_0$,
\begin{equation}\label{eq:abr0}
    \pi(\mathcal{B}(x,r))\ge \min(1,ar^b),
\end{equation}
where $\mathcal{B}(x,r)$ is the metric ball centered at $x$ with radius $r$. For $r_0=0$, this is called the $(a,b)$-{\em standard assumption.}
\end{definition}

Intuitively, the parameter $r_0$ distinguishes the setting between continuous spaces and discrete spaces. In fact, when the underlying space $\mathcal{X}$ is discrete, the ball centered around each point $\mathcal{B}(x,r)$ is a singleton when $r$ is smaller than the critical value $r_0$. The volume of the ball is constant and should not increase in the form of $ar^b$. 

For a continuous underlying space, $r_0$ is by default set to be 0. The $(a,b)$-standard assumption is used in random set estimation as a condition to prevent a probability measure from being too singular \citep{cuevas_boundary_2004,set_estimation}. Essentially, the condition implies that the volume (or probability) of a metric ball should grow as a $b$-dimensional Euclidean ball.


For a discrete underlying space, the $(a,b,r_0)$-standard assumption provides a similar intuition as in the continuous case, however, there is a threshold required for \eqref{eq:abr0} to hold. If $\mathcal{X}$ is a discrete set sampled from a probability measure satisfying the $(a,b)$-standard assumption, then the discrete measure on $\mathcal{X}$ satisfies $(a,b,r_0)$-standard assumption with $\displaystyle r_0 = O\bigg(\frac{\log N}{N}\bigg)^{1/b}$, where $N$ is the number of points in $\mathcal{X}$ \citep{chazal-2015-subsampling}. Therefore when $\mathcal{X}$ is a massive, dense point cloud, it is reasonable to assume that $r_0$ is a fixed small constant. The estimation error introduced by $r_0$ is usually considered to be negligible with respect to the estimation error generated by subsampling.   


\paragraph{The One-Sample Case.} Let $D[S_N]$ be the persistence diagram of the VR filtration of the dataset $S_N$ and $D[\mathcal{X}]$ be the persistence diagram of $\mathcal{X}$.  \cite{chazal-2014-convergence} proved the following result, which provides guarantees on how well the persistent homology of a single sample drawn from a large dataset approximates the persistent homology of the large dataset.

\begin{theorem}{\citep[Theorem 2]{chazal-2014-convergence}}
    Suppose $\mathcal{X}$ is a metric space and $\pi$ is any probability measure satisfying the $(a,b)$-standard assumption. Then
    $$
    \mathbb{E}[\mathrm{W}_\infty(D[S_N],D[\mathcal{X}])]\le C(a,b)\bigg(\frac{\log N}{N}\bigg)^{1/b} ,
    $$
    where the constant $C(a,b)$ only depends on $a$ and $b$. 
\end{theorem}

In practice, however, when the size of the sample $N\gg 1$ is very large, it is not feasible to compute the persistent homology of $S_N$. 

\paragraph{Multiple Subsamples.} To bypass the high computational complexity of persistent homology, \cite{chazal-2015-subsampling} adopt the following multiple subsampling method for persistent homology, which we also follow in our work: Instead of sampling a single large data set $S_N$ only once, instead, we may sample $B$ subsets $S_n^{(1)},\ldots,S_n^{(B)}$, each consisting of $n\ll N$ i.i.d.~samples from $\pi$. The persistent homology of $\mathcal{X}$ may then be approximated using the ``average'' persistent homology of $S_n^{(j)},j=1,\ldots,B$. An important challenge here lies in computing the average persistent homology. 

Due to the intrinsic nonlinearity of persistence diagrams, it is technically much easier to compute averages of linear representations of persistence diagrams rather than persistence diagrams themselves. This is the approach taken by \cite{chazal-2015-subsampling}, where they studied the convergence of the average of {\em persistence landscapes} of subsamples drawn from the large dataset, denoted by $\lambda$.  Essentially, as vectorizations of persistent homology, persistence landscapes $\lambda$ are nothing other than functions defined on the real line; more precisely, persistence landscapes are functional representations of persistence diagrams living in Banach space.  Therefore, for each sample set $S_n^{(j)}$, if $\lambda_{S_n^{(j)}}$ is a persistence landscape, we can obtain the mean persistence landscape $\bar{\lambda}=\frac{1}{B}\sum_{j=1}^B\lambda_{S_n^{(j)}}$, by simply taking the pointwise average.  This can be regarded as an approximation of the true persistence landscape $\lambda_{\mathcal{X}}$ of the space $\mathcal{X}$. \cite{chazal-2015-subsampling} showed that the mean persistence landscape, as an approximation, converges to the true persistence landscape at the following rate.
\begin{theorem}{\citep[Theorem 9]{chazal-2015-subsampling}}\label{thm:landscape-subsampling}
    Suppose $\mathcal{X}$ is a metric space and $\pi$ is a probability measure satisfying the $(a,b,r_0)$-standard assumption, then
\begin{equation}\label{eq:landscape-rate}
        \mathbb{E}[\|\bar{\lambda}-\lambda_{\mathcal{X}}\|_\infty]\le r_0 + r_n \bm{1}_{\{r_n>r_0\}} + C_1\frac{r_n}{(\log n)^2} + C_2\frac{1}{\sqrt{B}} ,
\end{equation}
where $\displaystyle r_n=2\bigg(\frac{\log n}{an}\bigg)^{1/b}$, and $C_1,C_2$ are constants that only depend on $a,b$. 
\end{theorem}

Note that \eqref{eq:landscape-rate} holds for both continuous and discrete underlying spaces. For a continuous space, or a dense sampling set where $r_0$ is sufficiently small, the estimation \eqref{eq:landscape-rate} is dominated by
$$
\mathbb{E}[\|\bar{\lambda}-\lambda_{\mathcal{X}}\|_\infty]\le O\bigg(\bigg(\frac{\log n}{n}\bigg)^{1/b}\bigg) + O\bigg(\frac{1}{\sqrt{B}}\bigg) .
$$

\section{Averaging Persistence Diagrams}
\label{sec:ph_means}

The multiple subsampling approach to obtain a representative of persistent homology for a massive dataset involves computing the mean of persistent homology computed from multiple subsamples.  It turns out that the concept of a mean and its computation is not straightforward in persistent homology.  

For vectorized persistence diagrams, the mean is directly computable: the arithmetic mean can be computed due to linearity and a Bochner integral can be computed to give the expectation.  It is currently unknown whether the inverse problem of getting back a single persistence diagram from an average vectorization is solvable. Persistence diagrams can be recovered from persistence landscapes \citep{betthauser2022graded}; the mapping from persistence diagrams to persistence landscapes is invertible. In general, however, average persistence landscapes do not accurately represent standard persistence landscapes that represent persistence diagrams, which means that it is likely that this inverse result is not applicable to the existing subsampling results involving persistence landscapes. Under certain assumptions, it is possible to reconstruct the set of persistence diagrams from its average persistence diagram, i.e., the mapping $(D_1,\ldots,D_B)\mapsto \bar{\lambda}$ is invertible \citep{bubenik2020persistence}. However, this algorithm returns a set of persistence diagrams rather than a single persistence diagram which encodes the statistical information.

\subsection{Fréchet Means of Persistence Diagrams}
A generalization of the mean to arbitrary metric spaces is known as the {\em Fréchet mean}; Fréchet means may be defined and computed in the space of persistence diagrams.  Let $\rho$ be a probability measure on $\mathcal{D}_p$. The {\em Fréchet function} with respect to $\rho$ is 
$$
\mathrm{Fr}_\rho(D) = \int_{\mathcal{D}_p}\mathrm{W}_{p}^{p}(D,Z)\diff{\rho(Z)}. $$ 
The infimum of $\mathrm{Fr}_\rho$ is called the {\em Fréchet variance}, and the set of minimizers achieving the infimum is called the \emph{Fréchet expectation}. When $\hat{\rho}=\frac{1}{B}\sum_{i=1}^B \delta_{D_i}$ is the empirical probability measure supported on a finite set of persistence diagrams $\{D_1,\ldots,D_B\}$, the Fréchet function becomes 
\begin{equation}\label{eq:fre-opt}
    \mathrm{Fr}_{\hat{\rho}}(D) = \frac{1}{B}\sum_{i=1}^B\mathrm{W}_{p}^{p}(D,D_i).
\end{equation}
Any minimizer of $\mathrm{Fr}_{\hat{\rho}}$ is a {\em Fréchet mean} of the set of persistence diagrams $\{D_1,\ldots,D_B\}$. When $p=2$, if $\rho$ has finite second moment and if $\rho$ has compact support, then the Fréchet expectation exists \citep{ohta2012barycenters,mileyko-2011-probability}. 

To compute Fréchet means of persistence diagrams, under the Alexandrov geometry of $(\mathcal{D}_2,\mathrm{W}_{2})$, \cite{turner-2014-frechet} proposed a greedy algorithm to find local minima of the Fréchet function with respect to empirical probability measures, using a variant of the Hungarian algorithm. The Fréchet function is not convex on the space $\mathcal{D}_2$ which means that often only local minima are achieved and it is difficult to find global minima.  Additionally, since the algorithm does not specify an initialization rule and arbitrary persistence diagrams are chosen as starting points, different initializations tend to yield different results resulting in an unstable performance of the algorithm in finding local minima. 
\cite{lacombe2018large} subsequently proposed a method to compute the optimal partial transport distance $\ot_{p}(\cdot,\cdot)$ using entropic regularization and the sinkhorn algorithm, which was used to solve a relaxed version of the optimization problem \eqref{eq:fre-opt}.   
Given a probability distribution on $\mathcal{D}_p$, there is currently no condition to guarantee that the Fréchet expectation is unique (nor on the larger space $\mathcal{M}_p$).  \cite{divol-2019-understanding} proved that for any probability distribution $\rho$ supported on $\mathcal{M}_p$ with finite $p$th moment, the Fréchet expectation with respect to $\rho$ is a nonempty compact convex subset. Furthermore, if $\rho$ is supported on a finite set of persistence diagrams, the Fréchet means form a convex set in $\mathcal{M}_p$ whose extreme points are in $\mathcal{D}_p$. 

\subsection{Mean Persistence Measures}
In the setting of persistence measures, a notion of centrality or mean also exists.  Let $D_1,\ldots,D_B$ be a set of persistence diagrams.   When viewed as persistence measures, the arithmetic mean $\bar{D} = \frac{1}{B}\sum_{i=1}^B D_i$ is simply the empirical mean. For any Borel set in the upper half plane, $A\subseteq \Omega$, the measure $\bar{D}(A)$ gives the average number of points in each persistence diagram within $A$. 

We may obtain a notion for expected persistence diagrams in a similar manner: Let $\bm{D}$ be an $\mathcal{M}_p$-valued random variable with probability distribution $\rho$. We can define the expectation of $\bm{D}$ in a way such that $\mathbb{E}_\rho[\bm{D}]$ is a deterministic Radon measure on $\Omega$, and for any Borel set $A\subseteq \Omega$, 
$$
\mathbb{E}_\rho[\bm{D}](A) = \mathbb{E}_\rho[\bm{D}(A)].
$$ 
When $\bm{D}$ takes values in $\mathcal{D}_p$, the measure $\mathbb{E}_\rho[\mathbf{D}](A)$ gives the expected number of points in persistence diagrams sampled from $\bm{D}$ within $A$. In this case, the expectation $\mathbb{E}_\rho[\bm{D}]$ is the {\em expected persistence diagram} \citep{divol-2021-estimation,divol2019density}.  Note here that as opposed to the setting of Fréchet means, considering persistence measures in $\mathcal{M}_p$ provides a direct approach for computing averages of persistence diagrams.

Rigorously, let $\mathcal{M}_f$ be the space of finite Radon measures on $\Omega$. $\mathcal{M}_f$ is Borel isomorphic to $\mathcal{M}_p$ \citep{divol-2019-understanding}. Let $\mathcal{M}_\pm$ be the space of finite signed Radon measures of bounded variation where the dual bounded Lipschitz norm $\|\cdot\|_{BL_*}$ can be defined. The completion of the space $(\mathcal{M}_\pm,\|\cdot\|_{BL_*})$ is a Banach space \citep{hille2021equivalence}. Every random variable $\bm{D}$ valued in $\mathcal{M}_p$ is also considered as a random variable valued in $\mathcal{M}_\pm$. Since $\|\rho\|_{BL_*}$ is finite for every $\rho\in\mathcal{M}_f$, $\bm{D}$ is Bochner integrable \citep{nonnenmacher1995new}. Thus, the expected persistence diagram $\mathbb{E}[\bm{D}]$ can be identified as the Bochner integral of $\bm{D}$ in $\mathcal{M}_\pm$.

\section{Stability Theorems}
\label{app:stability}

The bounded perturbation of persistence diagrams when the input data are perturbed is a crucial property used in our proofs.  Here, we overview stability results in persistent homology.

The first stability result for persistent homology was established by \cite{cohen2007stability}, who showed that
for two continuous tame functions on a triangulable space, the bottleneck distance between two persistence diagrams is bounded by the max-norm distance between two functions. 
It was generalized to $p$-Wasserstein distance between two persistence diagrams, though in a much more complicated form, by placing additional assumptions on the triangulable space \citep{cohen-2010-lipschitz}. 

\begin{definition}
    A metric space $\mathcal{X}$ is said to have \emph{bounded $k$-total persistence} if there is a constant $C_\mathcal{X}$ that depends only on $\mathcal{X}$ such that $\mathrm{W}_k(D[f],D_\emptyset)\le C_\mathcal{X}$ for every tame function $f: \mathcal{X}\to\mathbb{R}$ with Lipschitz norm $\|f\|_{\mathrm{Lip}}\le 1$.
\end{definition}

From this definition, Wasserstein stability for functions can be stated as follows.

\begin{theorem}{\citep[Section 3]{cohen-2010-lipschitz}}\label{thm:Wasser-functions}
    Let $\mathcal{X}$ be a compact triangulable metric space with bounded $k$-total persistence for some $k\ge 1$. Let $f,g:\mathcal{X}\to\mathbb{R}$ be two tame Lipschitz functions. Then, for any $p\ge k$,
    $$
    \mathrm{W}_p(D[f],D[g])\le C_\mathcal{X}^{\frac{1}{p}}\max\{\|f\|_{\mathrm{Lip}},\, \|g\|_{\mathrm{Lip}}\}\|f-g\|_\infty^{1-\frac{k}{p}}.
    $$
\end{theorem}

These two classical stability theorems were then generalized to many other settings. From an algebraic viewpoint, the bottleneck distance may first be bounded by the interleaving distance between persistence modules, which is then bounded by max-norm distance between functions meaning that it suffices to consider the stability of persistence diagrams with respect to perturbations of persistence modules \citep{chazal2009proximity}. Generalizations to other notions of persistence---including uniparameter, multiparameter, and zigzag---persistence modules have then been established \citep{chazal2016structure,lesnick2015theory,botnan2018algebraic}. Generalized stability in categorical settings have also been studied by \cite{bubenik2014categorification,bubenik2018algebraic,bauer2020persistence}.


In geometric settings, the relevant notion of stability of persistence diagrams is that with respect to the {\em Gromov--Hausdorff distance} between metric spaces. 

\begin{definition}
Let $\mathcal{X}$ and $\mathcal{Y}$ be two sets. A {\em correspondence} is a set $\mathbf{C}\subseteq\mathcal{X}\times\mathcal{Y}$ such that for any $x\in\mathcal{X}$ there is some $y\in\mathcal{Y}$ with $(x,y)\in\mathbf{C}$, and for any $y\in\mathcal{Y}$ there is some $x\in\mathcal{X}$ with $(x,y)\in\mathbf{C}$. The set of all correspondences between $\mathcal{X}$ and $\mathcal{Y}$ is denoted by $\mathbf{C}(\mathcal{X},\mathcal{Y})$. 
\end{definition}

\begin{definition}
Let $(\mathcal{X},d_\mathcal{X})$ and $(\mathcal{Y},d_\mathcal{Y})$ be two compact metric spaces. The {\em Gromov--Hausdorff distance} between them is defined by
\begin{equation}
    \mathrm{GH}((\mathcal{X},d_\mathcal{X}),(\mathcal{Y},d_\mathcal{Y})) = \frac{1}{2}\inf_{\mathbf{C}}\bigg\{\sup_{(x,y),(x',y')\in\mathbf{C}}|d_{\mathcal{X}}(x,x')-d_{\mathcal{Y}}(y,y')|\bigg\},
\end{equation}
where $\mathbf{C}$ ranges over all correspondences in $\mathbf{C}(\mathcal{X}, \mathcal{Y})$. 
\end{definition}

For finite metric spaces, the stability theorem established by \cite{chazal2009gromov} certifies that the bottleneck distance between two persistence diagrams is bounded by the Gromov--Hausdorff distance between two finite metric spaces: 

\begin{theorem}{\cite[Theorem 3.1]{chazal2009gromov}}\label{thm:bottleneck-stability-general}
Let $(\mathcal{X},d_\mathcal{X})$ and $(\mathcal{Y},d_\mathcal{Y})$ be two finite metric spaces. Then
   $$
    \mathrm{W}_{\infty}(D[(\mathcal{X},d_\mathcal{X})],\, D[(\mathcal{Y},d_\mathcal{Y})])\le 2\mathrm{GH}((\mathcal{X},d_\mathcal{X}),(\mathcal{Y},d_\mathcal{Y})).
$$ 
\end{theorem}

As a corollary, when $\mathcal{X}$ and $\mathcal{Y}$ are subsets in Euclidean spaces, we have the following result.

\begin{corollary}\label{thm:bottleneck-stability}
    Let $\mathcal{X},\mathcal{Y}\subseteq \mathbb{R}^m$ be two compact sets, and $D[\mathcal{X}],D[\mathcal{Y}]$ be corresponding persistence diagrams of \v{C}ech or VR filtrations. Then
    $$
\mathrm{W}_\infty(D[\mathcal{X}],\, D[\mathcal{Y}])\le 2\mathrm{H}_\infty(\mathcal{X},\mathcal{Y}),
    $$
    where $\mathrm{H}_\infty$ is the Hausdorff distance in $\mathbb{R}^m$.
\end{corollary}

The generalization of Wasserstein stability to (finite) metric spaces and VR filtrations is much more complicated. As remarked by \cite{skraba-2020-wasserstein}, Theorem \ref{thm:Wasser-functions} appears to be one of the most misunderstood results where common errors or misuse can easily arise, such as applying the theorem for small $p$ on high dimensional data, or applying the theorem to VR filtrations. In the case of finite $p<\infty$ and $p$-Wasserstein distances, to date, there exists no Wasserstein stability result for general (finite) metric spaces.  For finite metric spaces that are Euclidean point clouds, i.e., subsets of an ambient Euclidean space $\mathbb{R}^m$, \cite{skraba-2020-wasserstein} certify that the $p$-Wasserstein distance between persistence diagrams for VR filtrations is bounded by the $p$-Hausdorff distance between two point clouds.

\begin{definition}
Let $\mathcal{X},\mathcal{Y}\subseteq\mathbb{R}^m$ be two finite sets in Euclidean space. For $1\le p<\infty$, the \emph{$p$-Hausdorff distance} between $\mathcal{X}$ and $\mathcal{Y}$ is  
$$
    \mathrm{H}_p(\mathcal{X},\mathcal{Y}) = \inf_{\mathbf{C}}\Bigg(\sum_{(x,y)\in\mathbf{C}}\|x-y\|_2^p\Bigg)^{\frac{1}{p}},
$$
where $\mathbf{C}$ ranges over all correspondences in $\mathbf{C}(\mathcal{X}, \mathcal{Y})$. For $p=\infty$, the \emph{Hausdorff distance} between $\mathcal{X}$ and $\mathcal{Y}$ is
$$
\mathrm{H}_\infty(\mathcal{X},\mathcal{Y}) = \inf_\mathbf{C}\sup_{(x,y)\in\mathbf{C}}\|x-y\|_2,
$$
where $\mathbf{C}$ ranges over all correspondences in $\mathbf{C}(\mathcal{X}, \mathcal{Y})$.

\end{definition}

\begin{theorem}{\cite[Theorem 6.10]{skraba-2020-wasserstein}}
\label{thm:wasserstein-point-clouds}
Let $\mathcal{X},\mathcal{Y}\subseteq\mathbb{R}^m$ be two finite sets in Euclidean space, and $D[\mathcal{X}],D[\mathcal{Y}]$ be corresponding persistence diagrams of VR filtrations. Assume the number of points in both sets are bounded by $N$, and $1\le p<\infty$. There exists a constant $C_N$ that only depends on $N$ such that
\begin{equation}
\label{eq:wass_stability}
    \mathrm{W}_p(D[\mathcal{X}],D[\mathcal{Y}])\le C_N^{1/p}\mathrm{H}_p(\mathcal{X},\mathcal{Y}).
\end{equation}
\end{theorem}


The assumption of a uniform bound $N$ on the number of points can be relaxed from Theorem \ref{thm:wasserstein-point-clouds} which then results in a constant depending on the dimension of the Euclidean space and the dimension of homology. However, the question of whether the constant is finite for all dimensions is still open \citep[Conjecture 6.26]{skraba-2020-wasserstein}. In our work, we assume a uniform bound on the number of points. 



\section{Technical Proofs for H\"{o}lder Continuous Vectorizations}
\label{app:holder}



The following two properties of persistence landscapes are key for us to generalize in our proof:
\begin{itemize}
     \item \textbf{Central Limit Theorem in Banach Spaces}: The sampling of data $S_n^{(1)},\ldots,S_n^{(B)}$ induces a sampling of $\lambda_{S_n^{(1)}},\ldots, \lambda_{S_n^{(B)}}$ on the space of persistence landscapes from the function-valued random variable $\bm{\lambda}$. The mean persistence landscape $\bar{\lambda}$ converges to the population persistence landscape $\mathbb{E}[\bm{\lambda}]$ due to the LLN and the CLT on Banach spaces \citep{hoffmann1976law}.
      \begin{theorem}{\citep{hoffmann1976law}}\label{thm:CLT-banach}
          Assume $(\mathcal{V},\|\cdot\|)$ is a Banach space of type 2. Let $\bm{x}$ be a $\mathcal{V}$-valued random variable such that $\mathbb{E}[\bm{x}]=0$ and $\mathbb{E}[\|\bm{x}\|^2]<\infty$, and let $x_1,\ldots,x_B$ be i.i.d.~samples from $\bm{x}$, then $\frac{1}{\sqrt{B}}(x_1+\ldots+x_B)$ converges weakly to a Gaussian random variable with the same covariance structure as $\bm{x}$. 
     \end{theorem}
    
     \item \textbf{Stability of Persistence Landscapes}: Let $D$ and $D'$ be two persistence diagrams and $\lambda$ and $\lambda'$ be the corresponding persistence landscapes.  It is known that
     $\|\lambda-\lambda'\|_\infty\le\mathrm{W}_\infty(D,D')$ \citep{bubenik2015statistical}, so estimating the bias reduces to estimating the Hausdorff distance via the stability of persistence diagrams (Theorem \ref{thm:bottleneck-stability}).  
 \end{itemize}

Adapting these observations to the case of HCVs allows us to extend the results of \cite{chazal-2015-subsampling}.  The technical details for our adaptation are given in this section.

We begin by providing a technical background on the CLT and LIL in Banach spaces which allow us to arrive at our results on the variance component for HCVs. Then we adapt the proof in \cite{chazal-2015-subsampling} to bias estimation of HCVs. We finish with a technical discussion on examples of HCVs.

\subsection{Background on Limit Theorems in Banach Spaces}
\label{app:holder_banach}

Let $(\mathcal{V},\|\cdot\|)$ be a separable Banach space and $(\mathcal{V}',\|\cdot\|')$ be its dual Banach space. Let $\bm{Z}$ be an $\mathcal{V}$-valued random variable. Without loss of generality we assume $\bm{Z}$ has zero mean. Let $\bm{Z}_1,\ldots,\bm{Z}_B$ be $B$ i.i.d.~copies of $\bm{Z}$ and set $\bm{W}_B:=\bm{Z}_1+\ldots+\bm{Z}_B$. For Banach space-valued random variables, we have the following definitions.

\begin{definition}
\begin{enumerate}
    \item A $\mathcal{V}$-valued random variable $\bm{G}$ is \emph{Gaussian} if for any $f\in \mathcal{V}'$, $f(\bm{G})$ is a real valued Gaussian random variable;
    \item A $\mathcal{V}$-valued random variable $\bm{Z}$ \emph{satisfies the CLT} if there exists a Gaussian, $\mathcal{V}$-valued random variable $\bm{G}$ with zero mean such that $\bm{W}_B/\sqrt{B}$ converges to $\bm{G}$ in distribution as $B\to \infty$.
\end{enumerate}
\end{definition}

In the classical setting, a common way to determine whether $\bm{Z}$ satisfies the CLT is to use the moment conditions $\mathbb{E}[\bm{Z}]=0$ and $\mathbb{E}[\|\bm{Z}\|^2]<\infty$: if $\mathcal{V}$ is finite dimensional, the moment conditions are sufficient and necessary conditions for $\mathcal{V}$ to satisfy the classical CLT \citep[Theorem 27.1]{billingsley2012probability}. For general Banach spaces, the moment conditions are neither sufficient nor necessary. In fact, if every $\mathcal{V}$-valued random variable satisfying the moment conditions also satisfies the CLT, the Banach space $\mathcal{V}$ is restricted to a special type.

\begin{definition}
    A Banach space $\mathcal{V}$ is \emph{of type} $p\in[1,2]$ if there exists a constant $C_{\mathcal{V}}>1$ such that for any sequence $v_1,\ldots,v_n\in\mathcal{V}$,
    $$
    \mathbb{E}[\|\epsilon_1v_1+\cdots+\epsilon_nv_n\|^p]\le C^p_{\mathcal{V}}(\|v_1\|^p+\cdots+\|v_n\|^p),
    $$
    where $\epsilon_1,\ldots,\epsilon_n$ is a sequence of independent Rademacher random variables, i.e., $\mathbb{P}[\epsilon_i=1]=\mathbb{P}[\epsilon_i=-1]=1/2$.
\end{definition}

This definition gives rise to the following result, which indicates what type of Banach space we are working in when Banach space-valued random variables satisfy the moment conditions in the classical sense.

\begin{theorem}{\cite[Theorem 3.5]{hoffmann1976law}}\label{thm:type-2}
    A Banach space $\mathcal{V}$ is of type 2 if and only if any $\mathcal{V}$-valued random variable $\bm{Z}$ with $\mathbb{E}[\bm{Z}]=0$ and $\mathbb{E}[\|\bm{Z}\|^2]<\infty$ satisfies the CLT.
\end{theorem}

An example of a nontrivial Banach space of type 2 is the class of $L^p$ spaces for $1\le p<\infty$. In the case of the space of continuous functions over compact metric spaces $\mathcal{V}=\mathcal{C}(\mathcal{Y})$, the Banach space is not of any type $p$ \citep{hoffmann1976law}. This restricts the direct application of Theorem \ref{thm:type-2}. However, it turns out that the moment conditions are sufficient on a subspace of Lipschitz functions over metric spaces.

\begin{definition}
    Let $(\mathcal{Y},d_{\mathcal{Y}})$ be a metric space. For $r>0$, the covering number  $\mathrm{cv}(\mathcal{Y},r)$ is the fewest balls of radius $r$ needed to cover $\mathcal{Y}$. Its \emph{metric entropy} function is defined by
    $$
    \mathrm{Ent}(\mathcal{Y},r) = \log \big(\mathrm{cv}(\mathcal{Y},r)\big).
    $$
\end{definition}

For the subset of Lipschitz functions over metric spaces with finite entropy, we have the following result specifying the sufficiency of moment conditions to satisfy the CLT.

\begin{theorem}{\citep[Theorem 1]{jain1975central}}\label{thm:CLT-CSspace}
    Let $\mathcal{V}=\mathcal{C}(\mathcal{Y})$ and $\bm{Z}$ be an $\mathcal{V}$-valued random variable such that $\mathbb{E}[\bm{Z}]=0$ and $\mathbb{E}[\|\bm{Z}\|^2]<\infty$. Suppose there is a nonnegative real valued random variable $\bm{L}$ such that $\mathbb{E}[\bm{L}^2]<\infty$, and for any $y,y'\in\mathcal{Y}$,
    $
    |\bm{Z}(y)-\bm{Z}(y')|\le \bm{L} F(d_{\mathcal{Y}}(y,y'))
    $
    for some strictly increasing function $F:[0,\infty)\to[0,\infty)$. Moreover, suppose that the metric entropy function of $\mathcal{Y}$ is such that
    $$
    \int_0^\infty \mathrm{Ent}^{\frac{1}{2}}(\mathcal{Y},F^{-1}(r))\,\diff{r}<\infty .
    $$
    Then $\bm{Z}$ satisfies the CLT.
\end{theorem}

The CLT is closely related to anther important limit theorem, namely, the Law of the Iterated Logarithm (LIL); intuitively, the latter provides upper and lower bounds for the former. There exist different versions of the LIL for Banach spaces; to derive a variance expression, we require a (compact) LIL in the sense of Strassen \citep{strassen1964invariance}.  Let $\mathrm{Log}(x)$ denote the function $\max\{1,\log x\}$ and let $\mathrm{LLog}(x) = \mathrm{Log}(\mathrm{Log}(x))$. For any $B\in\mathbb{N}$, set
$$
s_B=\sqrt{2B\mathrm{LLog}(B)}, \quad t_B = \frac{s_B}{B}=\frac{\sqrt{2\mathrm{LLog}(B)}}{\sqrt{B}}
$$
\begin{definition}
    Let $\bm{Z}$ be an $\mathcal{V}$-valued random variable with zero mean, and let $\bm{W}_B=\bm{Z}_1+\ldots+\bm{Z}_B$ be the sum of i.i.d. $\mathcal{V}$-valued random variables with the same distribution as $\bm{Z}$. The $\mathcal{V}$-valued random variable $\bm{Z}$ is said to satisfy the LIL if the sequence $\{\bm{W}_B/s_B\}$ is almost surely relatively compact in $\mathcal{V}$.
\end{definition}
The following theorem gives an equivalent definition for $\bm{Z}$ to satisfy the LIL.

\begin{theorem}{\cite[Theorem 8.5]{ledoux2013probability}}\label{thm:LIL-consequence}
    If $\bm{Z}$ satisfies the LIL, then with probability one,
    $$
    \lim_{B\to\infty}\mathrm{dist}\bigg(\frac{\bm{W}_B}{s_B},K\bigg)=0 \text{ \ and \ } \mathrm{CP}\bigg(\bigg\{\frac{\bm{W}_B}{s_B}\bigg\}_{B=1}^\infty\bigg)=K,
    $$
    where $K$ is a compact set in $\mathcal{V}$, $\mathrm{dist}(x,K)=\inf_{y\in K}\|x-y\|$ is the distance function, and $\mathrm{CP}(\cdot)$ denotes the set of all limiting points. Conversely, if the preceding holds for some compact set $K$, then $\bm{Z}$ satisfies the LIL.
\end{theorem}

The limit set $K$ is known to coincide with the unit ball of the reproducing kernel Hilbert space $\mathcal{H}\subset \mathcal{V}$ associated to the covariance of $\bm{Z}$; see \cite{banholzer2022rates} for further details. Here, we only need the following fact concerning the set $K$.

\begin{proposition}{\citep[Section 2.2]{banholzer2022rates}}\label{prop:lambdaZ}
    Let 
    $$
    \Lambda(\bm{Z})=\limsup_{B\to\infty}\frac{\|W_B\|}{s_B}.
    $$
    If $\bm{Z}$ satisfies the LIL, then
    $$
    \Lambda(\bm{Z})=\sup_{y\in K}\|y\|=\sup_{\substack{f\in \mathcal{V}',\\ \|f\|'\le 1}}\big(\mathbb{E}[f^2(\bm{Z})]\big)^{\frac{1}{2}}.
    $$
\end{proposition}

In general, a $\mathcal{V}$-valued random variable $\bm{Z}$ satisfying the CLT does not necessarily satisfy the LIL. However, given the moment conditions $\mathbb{E}[\bm{Z}]=0$ and $\mathbb{E}[\|\bm{Z}\|^2]<\infty$, satisfying the CLT implies satisfying the LIL \citep[Section 2.2]{banholzer2022rates}. Thus we can restate Theorem \ref{thm:CLT-CSspace} for the LIL.

\begin{theorem}\label{thm:LIL-CSspace}
    Given the same conditions as in Theorem \ref{thm:CLT-CSspace}, $\bm{Z}$ satisfies the LIL.
\end{theorem}

We are now ready to prove the convergence of the variance component for approximations of persistent homology in the form of H\"{o}lder continuous vectorizations of persistence diagrams.

\subsection{Variance Estimation}
\label{app:holder_var_proofs}

We first show that the assumption (A1) of HCV implies finiteness of the moments of $\Phi$.

\begin{proof}[Proof of Proposition \ref{prop:moment-control}]
    Let $\mathrm{diam}(\mathcal{Y})=\sup_{y,y'\in\mathcal{Y}}d_{\mathcal{Y}}(y,y')$ be the diameter of the compact space $\mathcal{Y}$. For any $y\in\mathcal{Y}$, by assumption (A1),
    \begin{equation}
    \label{eq:diam}
    |\bm{\phi}(y)|\le |\bm{\phi}(y_0)|+C_{\mathcal{Y}}d_{\mathcal{Y}}(y,y_0)^{\alpha}\le |\bm{\phi}(y_0)|+C_{\mathcal{Y}}\mathrm{diam}(\mathcal{Y})^{\alpha}.
    \end{equation}
    Thus,
    $$
    \mathbb{E}[\|\bm{\phi}\|_\infty^2]\le \mathbb{E}\bigg[\bigg(|\bm{\phi}(y_0)|+C_{\mathcal{Y}}\mathrm{diam}(\mathcal{Y})^{\alpha}\bigg)^2\bigg]\le 2\mathbb{E}[|\bm{\phi}(y_0)|^2]+2C_\mathcal{Y}^2\mathrm{diam}(\mathcal{Y})^{2\alpha}<\infty.
    $$
\end{proof}

In fact, if $\Phi$ is a H\"older continuous function and if we assume for some $y_0\in\mathcal{Y}$ there exists some constant $C_{y_0}>0$ such that $\phi(y_0)\le C_{y_0}$ for all $\phi\in\mathcal{Y}$, then \eqref{eq:diam} in the proof in fact becomes
$$
|\phi(y)|\le |\phi(y_0)|+C_\mathcal{Y}\mathrm{diam}(\mathcal{Y})^{\alpha}\le C_{y_0}+C_\mathcal{Y}\mathrm{diam}(\mathcal{Y})^{\alpha},
$$
and we then see that $\phi$ is uniformly bounded on $\mathcal{Y}$ and that $\mathbb{E}[\|\bm{\phi}\|_\infty^p]<\infty$ for all $p\in\mathbb{N}$.\\

We now prove our result for the variance estimate of HCVs.

\begin{proof}[Proof of Lemma \ref{lemma:var-rate-vectorization}]
Let $\widetilde{\bm{\phi}}=\bm{\phi}-\mathbb{E}[\bm{\phi}]$. By Proposition \ref{prop:moment-control}, we have $\mathbb{E}[\widetilde{\bm{\phi}}]=0$ and $\mathbb{E}[\|\widetilde{\bm{\phi}}\|_\infty^2]<\infty$. Since $\Phi$ is H\"older continuous, we can take $F(t)=t^{\alpha}$ and by Theorem \ref{thm:LIL-CSspace}, $\widetilde{\bm{\phi}}$ satisfies the LIL. Let 
    $$
    \widetilde{\bm{W}}_B=\sum_{i=1}^B\widetilde{\bm{\phi}_i}=\sum_{i=1}^B\bm{\phi}_i-B\mathbb{E}[\bm{\phi}]=B(\bar{\phi}-\mathbb{E}[\bm{\phi}])
    $$
    so that 
    $$
    \frac{\widetilde{\bm{W}}_B}{s_B}=\frac{B(\bar{\phi}-\mathbb{E}[\bm{\phi}])}{s_B}=\frac{\bar{\phi}-\mathbb{E}[\bm{\phi}]}{t_B}.
    $$
    By Theorem \ref{thm:LIL-consequence}, with probability one, there exists a compact $K\subseteq \mathcal{V}$ such that 
    $$
    \lim_{B\to\infty}\mathrm{dist}\bigg(\frac{\bm{W}_B}{s_B},K\bigg)=0, \text{ and } \mathrm{CP}\bigg(\bigg\{\frac{\bm{W}_B}{s_B}\bigg\}_{B=1}^\infty\bigg)=K ,
    $$
    and by Proposition \ref{prop:lambdaZ}, 
    $$
    \Lambda(\widetilde{\bm{\phi}})=\limsup_{B\to\infty}\frac{\|\widetilde{\bm{W}}_B\|_\infty}{s_B}=\sup_{\phi\in K}\|\phi\|_\infty<\infty
    $$
    Therefore, with probability one,
    $$
    \limsup_{B\to\infty}\frac{\|\bar{\phi}-\mathbb{E}[\bm{\phi}]\|_\infty}{t_B}=\Lambda(\widetilde{\bm{\phi}})
    $$
    By definition of supremum limit we know for any $\epsilon>0$ there exists $B(\epsilon)$ such that for any $B>B(\epsilon)$,
    $$
    \|\bar{\phi}-\mathbb{E}[\bm{\phi}]\|_\infty\le (\Lambda(\widetilde{\bm{\phi}})+\epsilon)t_B.
    $$
    Let $\epsilon=1$ and set $C_K=\Lambda(\widetilde{\bm{\phi}})+1$. Then with probability one, there exists an integer $B_0$ such that for any $B>B_0$,  we have
    $$
    \mathbb{E}[\|\bar{\phi}-\mathbb{E}[\bm{\phi}]\|_\infty]\le C_K\frac{\sqrt{2\mathrm{LLog}(B)}}{\sqrt{B}}.
    $$
\end{proof}

We remark that the assumption of finiteness of the metric entropy in Lemma \ref{lemma:var-rate-vectorization} is not restrictive. In practice, the pushforward distribution of HCV often has compact support in some Euclidean space $\mathbb{R}^m$. For a unit ball $\mathcal{B}(1)\subseteq \mathbb{R}^m$, the metric entropy rate is
$\mathrm{Ent}(\mathcal{B}(1),r)\approx \log(\frac{1}{r})$
for any norm defined on $\mathbb{R}^m$ \citep[Example 5.8]{wainwright2019high}. Thus the integration 
$$
\int_{0}^1\mathrm{Ent}^{\frac{1}{2}}(\mathcal{B}(1),r^{\frac{1}{\alpha}})\diff{r}\approx \int_0^\infty \frac{1}{\sqrt{\alpha}}r^{\frac{1}{2}}e^{-r}\diff{r}<\infty
$$
is finite for any $0<\alpha\le 1$.

\subsection{Bias Estimation}
\label{app:holder_bias_proofs}

The bias component is bounded as follows:
\begin{equation}\label{eq:vec-bias1}
    \|\mathbb{E}[\bm{\phi}]-\phi_\mathcal{X}\|_{\infty}\le \mathbb{E}[\|\bm{\phi}-\phi_{\mathcal{X}}\|_{\infty}]\le C_{\Phi}\mathbb{E}[\mathrm{W}_\infty^\beta(D[\bm{S}_n],\, D[\mathcal{X}])],
\end{equation}
where the first inequality follows from the convexity of norm $\|\cdot\|_\infty$ and Jensen's inequality, the second inequality follows from the H\"{o}lder continuity. Futhermore, by the bottleneck stability of persistence diagrams, \eqref{eq:vec-bias1} is bounded by the Hausdorff distance
\begin{equation}
    \mathbb{E}[\mathrm{W}_\infty^\beta(D[\bm{S}_n],\, D[\mathcal{X}])]\le 2\mathbb{E}[\mathrm{H}_\infty^\beta(\bm{S}_n,\mathcal{X})] .
\end{equation}

\cite{chazal-2015-subsampling} give the following tail bound on Hausdorff distance.

\begin{lemma}{\citep[Lemma 14]{chazal-2015-subsampling}}\label{lemma:chazal}
    Let $\bm{S}_n$ be a random set of size $n$ from a measure $\pi$ on $\mathcal{X}$ that satisfies the $(a,b,r_0)$-standard assumption. For any $r>r_0$ we have
    $$
    \mathbb{P}[\mathrm{H}_\infty(\bm{S}_n,\mathcal{X})>r]\le \frac{4^b}{ar^b}\exp\left(-\frac{nar^b}{2^b}\right) ,
    $$
\end{lemma}

We then prove the following bias estimation.

 \begin{proof}[Proof of Lemma \ref{lemma:vec-bias-rate}]
    Notice that
\begin{equation*}
\mathbb{E}[\mathrm{H}^\beta_\infty(\bm{S}_n, \mathcal{X})]  = \int_{t>0}\mathbb{P}(\mathrm{H}^\beta_\infty(\bm{S}_n, \mathcal{X}) > t) \diff t 
= \beta \int_{r>0} \mathbb{P}(\mathrm{H}_\beta(\bm{S}_n, \mathcal{X}) > r)r^{\beta-1} \diff r.
\end{equation*}
If $r_0>r_n$, then
\begin{equation*}
    \int_{r>0} \mathbb{P}(\mathrm{H}_\infty(\bm{S}_n, \mathcal{X})>r)r^{\beta-1} \diff r  = \Bigg( \int_0^{r_0} + \int_{r_0}^\infty \Bigg) \mathbb{P}(\mathrm{H}_\infty(\bm{S}_n, \mathcal{X})>r)r^{\beta-1} \diff r
\end{equation*}
Otherwise we have
\begin{equation*}
    \int_{r>0} \mathbb{P}(\mathrm{H}_\infty(\bm{S}_n, \mathcal{X})>r)r^{\beta-1} \diff r  = \Bigg( \int_0^{r_0} + \int_{r_0}^{r_n}+\int_{r_n}^\infty \Bigg) \mathbb{P}(\mathrm{H}_\infty(\bm{S}_n, \mathcal{X})>r)r^{\beta-1} \diff r
\end{equation*}
In summary we have
\begin{equation*}
    \int_{r>0} \mathbb{P}(\mathrm{H}_\infty(\bm{S}_n, \mathcal{X})>r)r^{\beta-1} \diff r  \le \frac{r_0^\beta}{\beta}+ \frac{r_n^\beta}{\beta}\mathbf{1}_{\{r_n>r_0\}}+ \int_{r_n}^\infty \mathbb{P}(\mathrm{H}_\infty(\bm{S}_n, \mathcal{X})>r)r^{\beta-1} \diff r
\end{equation*}
By Lemma \ref{lemma:chazal}, the integral is bounded by
\begin{equation*}
\int_{r_n}^\infty \mathbb{P}(\mathrm{H}_\infty(\bm{S}_n, \mathcal{X})>r)r^{\beta-1} \diff r 
\leq \int_{r_n}^\infty \frac{4^br^{\beta-1}}{ar^b} \exp\bigg( -\frac{nar^b}{2^b} \bigg) \diff r.
\end{equation*}
Applying a change of variables the integral simplifies to
\begin{equation*}
\int_{r_0}^\infty \frac{4^b}{ar^b} \exp\bigg( -\frac{nar^b}{2^b} \bigg) \diff r = \frac{2^{\beta+b}}{ba^{\beta/b} n^{\beta/b-1}} \int_{\log n}^\infty v^{\beta/b-2}e^{-v}\diff v.
\end{equation*}
Notice that $v^{\beta/b-2}e^{-v}$ is montone decreasing, so
$$
\int_{\log n}^\infty v^{\beta/b-2}e^{-v} \diff v  \leq \frac{(\log n)^{\beta/b-2}}{n}.
$$
Set $C_1 = \frac{2^{\beta+b}}{ba^{\beta/b}}$, we have proved the claim.
 \end{proof}

\subsection{Examples of HCVs}
\label{app:holder_class}

The following classes of vectorizations  persistence diagrams turn out to be HCVs.

\paragraph{Lipschitz Continuous Vectorizations (LCVs).}
These are special cases of HCVs with exponent $\alpha=1$ and possess the same convergence rate as persistence landscapes, derived by \cite{chazal-2015-subsampling}. Other LCVs include \emph{persistence complex vectors} \citep{di2015comparing}, \emph{tropical coordinate functions} \citep[Theorem 5.1]{kalivsnik2019tropical}, and \emph{integrated landscape embeddings} \citep[Theorem 3]{chevyrev2018persistence}.
    
    
\paragraph{Linear Representations of Persistent Homology.}
Since the space of persistence diagrams $\mathcal{D}$ is a proper subspace of the space of persistence measures $\mathcal{M}$, it is reasonable to consider linear maps $\Phi:\mathcal{D}\to\mathcal{V}$. Fix a function $\xi:\Omega\to \mathcal{V}$. A \emph{linear representation} of a persistence diagram $D$ is defined as
    $$
    \Phi(D) := D(\xi) = \int\xi(x)\diff{D}(x) ,
    $$
    by integrating with respect to the measure $D$. To establish stability of such linear representations, it is necessary to add weights to points in a persistence diagram. \cite{divol2019choice} considered the following class of weight functions: Let $\Tilde{\omega}:\mathbb{R}_+\to\mathbb{R}_+$ be a differentiable function such that $\Tilde{\omega}(0)=0$, and for some $A>0$, $\kappa\ge 1$, $|\Tilde{\omega}'(t)|\le At^{\kappa-1}$. An admissible weight function $\omega:\Omega\to\mathbb{R}_+$ is of the form $\omega((x,y))=\Tilde{\omega}(y-x)$ for some fixed function $\Tilde{w}$. The set of all admissible weight functions is denoted by $\mathcal{W}(\kappa,A)$. For linear representations with weights, \cite{divol2019choice} proved a general stability result with respect to any $p$-Wasserstein distance for $1\le p\le \infty$. We reformulate the stability result in terms of bottleneck distance. 
    \begin{theorem}{\citep[Theorem 3]{divol2019choice}}\label{thm:linear-rep-stability}
        Let $\xi:\Omega\to\mathcal{V}$ be a Lipschitz function. For $\omega\in\mathcal{W}(\kappa,A)$ with $A>0$, $\kappa\ge 1$, let $\Phi_\omega(D)=D(\omega\xi)$. For two persistence diagrams $D$ and $D'$, let $G=\max\{|y-x|:(x,y)\in D\bigsqcup D'\}$. Then for $0\le \alpha\le 1$,
        $$
        \|\Phi_\omega(D)-\Phi_\omega(D')\|_\infty\le \frac{\|\xi\|_{\mathrm{Lip}}AG}{\alpha}\mathrm{W}_\infty(D,D')+2\|\xi\|_\infty AG\mathrm{W}_\infty^\alpha(D,D') .
        $$
    \end{theorem}
    Therefore, locally, linear representations satisfying conditions in Theorem \ref{thm:linear-rep-stability} are H\"{o}lder continuous. Linear representations are common in the literature on TDA, especially involving vectorizations such as persistence surfaces and persistence images \citep{adams-2017-persistence}; persistence kernel embeddings \citep{kusano2017kernel}; and accumulated persistence functions \citep{biscio2019accumulated}. Note, however, that persistence landscapes are not linear representations.







\section{Technical Proofs for Persistence Measures}

\subsection{Case I: $\mathcal{X}$ Finite}

To bound the optimal partial transport distance by the $p$-Hausdorff distance, we make use of the following result.

\begin{proposition}{\cite[Proposition  5.4]{divol-2019-understanding}}
\label{prop:wass_convexity}
Let $\bm{D}$ and $\bm{D}'$ be two $\mathcal{M}_p$-valued random variables with finite moments, and let $D_\mu$ and $D'_\mu$ be corresponding expected persistence measures (population means).  Then
\begin{equation}
\label{eq:convex}
\mathrm{OT}^p_p(D_\mu, D'_{\mu}) \leq \mathbb{E}[\mathrm{OT}^p_p(\bm{D}, \bm{D}')].
\end{equation}
\end{proposition}

In essence, Proposition \ref{prop:wass_convexity} establishes convexity of the optimal partial transport distance and shows that this distance admits an inequality akin to Jensen's inequality.

\begin{proof}[Proof of Proposition \ref{prop:bias-control}]


Taking $\bm{D}'$ to be the Dirac measure at a fixed persistence measure $\nu_0 \in \mathcal{M}_p$, (\ref{eq:convex}) becomes $\ot^p_p(D_\mu, \nu_0) \leq \mathbb{E}\big[ \ot^p_p(\bm{D}, \nu_0) \big].$  Notice that our bias expression of interest takes precisely this form, which gives us the following inequality:
\begin{equation}
\label{eq:mpm_convexity}
\ot^p_p(D_\mu, D[\mathcal{X}]) \leq \mathbb{E}\big[ \ot^p_p(\bm{D}_n, D[\mathcal{X}]) \big].
\end{equation}

Since $\bm{D}_n$ is valued in $\mathcal{D}_p$ and the optimal partial transport distance $\ot_p$ coincides with the Wasserstein $p$-distance $\mathrm{W}_p$ on $\mathcal{D}_p$, so by Theorem \ref{thm:wasserstein-point-clouds} for point clouds, for the right-hand expression of (\ref{eq:mpm_convexity}), we have
\begin{equation}
\label{eq:mpm_wass_stability}
\mathbb{E}\big[ \ot^p_p(\bm{D}_n, D[\mathcal{X}]) \big] = \mathbb{E}\big[ \mathrm{W}_p^p(\bm{D}_n, D[\mathcal{X}]) \big] \leq C_N\mathbb{E}[\mathrm{H}_p^p(\bm{S}_n, \mathcal{X})].
\end{equation}
Summarizing above, we have proved \eqref{eq:ineqs}.
\end{proof}

For a metric space $\mathcal{X}$ and a fixed radius $r > 0$, the {\em covering number} $\mathrm{cv}(\mathcal{X}, r)$ is the fewest balls of radius $r$ needed to cover $\mathcal{X}$.  For a probability measure $\pi$ on $\mathcal{X}$ satisfying the $(a,b,r_0)$-standard assumption, we have the following estimate for the covering number.

\begin{lemma}{\citep[Lemma 10]{chazal-2014-convergence}}
\label{lem:chazal_cv}
Assume that the probability measure $\pi$ satisfies the $(a,b,r_0)$-standard assumption, then for $r>r_0$, the covering number of $\mathcal{X}$ is bounded as follows:
$$
\mathrm{cv}(\mathcal{X}, r) \leq \max\bigg( \frac{2^b}{ar^b}, 1 \bigg).
$$
\end{lemma}

We then derive a tail bound for the $p$-Hausdorff distance.

\begin{proof}[Proof of Lemma \ref{lem:tail_prob}]

Let $\hat r = r/(2N^{1/p}) > r_0$ and let $\mathcal{U}$ be a subset of $\mathcal{X}$ with covering number $\mathrm{cv}(\mathcal{X}, \hat r)$.  By the triangle inequality, we have
\begin{align}
\mathbb{P}(\mathrm{H}_p(\bm{S}_n, \mathcal{X})>r) & \leq \mathbb{P}(\mathrm{H}_p(\bm{S}_n, \mathcal{U}) + \mathrm{H}_p(\mathcal{U}, \mathcal{X})>r) \nonumber\\
& \leq \mathbb{P}\bigg( \mathrm{H}_p(\bm{S}_n, \mathcal{U}) > \frac{r}{2} \bigg) + \mathbb{P}\bigg( \mathrm{H}_p(\mathcal{U}, \mathcal{X}) > \frac{r}{2} \bigg). \label{eq:triangle_cv}
\end{align}

For the second term of (\ref{eq:triangle_cv}), the cardinality of $\mathcal{U}$ is the covering number $\mathrm{cv}(\mathcal{X},\hat r)$.   Consider a correspondence $\mathbf{C}_2 \subseteq \mathcal{U} \times \mathcal{X}$ that assigns each point in $\mathcal{U}$ to itself and each point in $\mathcal{X} - \mathcal{U}$ to a point in $\mathcal{U}$.  The cardinality of $\mathbf{C}_2$ is $N$ and we have
$$
\mathrm{H}_p(\mathcal{U}, \mathcal{X}) \leq \Bigg( \sum_{(x,y) \in \mathbf{C}_2} \| x-y \|^p_2 \Bigg)^{1/p} \leq (N\hat r^p)^{1/p} = \frac{r}{2},
$$
which means that the second term of (\ref{eq:triangle_cv}) vanishes.

It therefore suffices to bound the first probability in (\ref{eq:triangle_cv}).  For all $i \in \{1,2,\ldots, \mathrm{cv}(\mathcal{X}, \hat r)\}$, assume that the ball $\mathcal{B}(u_i, \hat r)$, $u_i \in \mathcal{U}$, contains a point of $\bm{S}_n$.  Then consider the correspondence $\mathbf{C}_1 \subseteq \bm{S}_n \times \mathcal{U}$ that assigns each point in $\bm{S}_n$ to a point in $\mathcal{U}$ and each unmatched point $u_i \in \mathcal{U}$ to a point in $\bm{S}_n \cap \mathcal{B}(u_i, \hat r)$.  The cardinality of the correspondence $\mathbf{C}_1$ is at most $\mathrm{cv}(\mathcal{X}, \hat r)$, where then
$$
\mathrm{H}_p(\bm{S}_n, \mathcal{U}) \leq \mathrm{cv}(\mathcal{X}, \hat r)^{1/p}\hat r \leq N^{1/p} \hat r \leq \frac{r}{2}.
$$
Therefore, the first probability in (\ref{eq:triangle_cv}) is
\begin{align}
\mathbb{P}\bigg( \mathrm{H}_p(\bm{S}, \mathcal{U}) > \frac{r}{2} \bigg) & = \mathbb{P}(\exists\ i \in \{1,2,\ldots, \mathrm{cv}(\mathcal{X}, \hat r)\}: \bm{S}_n \cap \mathcal{B}(u_i, \hat r) = \emptyset) \nonumber\\
& \leq \sum_{i=1}^{\mathrm{cv}(\mathcal{X},\hat r)} \mathbb{P}(\bm{S}_n \cap \mathcal{B}(u_i, \hat r) = \emptyset) \nonumber\\
& \leq \mathrm{cv}(\mathcal{X}, \hat r)(1-a\hat r^b)^n. \label{eq:cv_bound}
\end{align}
By Lemma \ref{lem:chazal_cv}, (\ref{eq:cv_bound}) is bounded by
$$
\frac{2^b}{a\hat r^b}(1-a\hat r^b)^n \leq \frac{2^b}{a\hat r^b}\exp(-na\hat r^b) = \frac{4^bN^{b/p}}{ar^b}\exp\bigg(-\frac{ar^bn}{2^b N^{b/p}} \bigg),
$$
as desired.
\end{proof}

\begin{proof}[Proof of Lemma \ref{lemma:hausdorff_bound}]
By integrating tail probabilities, we have
\begin{align}
\mathbb{E}[\mathrm{H}^p_p(\bm{S}_n, \mathcal{X})] & = \int_{t>0}\mathbb{P}(\mathrm{H}^p_p(\bm{S}_n, \mathcal{X}) > t) \diff t \nonumber\\
& = \int_{t>0} \mathbb{P}(\mathrm{H}_p(\bm{S}_n, \mathcal{X})>t^{1/p})\diff t \nonumber\\
& \myeq \ \ \ p \int_{r>0} \mathbb{P}(\mathrm{H}_p(\bm{S}_n, \mathcal{X}) > r)r^{p-1} \diff r. \label{eq:tailprob_haus_bound_1}
\end{align}

By Lemma \ref{lem:tail_prob}, (\ref{eq:tailprob_haus_bound_1}) is bounded as follows:
\begin{align}
\int_{r>0} \mathbb{P}(\mathrm{H}_p^p(\bm{S}_n, \mathcal{X})>r)r^{p-1} \diff r & = \Bigg( \int_0^{2r_0N^{1/p}} + \int_{2r_0N^{1/p}}^\infty \Bigg) \mathbb{P}(\mathrm{H}_p(\bm{S}_n, \mathcal{X})>r)r^{p-1} \diff r \nonumber\\
& \leq \frac{2^p N r_0^p}{p} + \int_{2r_0N^{1/p}}^\infty \frac{4^b N^{b/p} r^{p-b-1}}{a} \exp\bigg( -\frac{ar^bn}{2^bN^{b/p}} \bigg) \diff r. \label{eq:lem_tail_haus_bound_1}
\end{align}
Applying a change of variables by setting $\displaystyle v := \frac{ar^bn}{2^bN^{b/p}}$, the integral in (\ref{eq:lem_tail_haus_bound_1}) simplifies to
\begin{equation}
\label{eq:chvar_lem}
\int_{2r_0N^{1/p}}^\infty \frac{4^bN^{b/p}r^{p-b-1}}{a} \exp\bigg( -\frac{ar^bn}{2^bN^{b/p}} \bigg) \diff r = \frac{2^{p+b}N}{ba^\theta n^\theta} \int_{anr_0^b}^\infty v^{\theta-1}e^{-v}\diff v.
\end{equation}

When $p > b$, the right-hand side of (\ref{eq:chvar_lem}) is bounded by $\displaystyle \frac{2^{p+b}N}{ba^\theta n^\theta}\Gamma(\theta)$, as desired, proving (\ref{eq:haus_bound_1}).

When $p \leq b$, we consider two cases of $r_n$.  If $r_n \leq 2r_0 N^{1/p}$, then (\ref{eq:tailprob_haus_bound_1}) is less than or equal to
$$
\frac{2^p N r_0^p}{p} + \int_{r_n}^\infty \frac{4^bN^{b/p}r^{p-b-1}}{a} \exp\bigg(-\frac{ar^bn}{2^bN^{b/p}} \bigg) \diff r.
$$
If $r_n > 2r_0N^{1/p}$, then (\ref{eq:tailprob_haus_bound_1}) is less than or equal to
$$
r_n + \int_{r_n}^\infty \frac{4^bN^{b/p}r^{p-b-1}}{a} \exp\bigg(-\frac{ar^bn}{2^bN^{b/p}} \bigg) \diff r.
$$
In both cases, we have that (\ref{eq:tailprob_haus_bound_1}) is less than or equal to
\begin{equation}
\label{eq:ind_haus_bound_1}
\frac{2^p N r_0^p}{p} + r_n \mathbf{1}\{r_n > 2r_0N^{1/p}\} + \int_{r_n}^\infty \frac{4^b N^{b/p} r^{p-b-1}}{a}\exp\bigg( -\frac{ar^bn}{2^bN^{b/p}} \bigg) \diff r.
\end{equation}
Via the same change of variables $v$ above, the integral in (\ref{eq:ind_haus_bound_1}) simplifies to
$\displaystyle \frac{2^{p+b}N}{ba^\theta n^\theta}\int_{\log n}^\infty v^{\theta-1}e^{-v}\diff v.$
Notice that when $p \leq b$, $v^{\theta-1}$ is montone decreasing, so
$$
\int_{\log n}^\infty v^{\theta-1}e^{-v} \diff v \leq (\log n)^{\theta-1}\int_{\log n}^\infty e^{-v} \diff v = \frac{(\log n)^{\theta-1}}{n},
$$
which proves (\ref{eq:haus_bound_2}), as desired.
\end{proof}

\subsection{Case II: $\mathcal{X}$ Compact}

we first show that for two subsets $\mathcal{X},\mathcal{Y}\subseteq\mathbb{R}^m$, we can bound the difference of distance functions by the Hausdorff distance.

\begin{lemma}\label{lemma:function-hausdorff}
    Let $\mathcal{X},\mathcal{Y}\subseteq\mathbb{R}^m$ be two bounded subsets. Then $\|\mathrm{dist}_{\mathcal{X}}-\mathrm{dist}_{\mathcal{Y}}\|_\infty\le \mathrm{H}_\infty(\mathcal{X},\mathcal{Y})$.
\end{lemma}
\begin{proof}[Proof of Lemma \ref{lemma:function-hausdorff}]
    For any $z\in\mathbb{R}^m$, and $y\in\mathcal{Y}$ we have 
    $$
    \mathrm{dist}_\mathcal{X}(z) \le \|z-y\|+\mathrm{dist}_\mathcal{X}(y) \le \|z-y\|+\mathrm{H}_\infty(\mathcal{X},\mathcal{Y}).
    $$
    Taking the infimum with respect to $y\in\mathcal{Y}$, we have $\mathrm{dist}_\mathcal{X}(z)-\mathrm{dist}_\mathcal{Y}(z)\le \mathrm{H}_\infty(\mathcal{X},\mathcal{Y})$ for all $z\in\mathbb{R}^m$. Similarly, we have $\mathrm{dist}_\mathcal{Y}(z)-\mathrm{dist}_\mathcal{X}(z)\le \mathrm{H}_\infty(\mathcal{X},\mathcal{Y})$ for all $z\in\mathbb{R}^m$. Taking the supremum over $z\in\mathbb{R}^m$ we have $\|\mathrm{dist}_{\mathcal{X}}-\mathrm{dist}_{\mathcal{Y}}\|_\infty\le \mathrm{H}_\infty(\mathcal{X},\mathcal{Y})$.
\end{proof}

We then prove the bias estimation.
\begin{proof}[Proof of Lemma \ref{lemma:continuous-bias}]
   We restrict the distance function to the ball $\mathcal{B}(x_0,2R)$ for some fixed $x_0\in \mathcal{X}$. Then we can apply Theorem \ref{thm:Wasser-functions} to level set filtrations of the distance function:
    \begin{equation*}
        \begin{aligned}
            \mathrm{OT}_p^p(D_\mu,\, D[\mathrm{dist}_\mathcal{X}])&\le \mathbb{E}[\mathrm{OT}_p^p(D[\mathrm{dist}_{\bm{S}_n}], D[\mathrm{dist}_\mathcal{X}])]\\
            &\le C_\mathcal{X}\max\big\{\|\mathrm{dist}_{\bm{S}_n}\|_{\mathrm{Lip}}^p,\, \|\mathrm{dist}_{\mathcal{X}}\|_{\mathrm{Lip}}^p\big\}\|\mathrm{dist}_{\bm{S}_n}-\mathrm{dist}_{\mathcal{X}}\|_\infty^{p-k} .
        \end{aligned}
    \end{equation*}

    Note that distance functions have Lipschitz norm 1, and by Lemma \ref{lemma:function-hausdorff} we can bound the bias by the Hausdorff distance:
    $$
\mathrm{OT}_p^p(D_\mu,\, D[\mathrm{dist}_\mathcal{X}])\le C_\mathcal{X}\mathrm{H}_\infty^{p-k}(\bm{S}_n,\mathcal{X}).
    $$
    Applying Lemma \ref{lemma:vec-bias-rate} to the $(p-k)$-th moment of the Hausdorff distance, we obtain \eqref{eq:continuous-bias}.  
\end{proof}

For variance estimate, we have the following result.

\begin{proof}[Proof of Lemma \ref{lemma:var-rate-continuous}]
    Let $R=\mathrm{diam}(\mathcal{X})$ be the diameter of $\mathcal{X}$. For any $r>R$, the \v{C}ech complex $\mathrm{\check{C}ech}(\mathcal{X},r)\simeq \mathrm{dist}_{\mathcal{X}}((\infty,r])$ has trivial topology. Fix a point $x_0\in \mathcal{X}$. It suffices to consider the restriction of the distance function $\mathrm{dist}_\mathcal{X}$ to the ball $\mathcal{B}(x_0,2R)$. As an $m$-dimensional compact manifold, $\mathcal{B}(x_0,2R)$ has bounded $k$-total persistence for any $k>m$ \citep[Section 2.3]{cohen-2010-lipschitz}. The estimation \eqref{eq:var-rate-continuous} is a consequence of \cite[Theorem 1]{divol-2021-estimation}. 
\end{proof}


\section{Technical Details for Fr\'echet Means of Persistence Diagrams}
\label{app:pd_bias}

Due to the impasse associated with deriving a rate estimate for the variance of Fréchet means, a full expression for the approximation error is not achievable. We give the following bound for bias. By the triangle inequality,
$$
\mathrm{W}^p_p(\mathbf{Fr}, D[\mathcal{X}]) \leq 2^{p-1}(\mathrm{W}^p_p(\mathbf{Fr},D) + \mathrm{W}^p_p(D, D[\mathcal{X}])).
$$
Integrating both sides with respect to $(\pi^{\otimes n})_*$, we obtain
\begin{equation}
\label{eq:bias_fr_int}
\mathrm{W}^p_p(\mathbf{Fr}, D[\mathcal{X}]) \leq 2^{p-1}\Bigg( \int_{\mathcal{D}_p} \mathrm{W}^p_p(\mathbf{Fr}, D) \diff (\pi^{\otimes n}_*(D)) + \int_{\mathcal{D}_p} \mathrm{W}^p_p(D, D[\mathcal{X}])\diff (\pi^{\otimes n}_*(D)) \Bigg).
\end{equation}
Denote $\displaystyle \sigma^2 = \min_{D_F}\int_{\mathcal{D}_p} \mathrm{W}^p_p(D_F, D)\diff (\pi^{\otimes m})_*(D)$.  Then by definition of the Fréchet population mean $\mathbf{Fr}$,
\begin{align}
\mathrm{W}^p_p(\mathbf{Fr}, D[\mathcal{X}]) & \leq 2^{p-1}\Bigg( \sigma^2 + \int_{\mathcal{D}_p} \mathrm{W}^p_p (D, D[\mathcal{X}])\diff (\pi^{\otimes n})_*(D) \Bigg) \nonumber\\
& = 2^{p-1}\sigma^2 + 2^{p-1}\mathbb{E}[\mathrm{W}^p_p(\bm{D}_n, D[\mathcal{X}])]. \label{eq:bias_fr_wass_stab}
\end{align}

Using Theorem \ref{thm:wasserstein-point-clouds}, $\mathrm{W}^p_p(\mathbf{Fr}, D[\mathcal{X}]) \leq C_N\mathrm{H}^p_p(\bm{S}_n, \mathcal{X})$, so \eqref{eq:bias_fr_wass_stab} becomes
$$
\mathrm{W}^p_p(\mathbf{Fr}, D[\mathcal{X}]) \leq 2^{p-1}\sigma^2 + 2^{p-1}C_N \mathbb{E}[\mathrm{H}^p_p(\bm{S}_n, \mathcal{X})].
$$
As above in the case of mean persistence measures, it now remains to bound the $p$-Hausdorff distance $\mathbb{E}[\mathrm{H}^p_p(\bm{S}_n, \mathcal{X})]$.  Together with Lemma \ref{lemma:hausdorff_bound}, we obtain the following bound for the bias of Fréchet means:
\begin{equation*}
    \mathrm{W}^p_p(\mathbf{Fr},D[\mathcal{X}])\leq
    \begin{cases}
    O(\sigma^2)+O(r_0^p)+  O\big(n^{-\theta}\big)& \mbox{~if~} p>b,\\
    \displaystyle O(\sigma^2)+O(r_0^p) + O\bigg(\bigg(\frac{\log n}{n}\bigg)^{\theta+1}\bigg)& \mbox{~if~} \displaystyle p\le b, 
    \end{cases}
\end{equation*}
where $\theta = p/b-1$.


\section{Additional Simulation Results}

\subsection{Comparing Mean Persistence Measures and Fr\'{e}chet Means}
\label{sec:mpm_fm}

\begin{figure}[htbp]
    \centering
    \includegraphics[width=\linewidth]{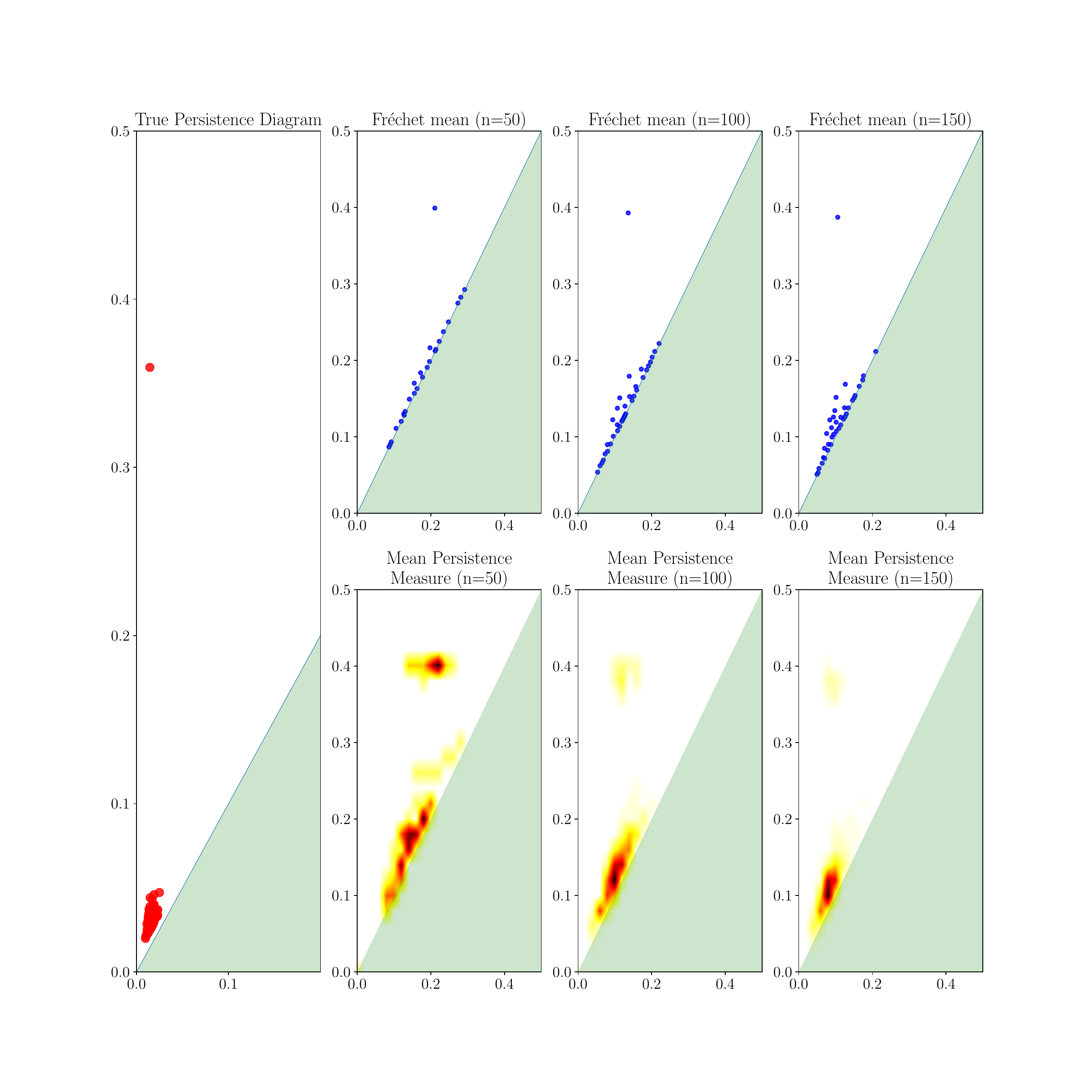}
    \caption{Illustration of Fréchet means of persistence diagrams and mean persistence measures for varying sample sizes. The left panel is the true persistence diagram of the sample set $\mathcal{X}$ computed from the Annulus. The top row shows the Fréchet mean of persistence diagrams of subsampled sets of different sizes. The bottom row shows the mean persistence measures of persistence diagrams of subsampled sets of different sizes.}
    \label{fig:compare-frechet-multi}
\end{figure}

We now provide experimental comparisons of the Fr\'echet mean and persistence measures as measures of central tendency in our subsampling approach on three types of data: Euclidean point clouds, abstract finite metric spaces, and networks. Although our convergence analysis assumes point cloud data in Euclidean space, these experiments demonstrate the applicability of our subsampling method to other types of data.

\paragraph{Annulus.} Here we take $\mathcal{X}$ to be a sample set of $N = 5,000$ points from an annulus with outer radius 0.5 and inner radius 0.2. The Fr\'echet mean and mean persistence measures are computed based on $B=20$ sets of subsamples from $\mathcal{X}$ each consisting of $n$ points, where $n$ ranges from 50 to 400.  Figure \ref{fig:compare-frechet-multi} compares Fr\'echet mean persistence diagreams and mean persistence measures for subsets of different sizes $n$ to the true persistence diagram.

\paragraph{HIV Data.} We consider the HIV dataset collected by  \cite{otter2017roadmap} as the true data set $\mathcal{X}$. The HIV data set consists of 1,088 genomic sequences and the difference between any two sequences is measured by the Hamming distance. Thus, the HIV data set can be viewed as a finite metric space.  We compute its persistent homology based on the VR filtration. We subsample $B = 25$ subsets from $\mathcal{X}$ each consisting of $n$ points, where $n$ ranges from 100 to 300. From these subsamples we compute both the Fr\'echet mean and mean persistence measure.

\begin{figure}[h]
\centering
\begin{subfigure}[t]{0.3\textwidth}
    \centering
    \includegraphics[width=\textwidth]{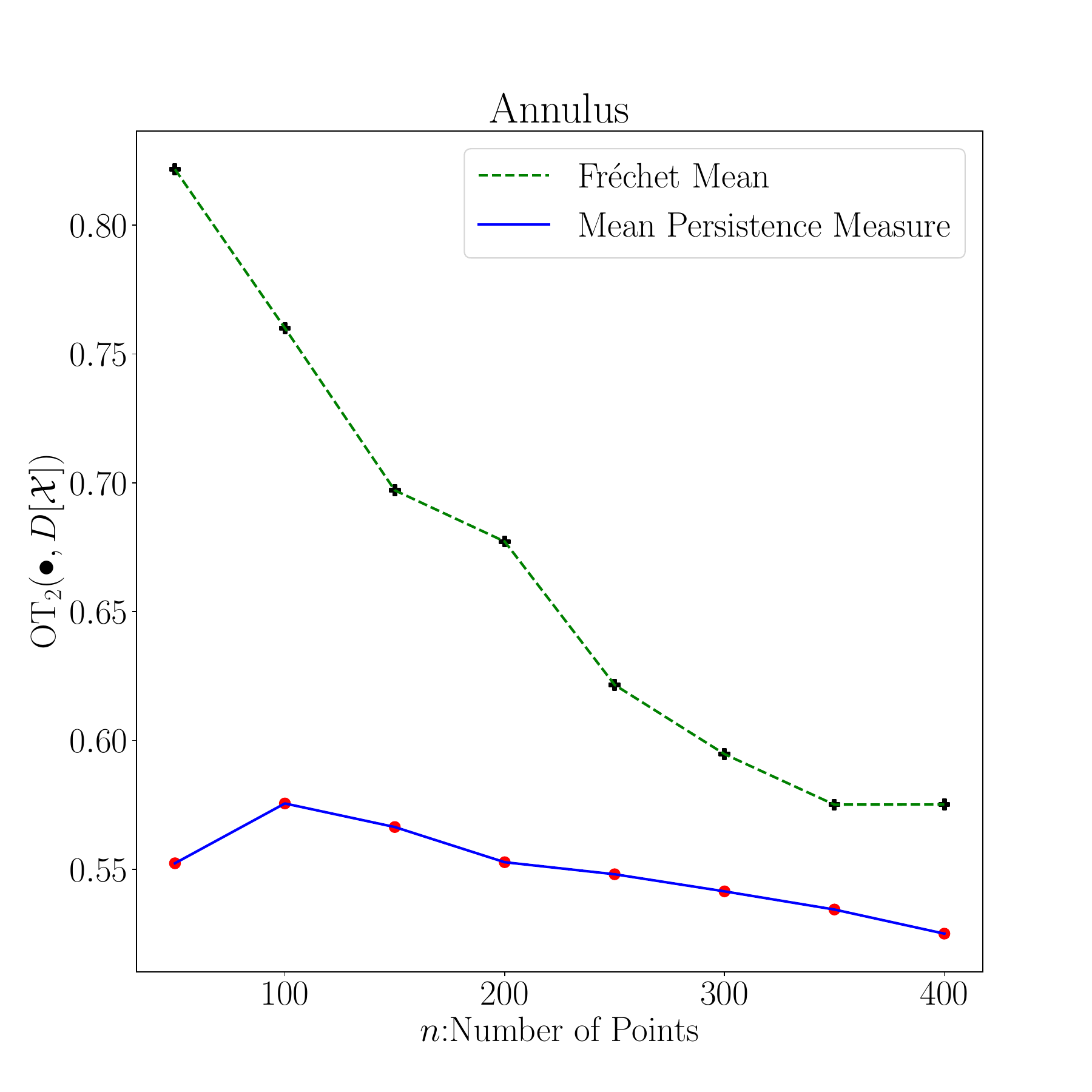}
\end{subfigure}
\begin{subfigure}[t]{0.3\textwidth}
    \centering
    \includegraphics[width=\textwidth]{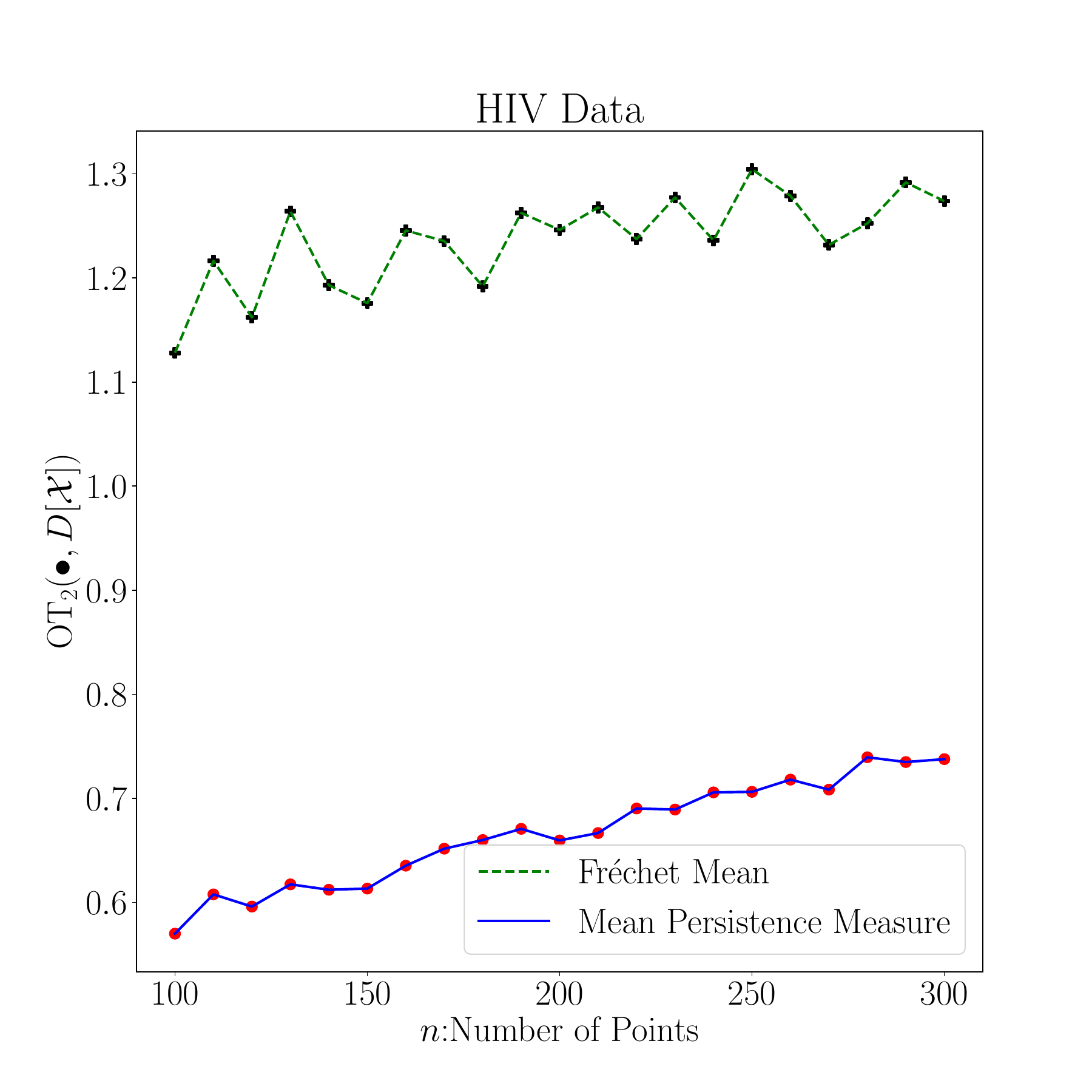}
\end{subfigure}
\begin{subfigure}[t]{0.3\textwidth}
    \centering
    \includegraphics[width=\textwidth]{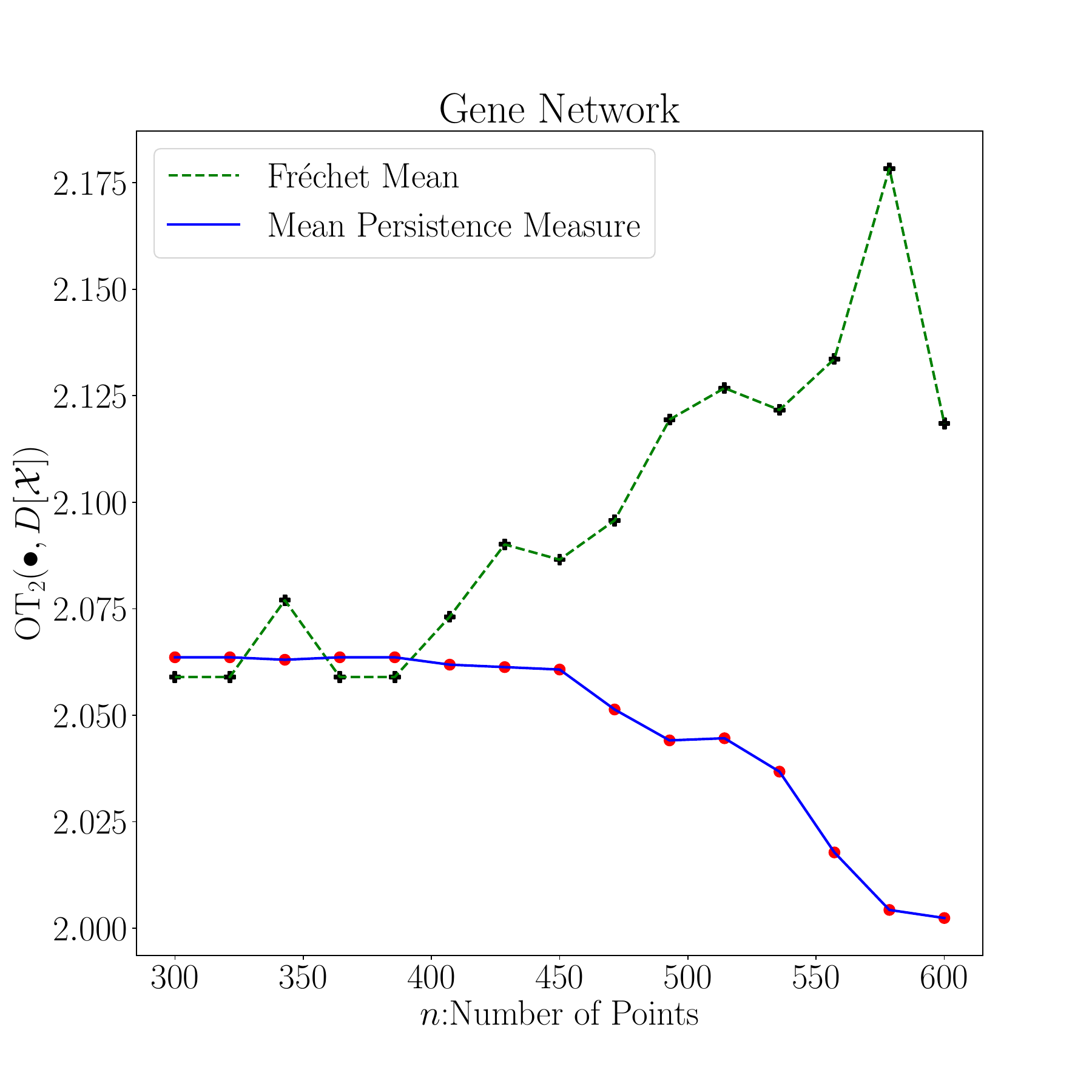}
\end{subfigure}
    \caption{Comparison of Fr\'echet means and mean persistence measures of persistence diagrams computed from subsamples drawn from the Annulus, HIV and Gene Network data.}
    \label{fig:compare-frechet-loss}
\end{figure}

\paragraph{Gene Network.} We use a gene network from \cite{nr,cho2014wormnet} as our true data set $\mathcal{X}$. The gene network is an undirected weighted graph consisting of 924 nodes and 3,233 edges. The nodes represent genes and weights represent intensities of genetic interactions. The persistent homology is computed using the weight matrix. We subsample $B = 30$ subgraphs from $\mathcal{X}$, each consisting of $n$ nodes, where $n$ ranges from 200 to 600. The persistent homology of each subgraph can be computed using the submatrix from the total weight matrix.  We compute the 2-Wasserstein distances between the mean diagrams or measures and the true persistence diagram $D[\mathcal{X}]$. The loss curves are shown in Figure \ref{fig:compare-frechet-loss}. 

\paragraph{Robustness Comparison.} We test the robustness of the Fr\'echet means against the mean persistence measures in our subsampling approach by contaminating the Annulus and Gene Network with Gaussian noise. Specifically, let $W$ be the matrix consisting of points in Annulus or weights in Gene Network. We create noisy data by $\widetilde{W}=W+V$ where $V$ is a matrix of the same shape as $W$ whose elements are i.i.d.~samples from the Gaussian distribution $\mathcal{N}(0,\sigma)$. We compute the mean persistence measures and Fr\'echet means using sample sets from the noisy data and then compute the error with respect to the true persistence diagrams of the original data sets. The loss curves are plotted in Figure \ref{fig:compare-robust}.  

These experimental results show that mean persistence measures are more accurate and more robust than Fr\'echet means. The reason lies in two aspects: from the theoretical perspective, the variance of Fr\'echet means decreases in a complicated manner due to the nonnegative curvature of $(\mathcal{D}_2,\mathrm{W}_2)$, while from the algorithmic prespective, the greedy algorithm returns unstable local minima which lead to fluctuating approximation errors. 

The results for HIV data and Gene Network also demonstrate that subsampling and approximating the true persistence diagram by mean persistence measures also performs well on general finite metric spaces.  

\begin{figure}[htbp]
\centering
    \begin{subfigure}[t]{0.45\textwidth}
        \includegraphics[width=\textwidth]{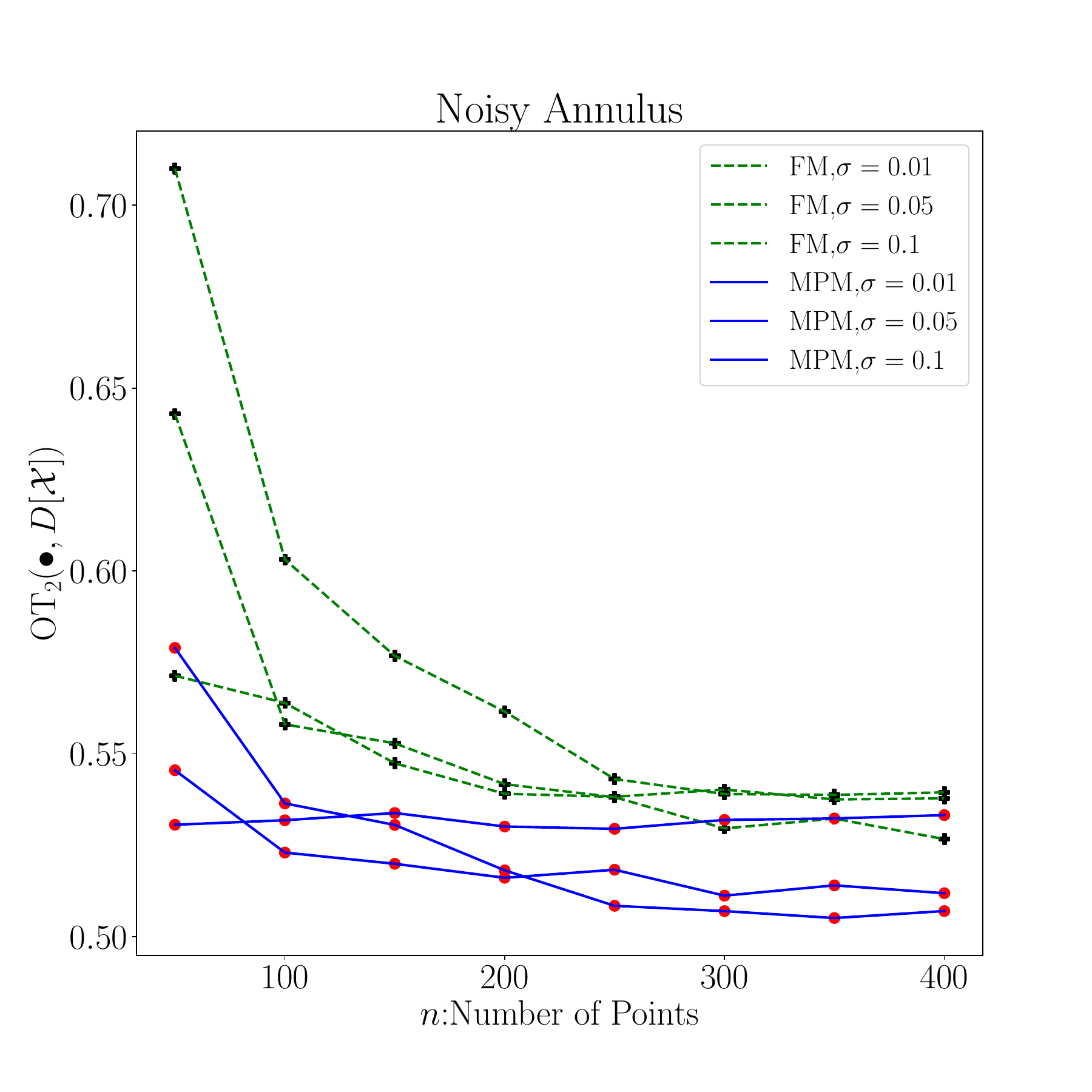}
    \end{subfigure}
    \begin{subfigure}[t]{0.45\textwidth}
        \includegraphics[width=\textwidth]{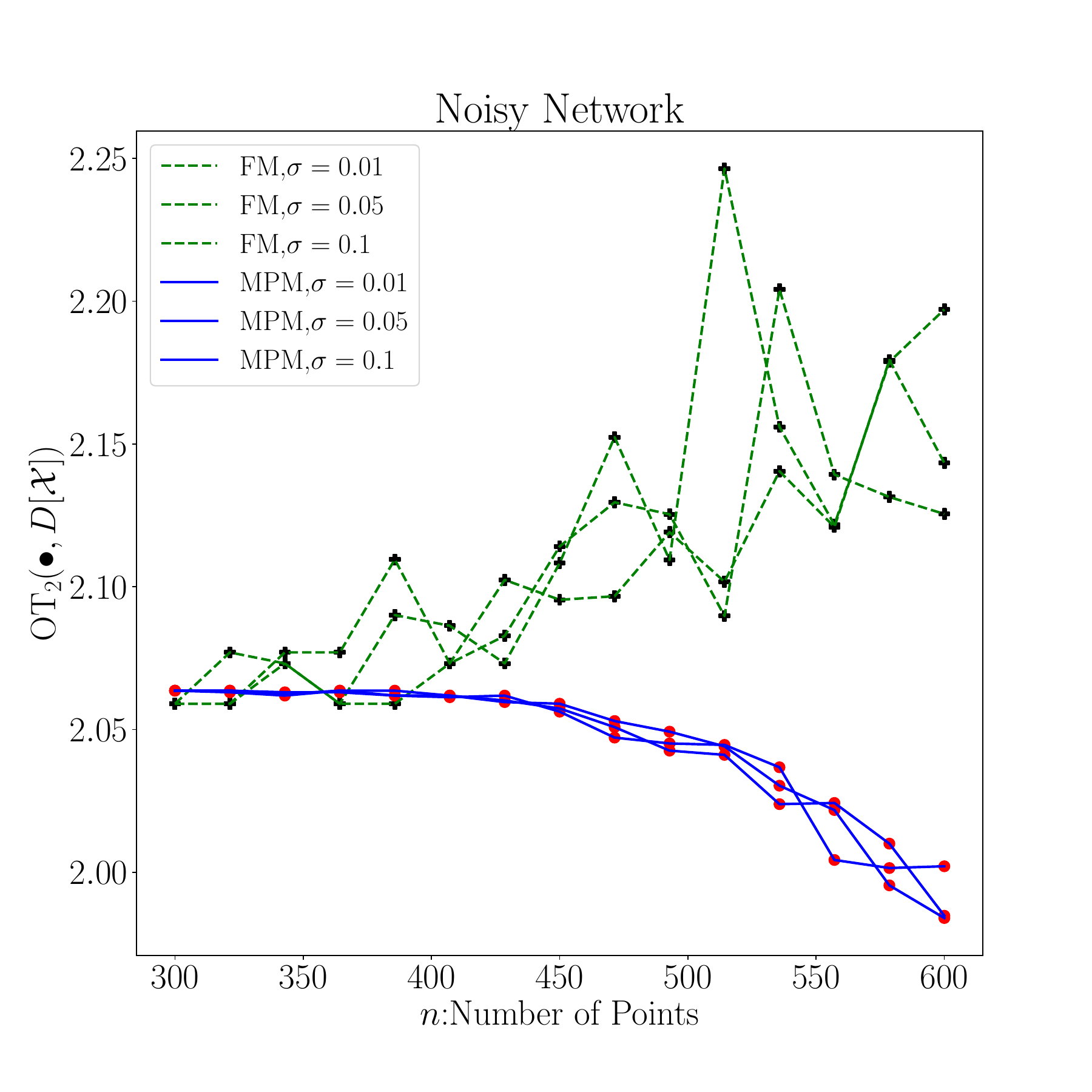}
    \end{subfigure}
    
    \caption{Robustness comparison for computed Fr\'echet means and mean persistence measure of persistence diagrams computed from subsamples drawn from the Annulus data and Gene Network data.}
    \label{fig:compare-robust}
\end{figure}

\section{Additional Application Results}

We now present two additional applications demonstrating the applicability of our results.

\subsection{Computing the Topology of Shapes}
\label{subsec:knot_lock}

We study two datasets, `Knot' and `Lock,' obtained from the publicly available Digital Shape Workbench shape repository.

\paragraph{Knot.} The underlying manifold for Knot is a tubular trefoil knot. The original point cloud consists of $N = 478,704$ points, which makes it infeasible to compute persistent homology directly. We subsample $B = 25$ subsets, each consisting of $n = 9,500$ points from the original set, and compute their 1-dimensional persistence diagrams. We manually filter points with persistence less than 0.1 since the star-like ornaments on the surface will generate small noisy circles during filtration. The mean persistence measure and Fréchet mean of persistence diagrams are presented in Figure \ref{fig:knot}. From the mean persistence measure we can see three clusters of points on the half plane. The top cluster corresponds to the longest circle, i.e., the trefoil knot, while the bottom cluster corresponds to the small circle of the tube. The middle cluster represents homology classes created by self-intersections of the surface during its growth in VR filtration. The top point in Fréchet mean persistence diagram is consistent with the top cluster in mean persistence measure which indicates a longest cycle. However, we see that the Fréchet mean shows artificial homological noise near the boundary, even though each persistence diagram does not show points with small persistence. 

To quantize the mean persistence measures, we use the resulting computed mean persistence measure and Fréchet mean diagrams to initialize 1 centroid at $(0.05,\, 0.35)$, 4 centroids around $(0.07,\, 0.2)$ and 2 centroids around $(0.04,\, 0.1)$, and then apply the quantization algorithm in \cite{divol-2021-estimation,10.1214/21-EJS1834}. Figure \ref{fig:quan-knot} shows the quantized persistence diagram of the mean persistence measure.

\begin{figure}[htbp]
    \begin{subfigure}[t]{0.23\textwidth}
    \centering
        \includegraphics[width=\textwidth]{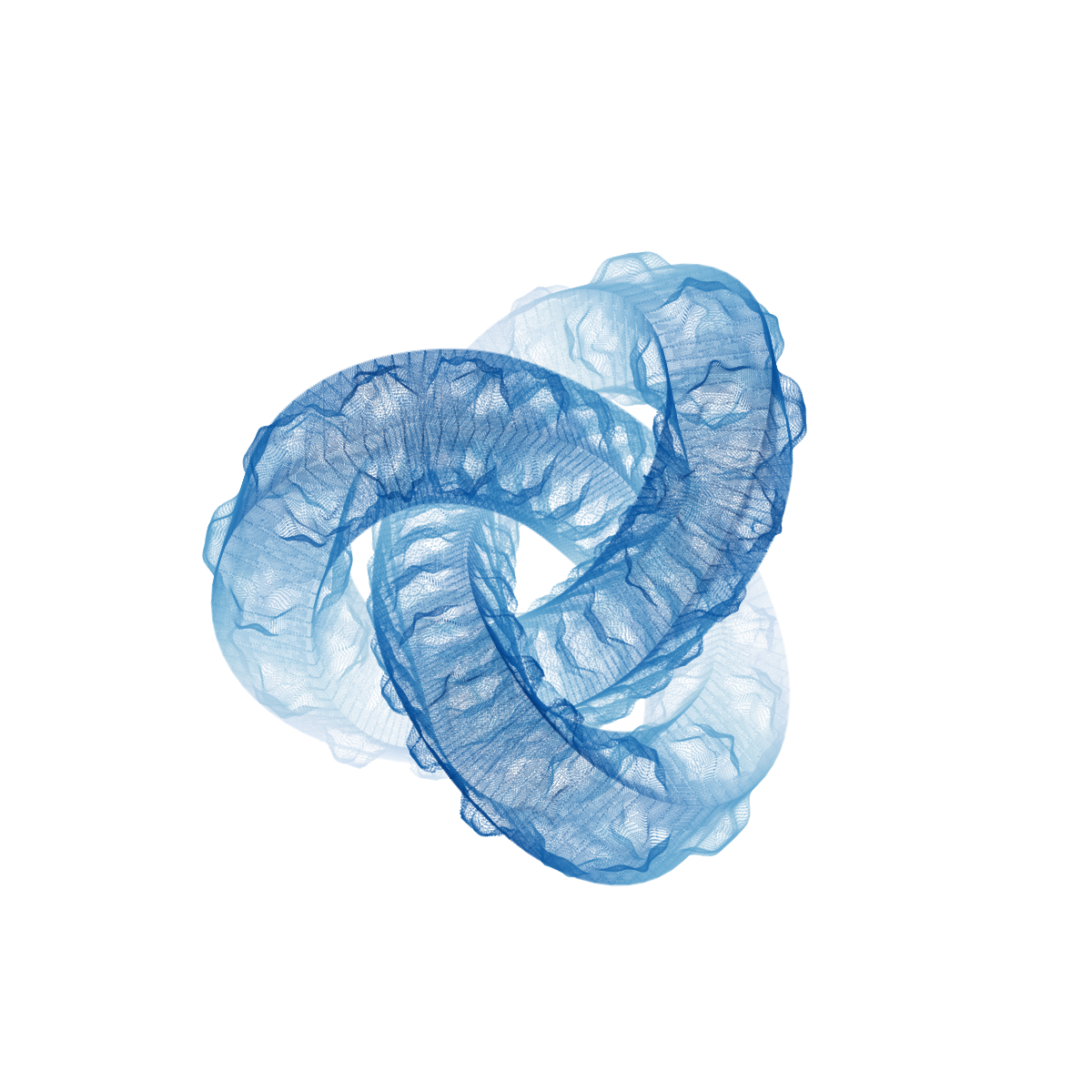}
        \caption{Original data}
    \end{subfigure}
        \begin{subfigure}[t]{0.23\textwidth}
    \centering
        \includegraphics[width=\textwidth]{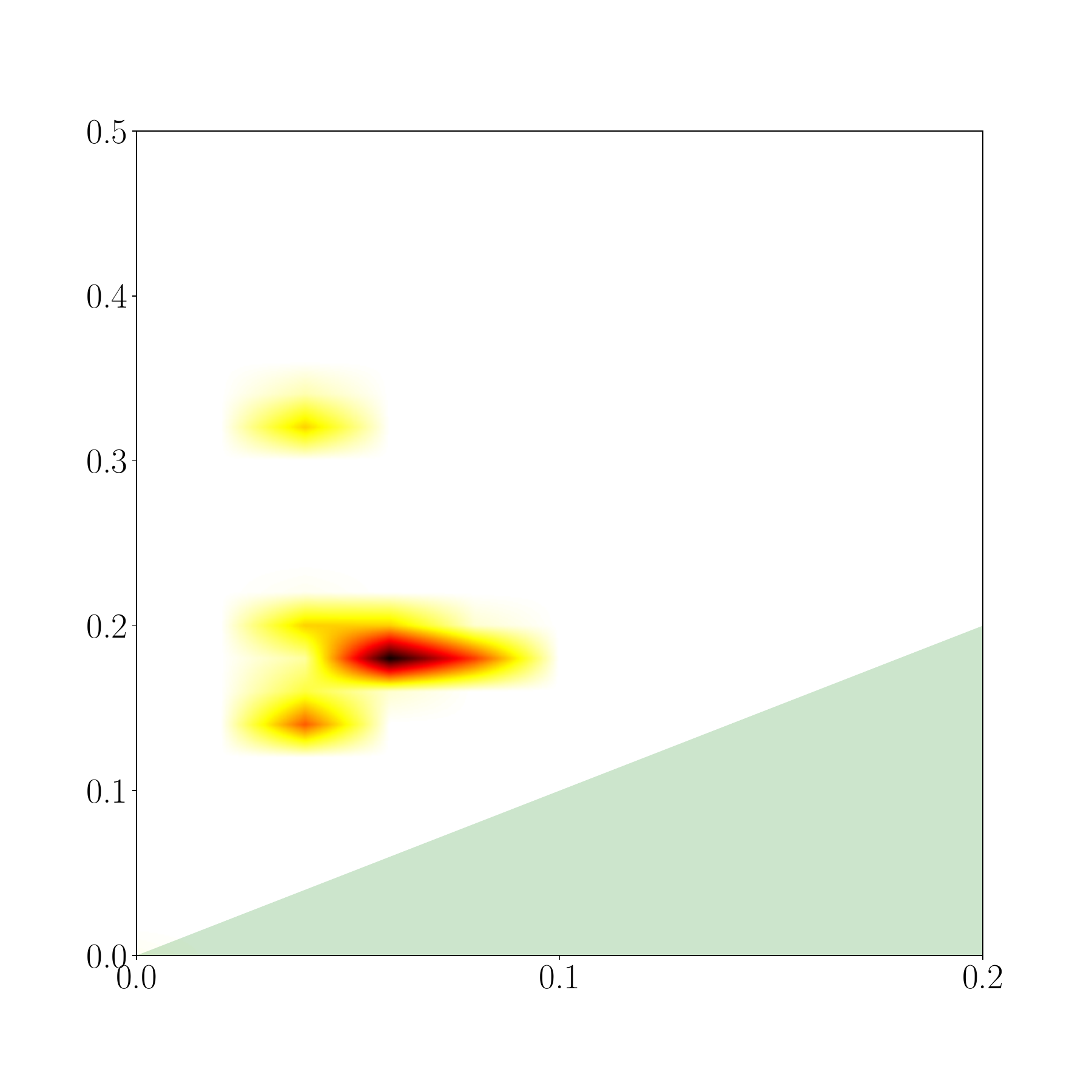}
        \caption{Mean persistence measure}
    \end{subfigure}
        \begin{subfigure}[t]{0.23\textwidth}
    \centering
        \includegraphics[width=\textwidth]{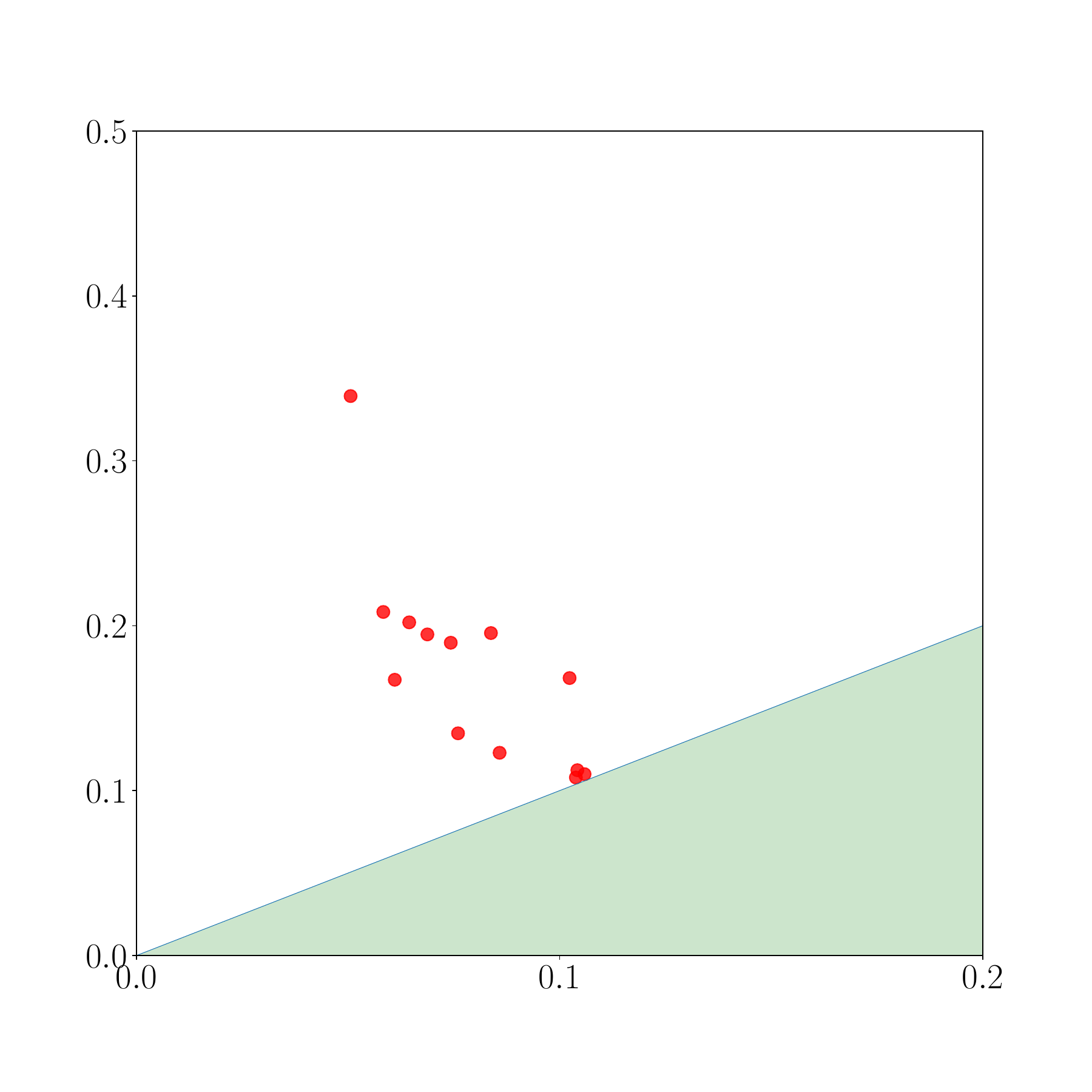}
        \caption{Fréchet mean}
    \end{subfigure}
    \begin{subfigure}[t]{0.23\textwidth}
    \centering
        \includegraphics[width=\textwidth]{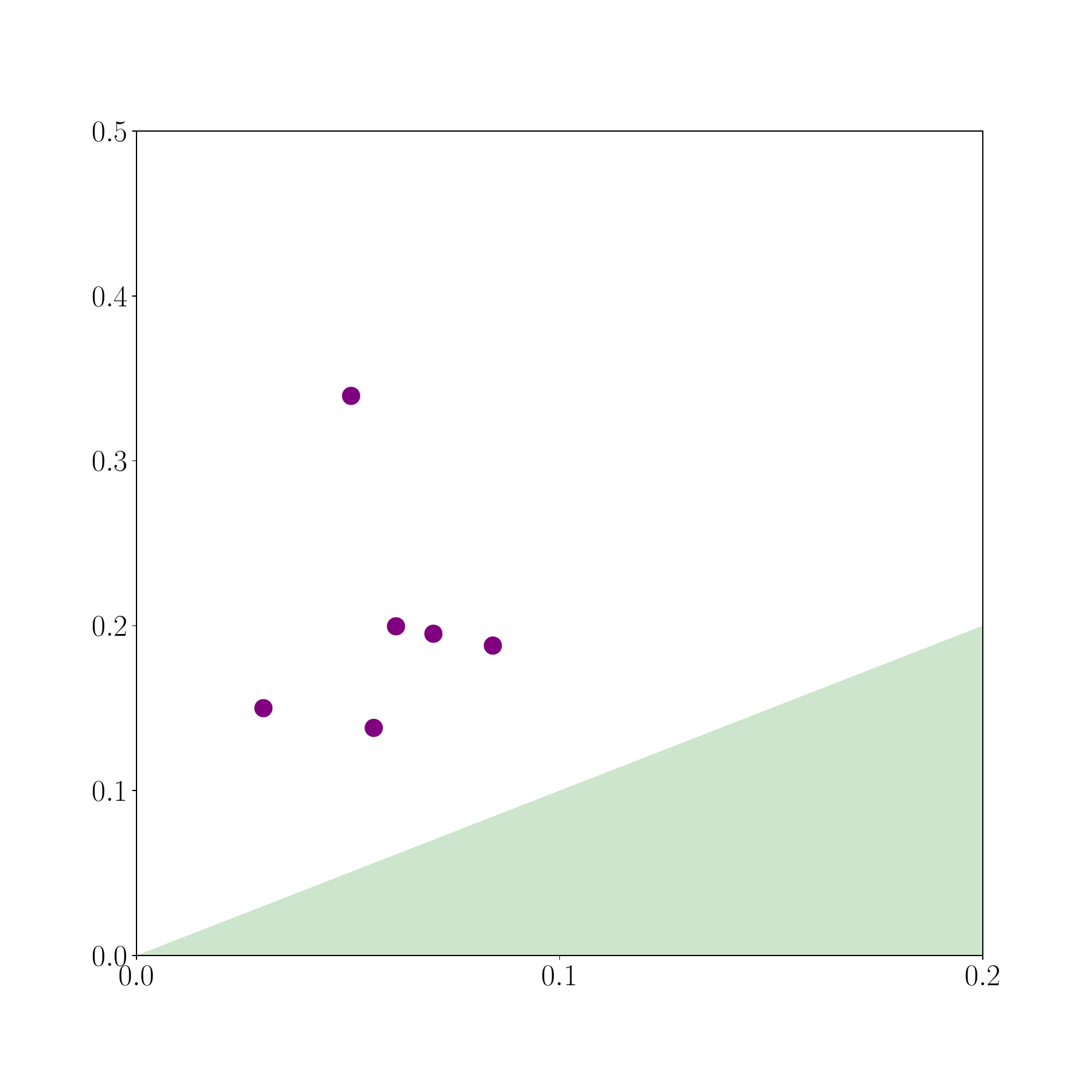}
        \caption{Quantization of (b)}
        \label{fig:quan-knot}
    \end{subfigure}
    \caption{Approximation of the persistent homology of Knot. (a) Original point cloud with $N = 478,704$ points; (b) Mean persistence measure for dimension 1 homology for one subsample of $n = 9,500$ points; (c) Fréchet mean for dimension 1 homology for one subsample of $n = 9,500$ points; (d) Quantization of the mean persistence measure shown in (b).  In (b), the bottom cluster shows the small circle of the tube, while the middle cluster in the mean persistence measure indicates homology classes generated by intersections of the surfaces during VR filtration.  In (c), the topmost point in the Fréchet mean represents the homology of the trefoil knot, which is also presented by the top yellow cluster in (b).}
    \label{fig:knot}
\end{figure}

\paragraph{Lock.} The point cloud is roughly composed of three parts, with one tube in the center and two caps on two sides. Each cap can be viewed as a closed surface of genus 4. Thus the underlying manifold for Lock is homeomorphic to the connected sum of a torus and two surfaces of genus 4, which is essentially a closed surface of genus 9. The original point cloud consists of $N = 460,592$ points, which again is too large to compute persistent homology on directly. We subsample $B = 30$ subsets, each consisting of $n = 9,000$ points from the original set and compute their 1-dimensional persistence diagrams. The mean persistence measure and Fréchet mean of persistence diagrams are presented in Figure \ref{fig:lock}. From the Fréchet mean persistence diagram we can easily see the point with the largest persistence. This point corresponds to the largest circle in the central tube. Right below the top point there are 8 points corresponding to 8 circles which bound holes on the caps. In theory, there should be 18 circles generating the 1-dimensional homology of the underlying manifold. However, the remaining 9 circles are of small radii and are mixed up with noisy circles generated by VR filtration. The middle cluster and bottom cluster in the mean persistence measure are consistent with results shown in Fréchet mean. The only difference is that the top cluster is not obvious (in very shallow yellow). This is a matter of visualization as there is only one point at the top in each persistence diagram and they cannot concentrate as a cluster compared to other group of points. 

As above, based on the observations from mean persistence measure and Fréchet mean, we initialize 1 centroid at $(0.05,\, 0.3)$, 8 centroids around $(0.05,\, 0.2)$ and 9 centroids around $(0.05,\, 0.1)$, and then apply the quantization algorithm. Figure \ref{fig:quan-loc} shows the quantized persistence diagram from the mean persistence measure, which can be regarded as a faithful representative of the persistent homology of the original data.  
\begin{figure}[htbp]
\begin{subfigure}[t]{0.23\textwidth}
    \centering
    \includegraphics[width=\textwidth]{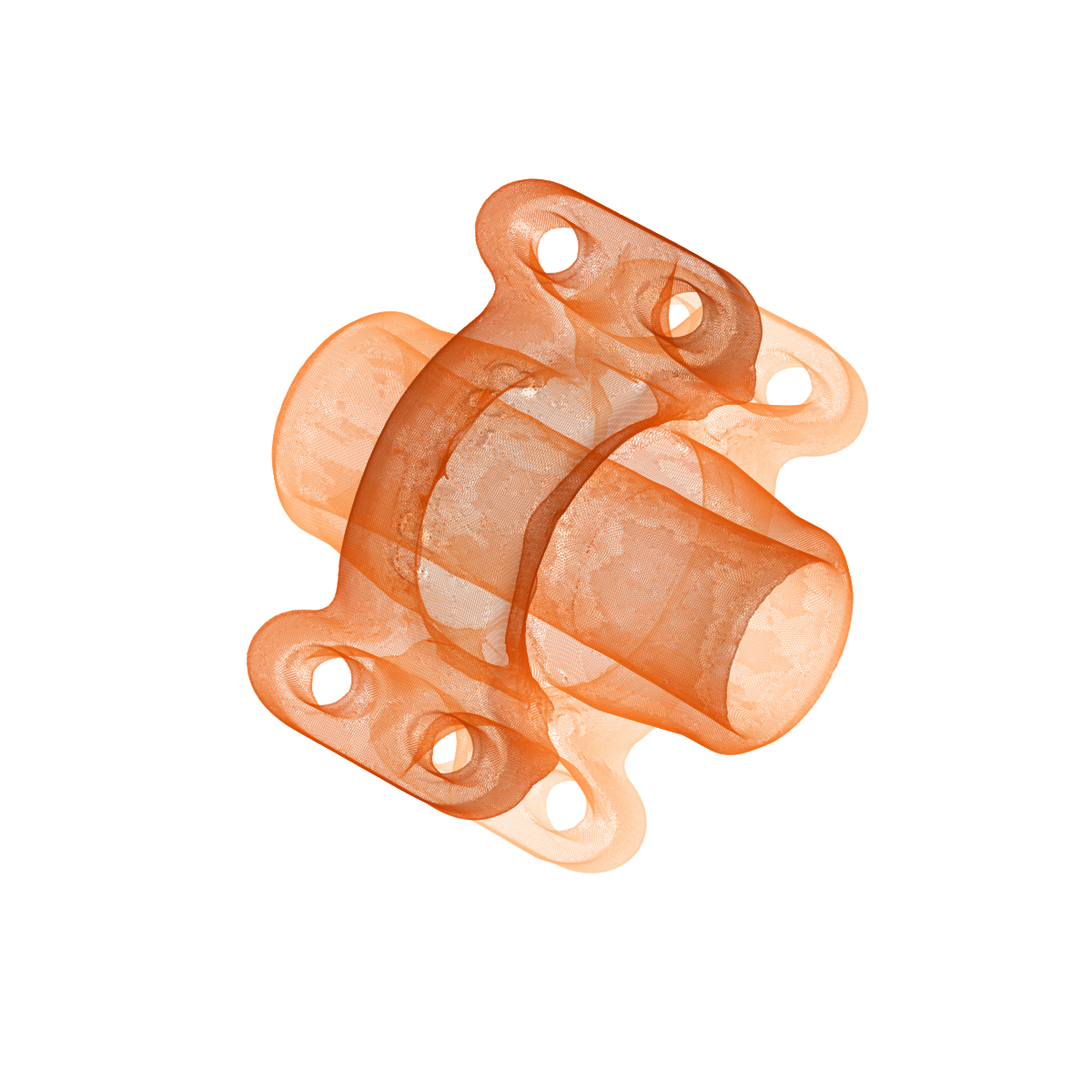}
    \caption{Original data}
\end{subfigure}
\begin{subfigure}[t]{0.23\textwidth}
    \centering
    \includegraphics[width=\textwidth]{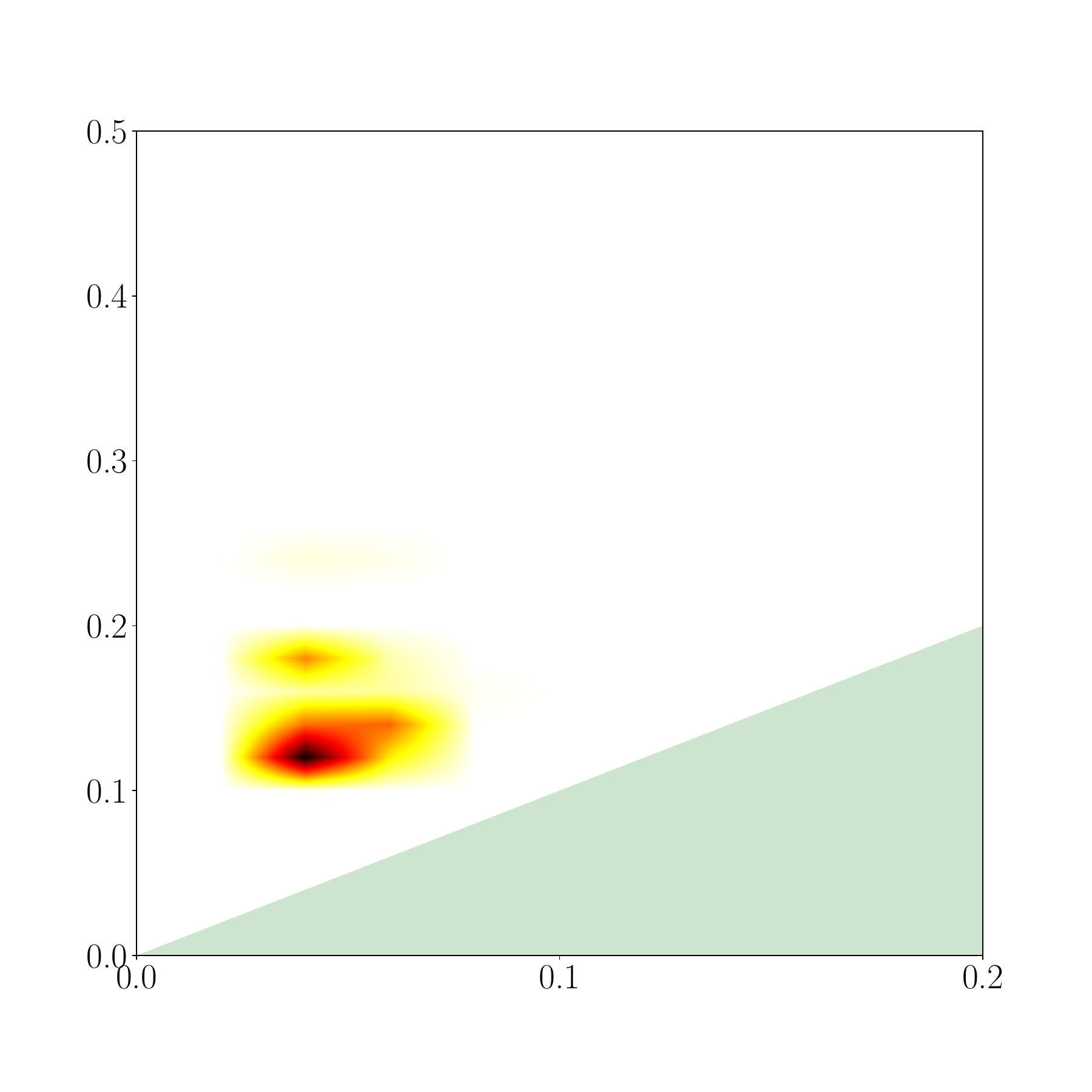}
    \caption{Mean persistence measure}
\end{subfigure}
\begin{subfigure}[t]{0.23\textwidth}
    \centering
    \includegraphics[width=\textwidth]{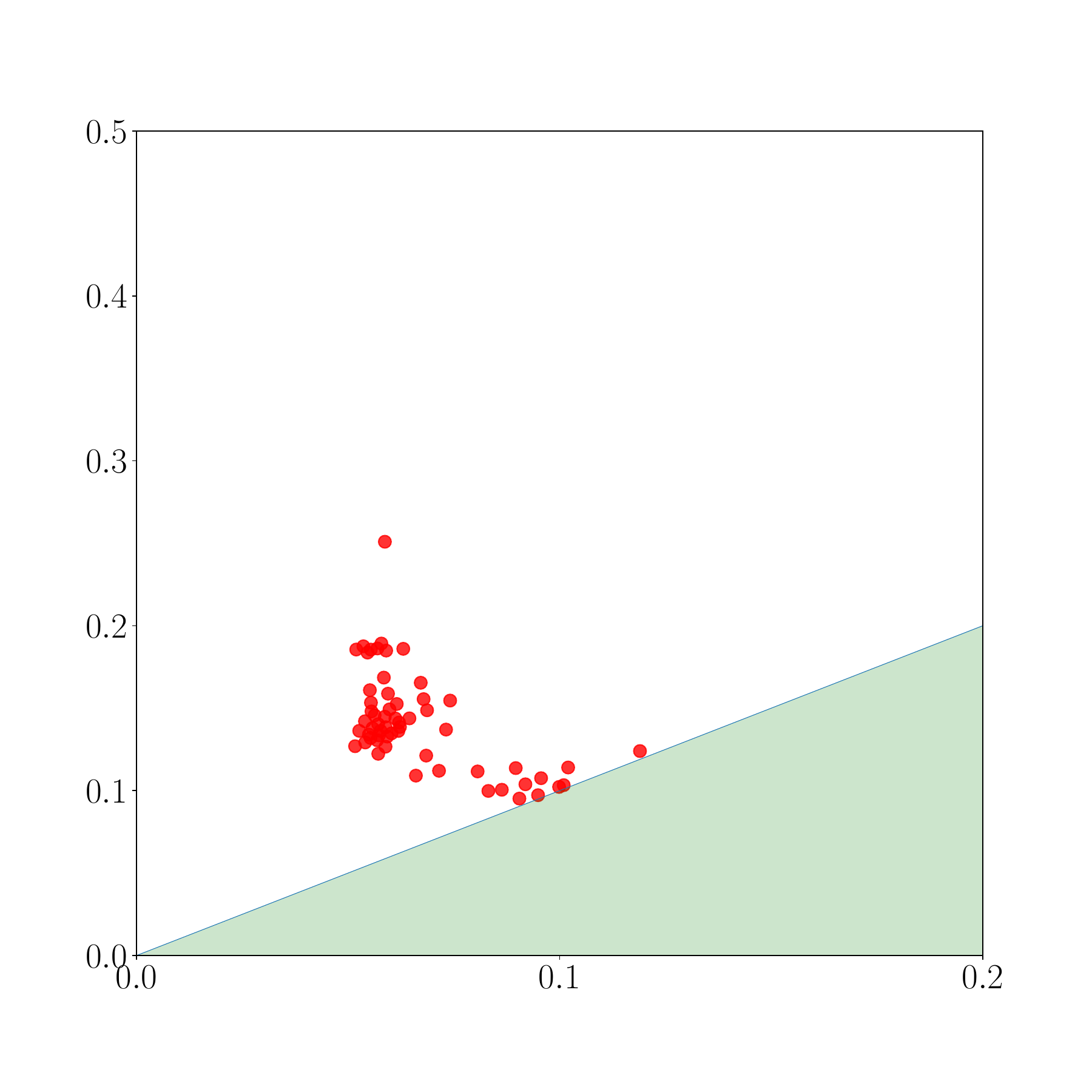}
    \caption{Fréchet mean}
\end{subfigure}
\begin{subfigure}[t]{0.23\textwidth}
    \centering
    \includegraphics[width=\textwidth]{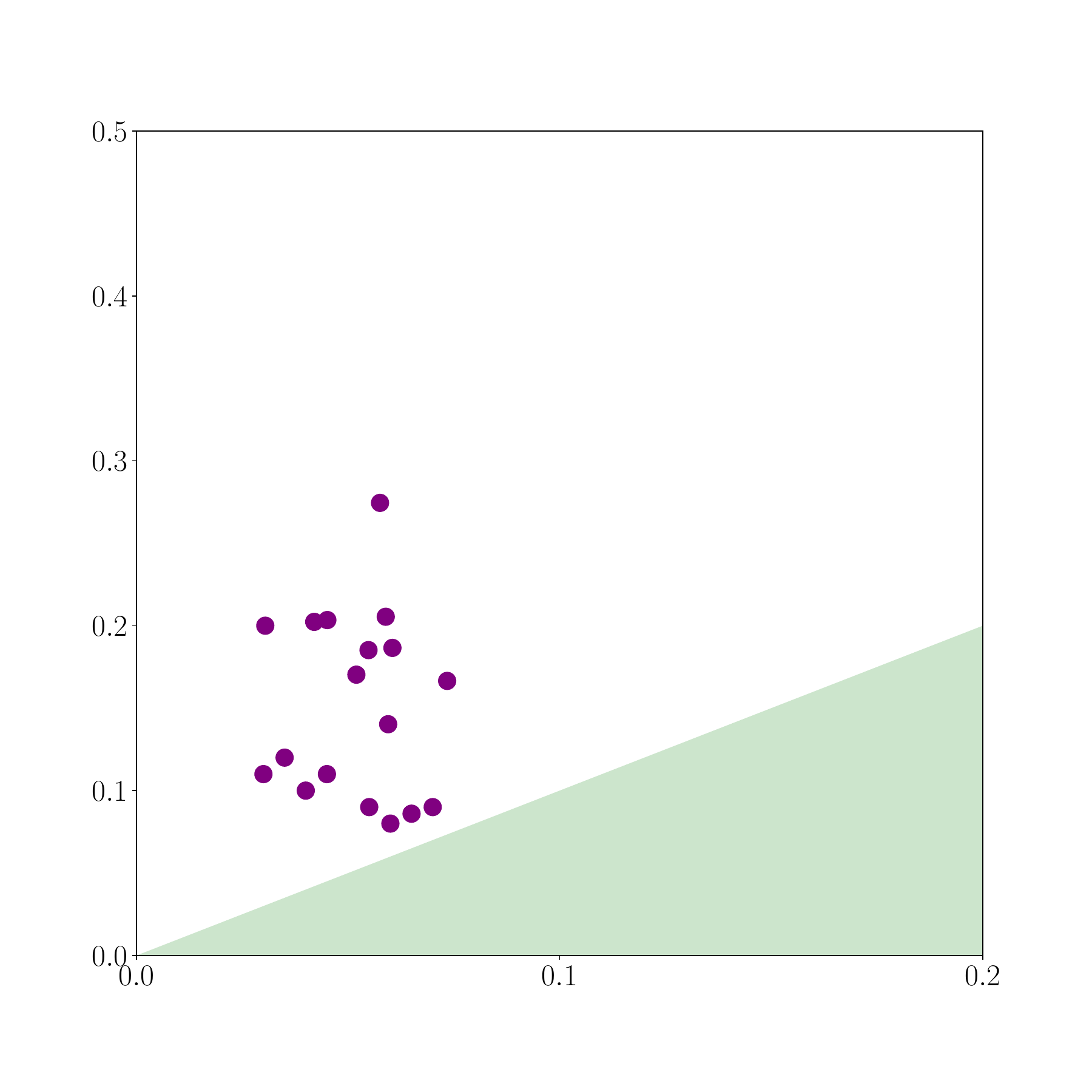}
    \caption{Quantization of (b)}
    \label{fig:quan-loc}
\end{subfigure}
    \caption{Approximation of the persistent homology of Lock. (a) Original point cloud with $N = 460,592$ points; (b) Mean persistence measure for dimension 1 homology computed from one subsample of $n=9,000$ points; (c) Fréchet mean for dimension 1 homology computed from one subsample of $n=9,000$ points; (d) Quantization of the mean persistence measure shown in (b).
    In (b), there are two clear clusters and a cluster at the top in very shallow yellow.  Combined with (c), the point with largest persistence (with respect to the top cluster in shallow yellow) corresponds to the largest circle in the central tube, while the points with $y$-coordinate close to 0.2 (with respect the middle cluster) correspond to 8 circles bounding the holes on two caps, the remaining points (with respect the bottom cluster) correspond to small circles together with some topological noise generated by the VR filtration.
    }
    \label{fig:lock}
\end{figure}

From these two real data examples we can see both merits and limitations of both the mean persistence measure and Fréchet mean as representatives of central tendency for persistence diagrams computed from subsampled data: the clusters (concentrations of measures) in mean persistence measures indicate partial locations of the persistent homology of the original data. However, we cannot read the actual number of points as mean persistence measures are not diagrams. In addition, points with large persistence may not group as clusters, which makes them difficult to visualize. In contrast, Fréchet means are diagrams so we can read the multiplicity of points, and points with large persistence are easy to visualize. However, Fréchet means can generate artificial noise, which can mix up with other points and affect the identification of ``true'' persistent homology.  To bypass the difficulty of interpretation for mean persistence measures, quantization may be applied to obtain a representative persistence diagram from mean persistence measures, with the results from the computed Fréchet mean to drive the initialization procedure.

\subsection{A Shape Clustering Task}

We now use the mean persistence measure and Fr\'echet mean of persistence diagrams in a machine learning task of shape clustering on real-world data. The Mechanical Components Benchmark (MCB) is a large-scale dataset of 3D objects of mechanical components collected from online 3D computer aided design (CAD) repositories \citep{sangpil2020large}. MCB has 18,038 point cloud datasets with 25 classes. The size of each point cloud has a wide range from hundreds to millions of points. The quality of each point cloud also varies, as some point clouds do not present a clear shape. We extract point clouds from two classes: `Bearing' and `Motor.' We discarded datasets with extremely small and large numbers of points. After this pre-selection, we obtain 74 sets from Bearing and 53 sets from Motor. Each dataset consists of a range from $N = 30,000$ to $N = 250,000$ points. Figure \ref{fig:bearing-motor} shows some examples from two classes. 

\begin{figure}[htbp]
\centering
    \begin{subfigure}[t]{0.23\textwidth}
        \centering
        \includegraphics[width=\textwidth]{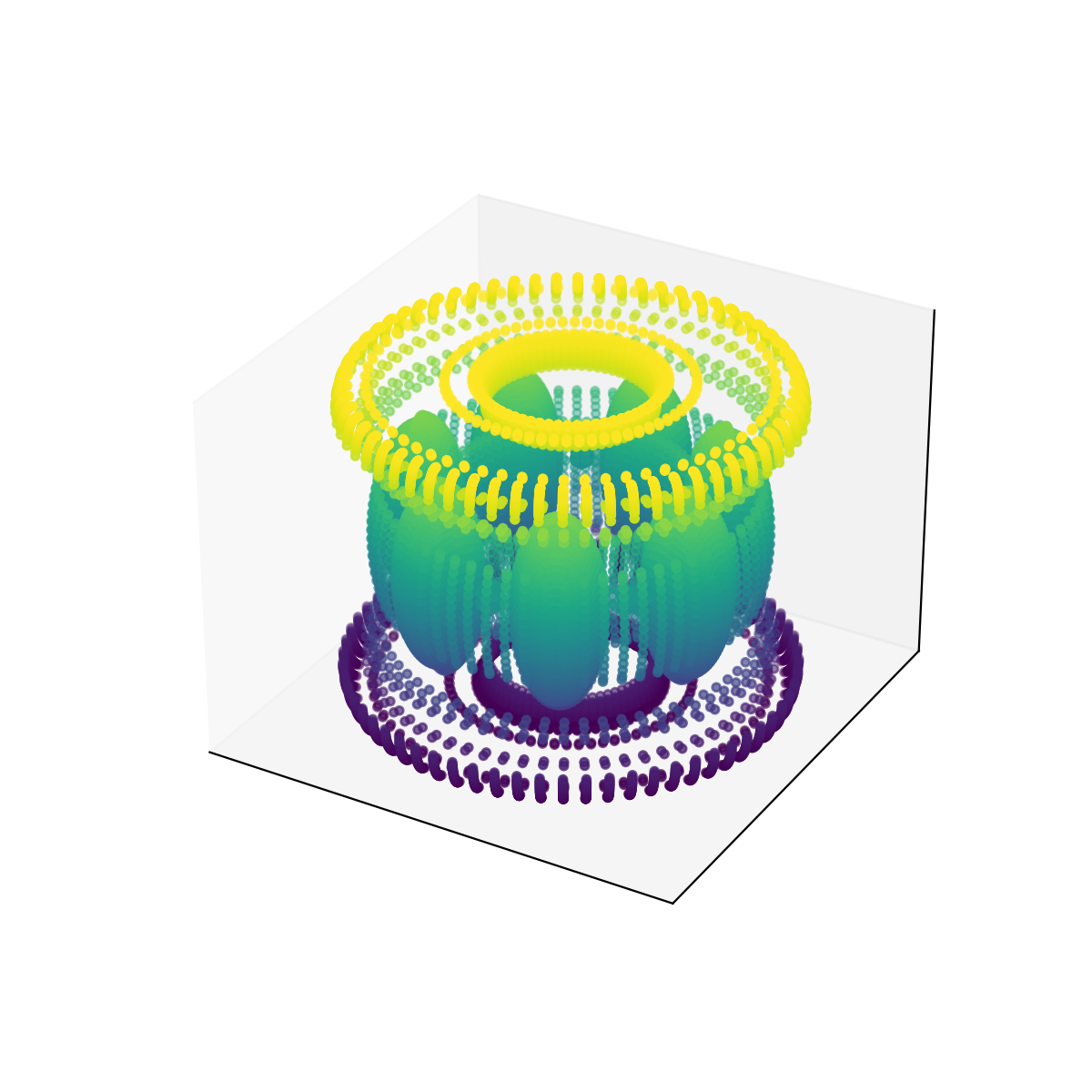}
    \end{subfigure}
        \begin{subfigure}[t]{0.23\textwidth}
        \centering
        \includegraphics[width=\textwidth]{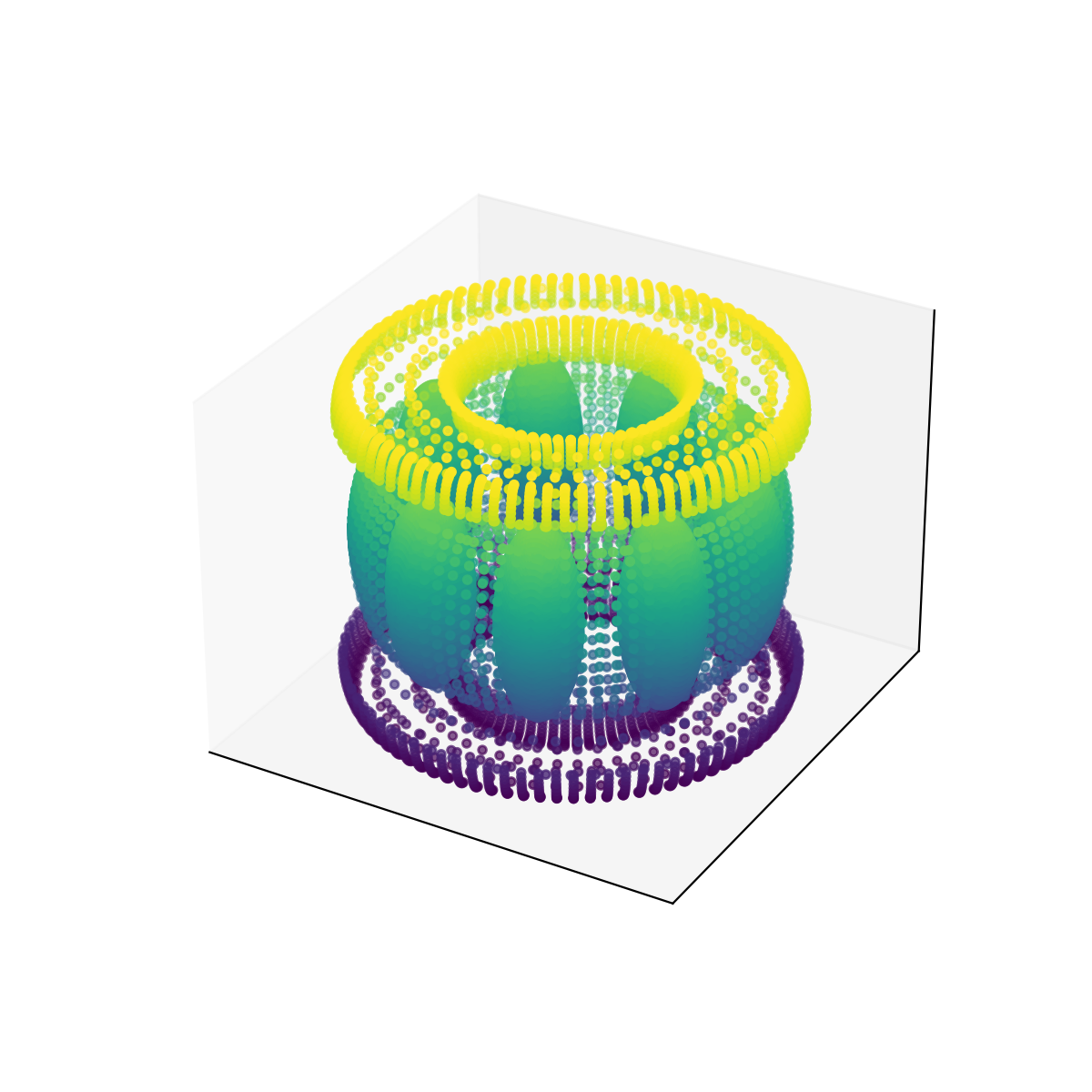}
    \end{subfigure}
        \begin{subfigure}[t]{0.23\textwidth}
        \centering
        \includegraphics[width=\textwidth]{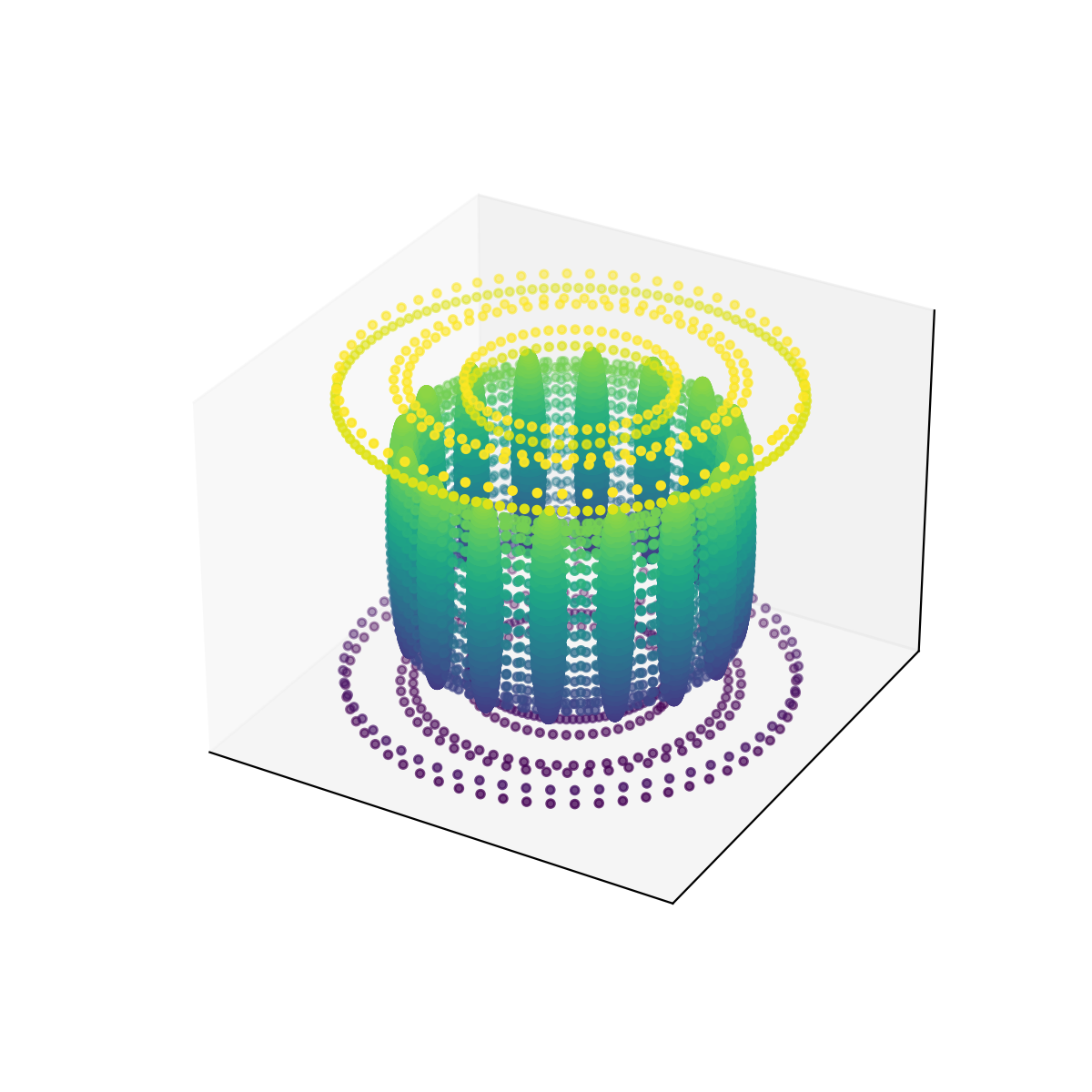}
    \end{subfigure}
        \begin{subfigure}[t]{0.23\textwidth}
        \centering
        \includegraphics[width=\textwidth]{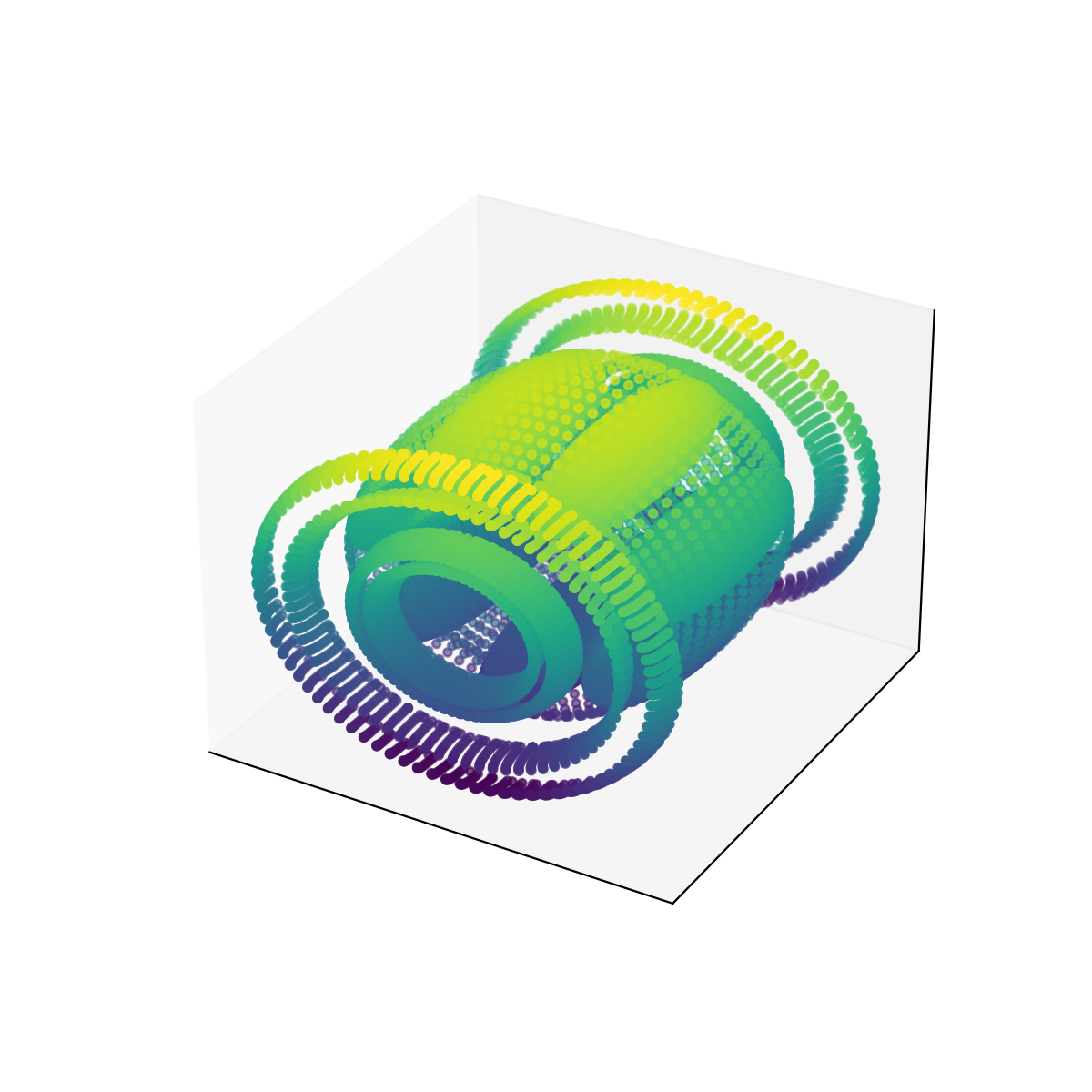}
    \end{subfigure}\\
    
        \begin{subfigure}[t]{0.23\textwidth}
        \centering
        \includegraphics[width=\textwidth]{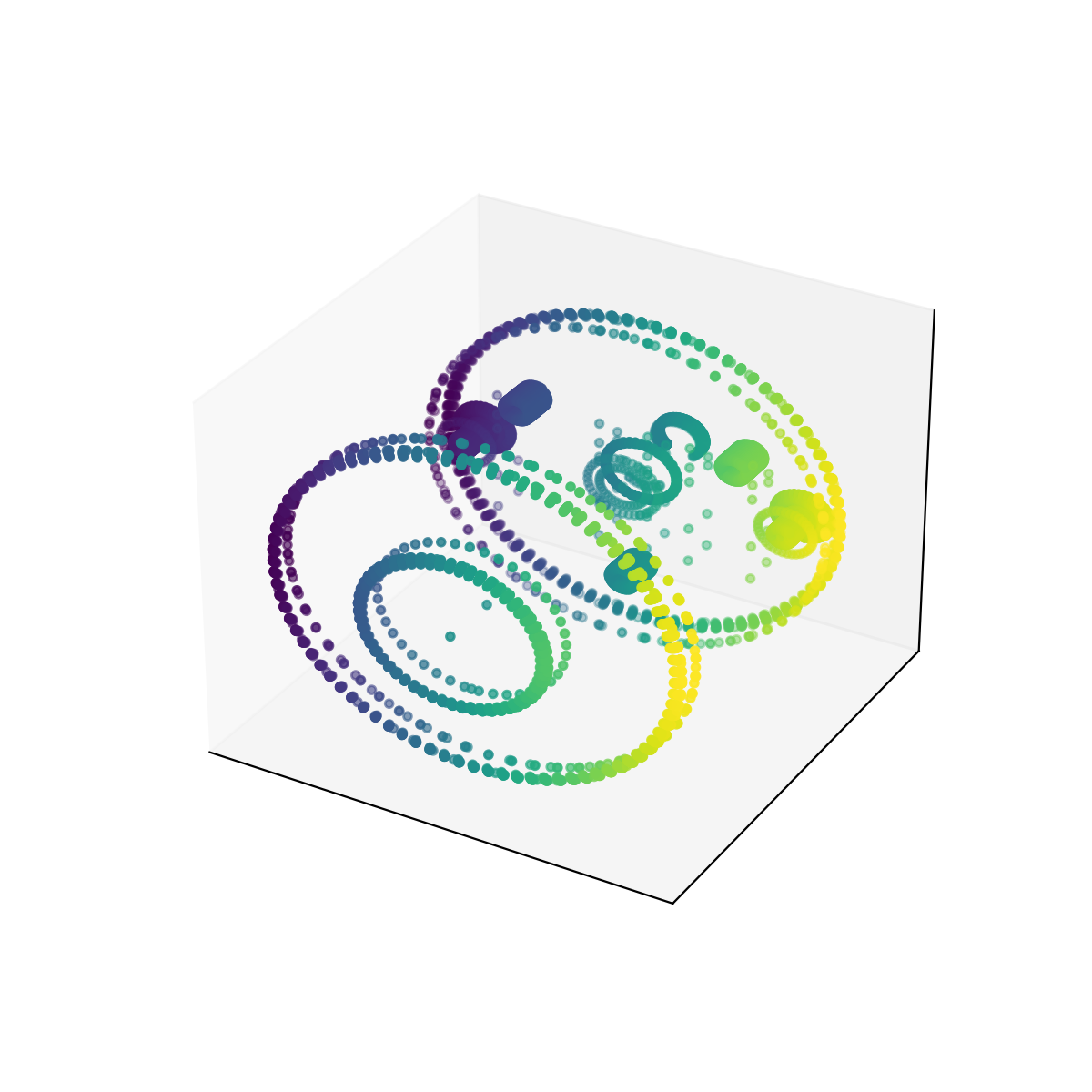}
    \end{subfigure}
    \begin{subfigure}[t]{0.23\textwidth}
        \centering
        \includegraphics[width=\textwidth]{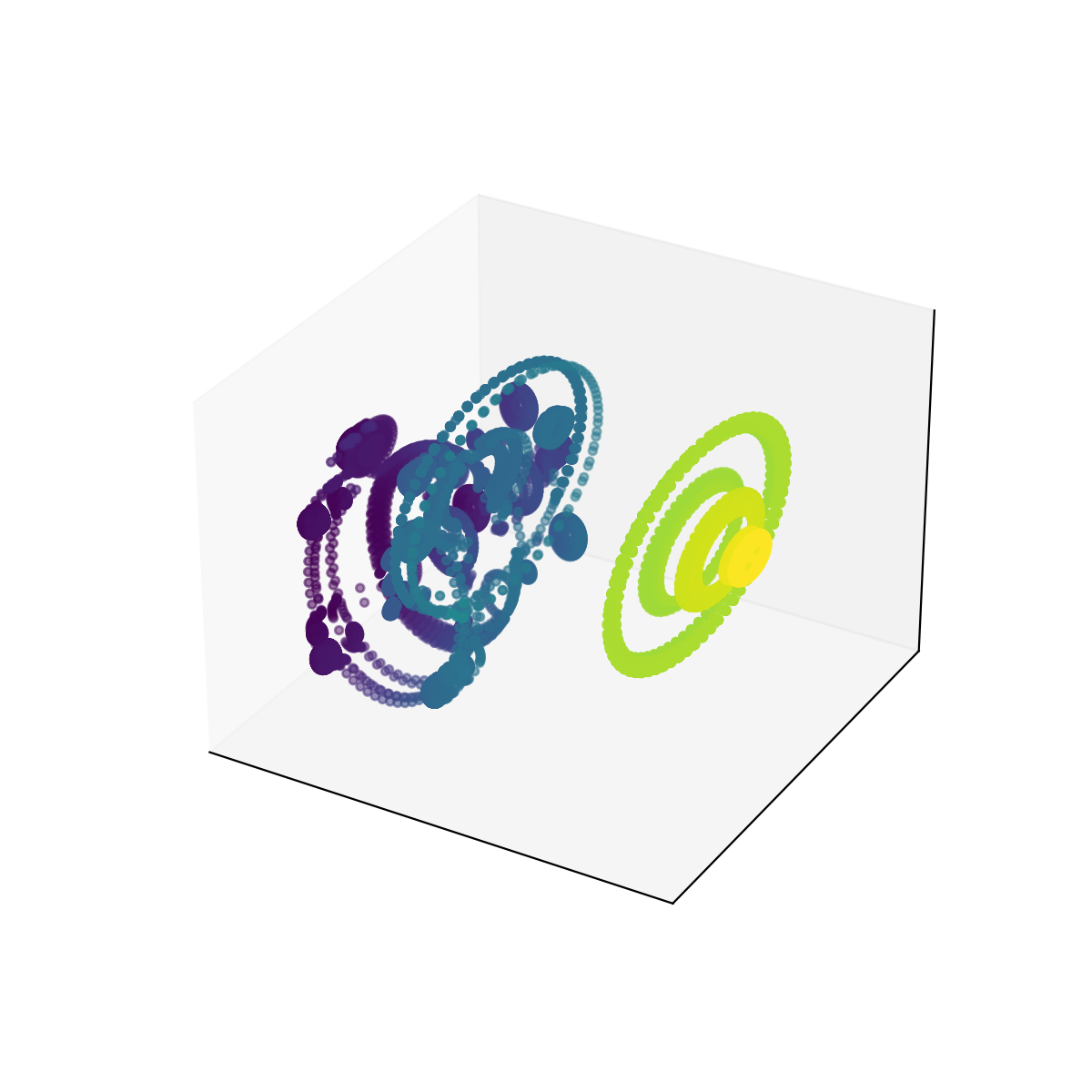}
    \end{subfigure}
    \begin{subfigure}[t]{0.23\textwidth}
        \centering
        \includegraphics[width=\textwidth]{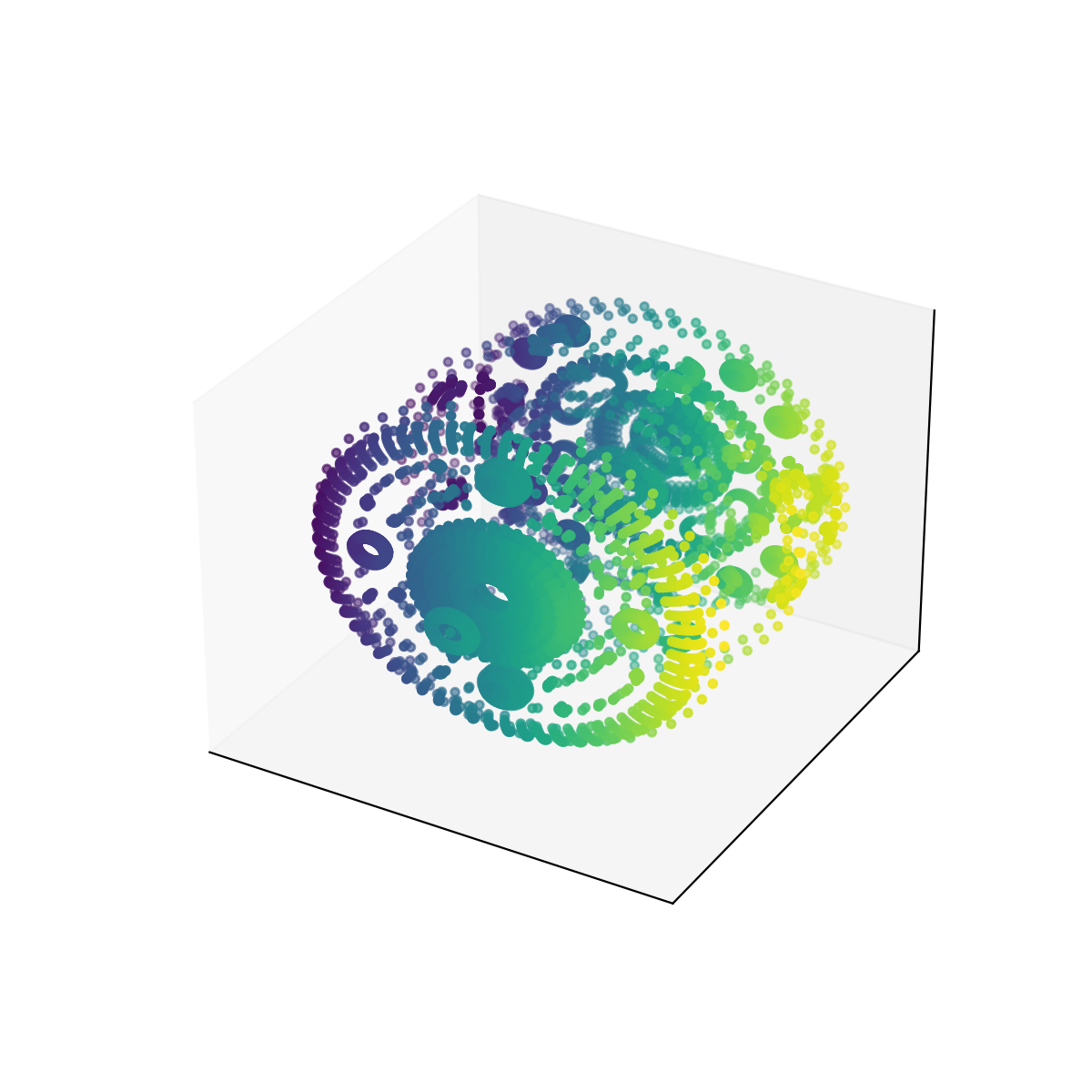}
    \end{subfigure}
    \begin{subfigure}[t]{0.23\textwidth}
        \centering
        \includegraphics[width=\textwidth]{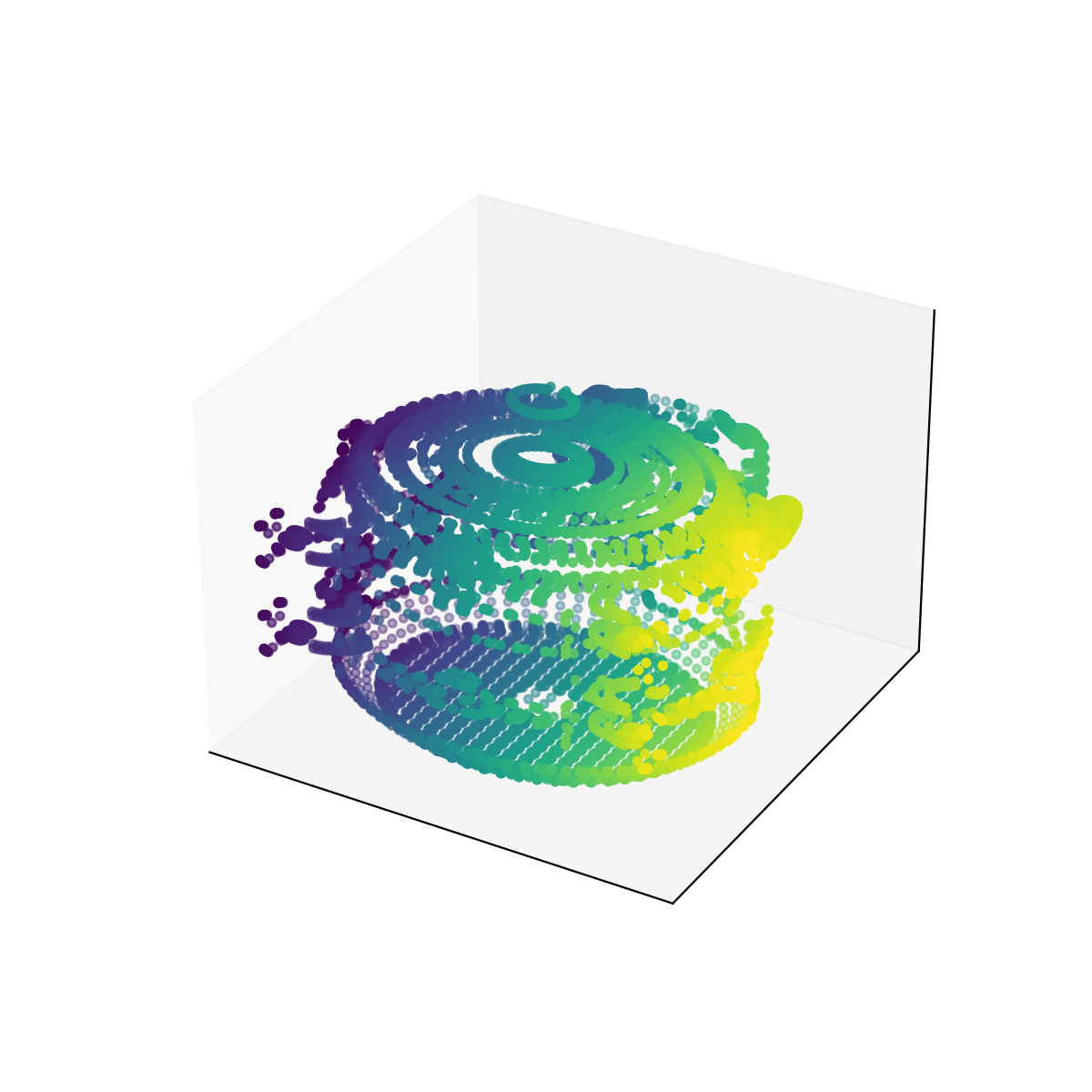}
    \end{subfigure}
    \caption{Some example point clouds from Bearing and Motor. The first row presents sets from Bearing and the second from Motor.}
    \label{fig:bearing-motor}
\end{figure}

Given the 127 point clouds, we would like to classify them into two clusters---representing Bearing and Motor, respectively---using persistent homology. Note that a direct computation of persistent homology is not feasible, as most point clouds are massive. Therefore, we approximate their persistent homology using mean persistence measures and Fr\'echet means computed from subsampled sets. For each point cloud, we subsample $B = 15$ sets, each consisting of 2\% number of points of the original data set. Then we compute the mean persistence measure and Fr\'echet mean as discussed above. In both computations, we compute the mutual $\ot_2$ distance, which is stored as distance matrix and then used as the input of UMAP for dimension reduction \citep{mcinnes2018umap-software}. The 127 point clouds are embedded into the 2D plane by UMAP. The results after dimension reduction are shown in Figure \ref{fig:dimension-reduction}, where points are colored according to their true labels. We then use DBSCAN to classify these points into two clusters \citep{ester1996density}. As shown in Figure \ref{fig:dimension-reduction}, the two clusters are close to the true labels.

\begin{figure}[htbp]
\centering
\begin{subfigure}[t]{0.23\textwidth}
        \centering
        \includegraphics[width=\textwidth]{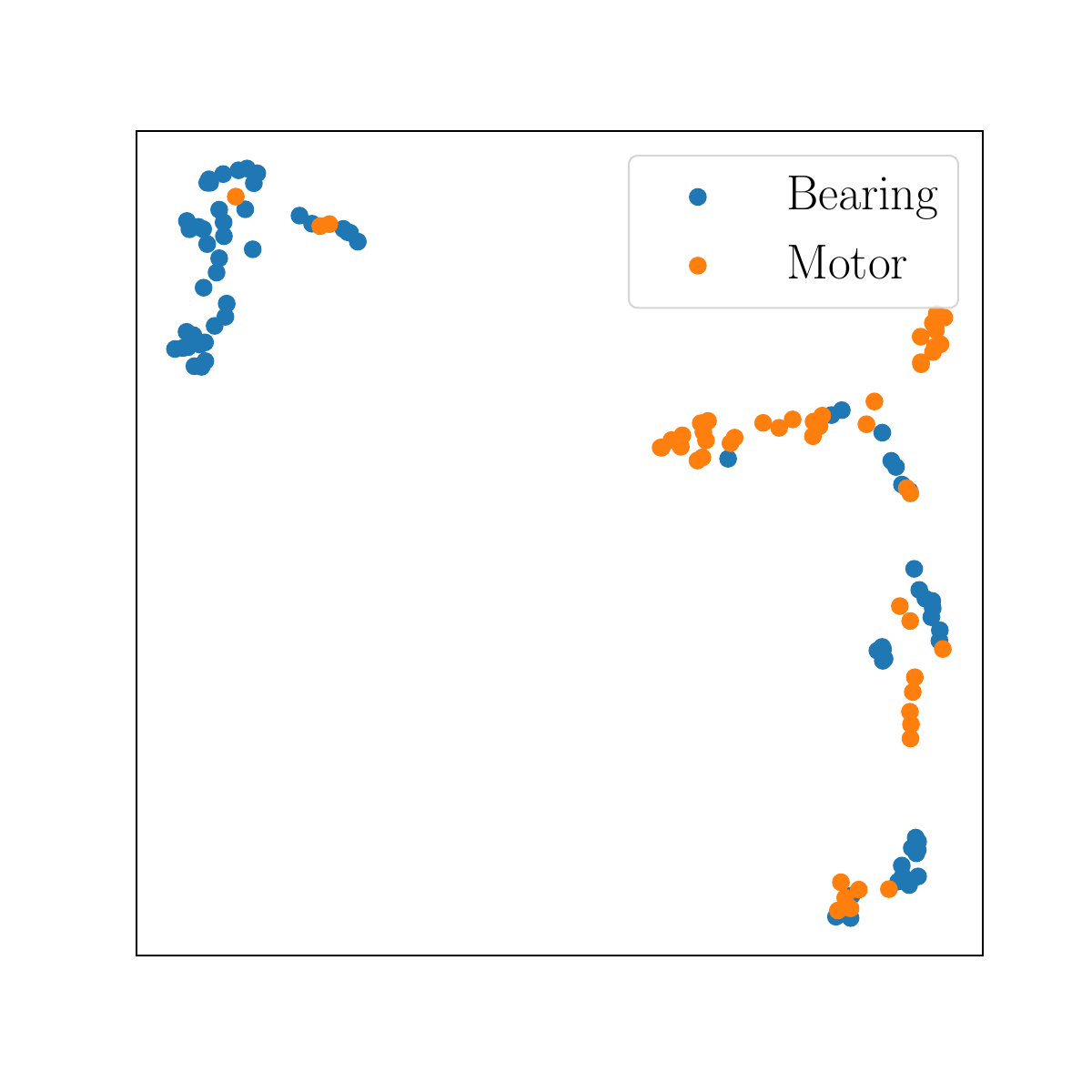}
        \caption{UMAP for mean persistence measures}
    \end{subfigure}
    \begin{subfigure}[t]{0.23\textwidth}
        \centering
        \includegraphics[width=\textwidth]{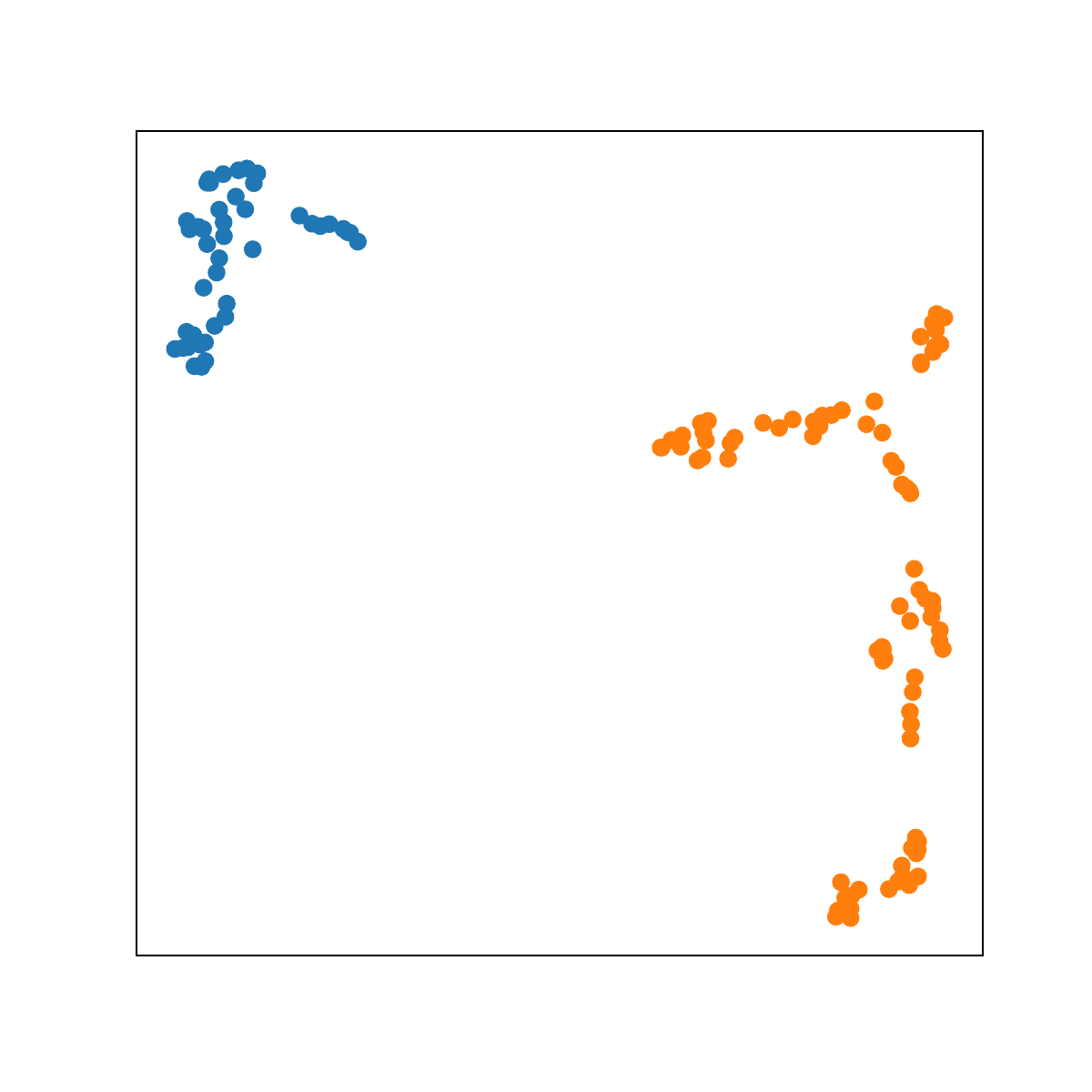}
        \caption{DBSCAN for mean persistence measures}
    \end{subfigure}
    \begin{subfigure}[t]{0.23\textwidth}
        \centering
        \includegraphics[width=\textwidth]{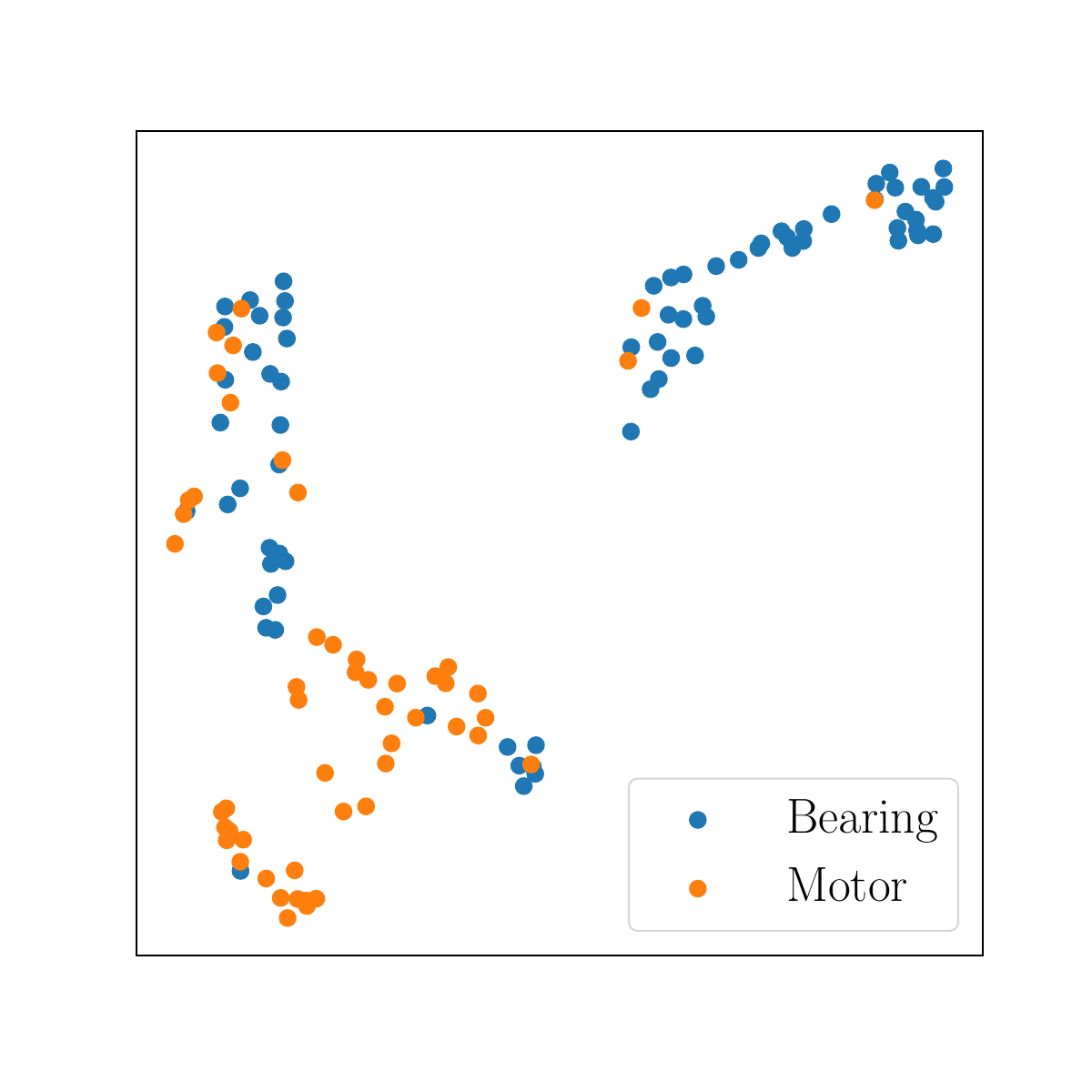}
        \caption{UMAP for Fr\'echet means}
    \end{subfigure}
    \begin{subfigure}[t]{0.23\textwidth}
        \centering
        \includegraphics[width=\textwidth]{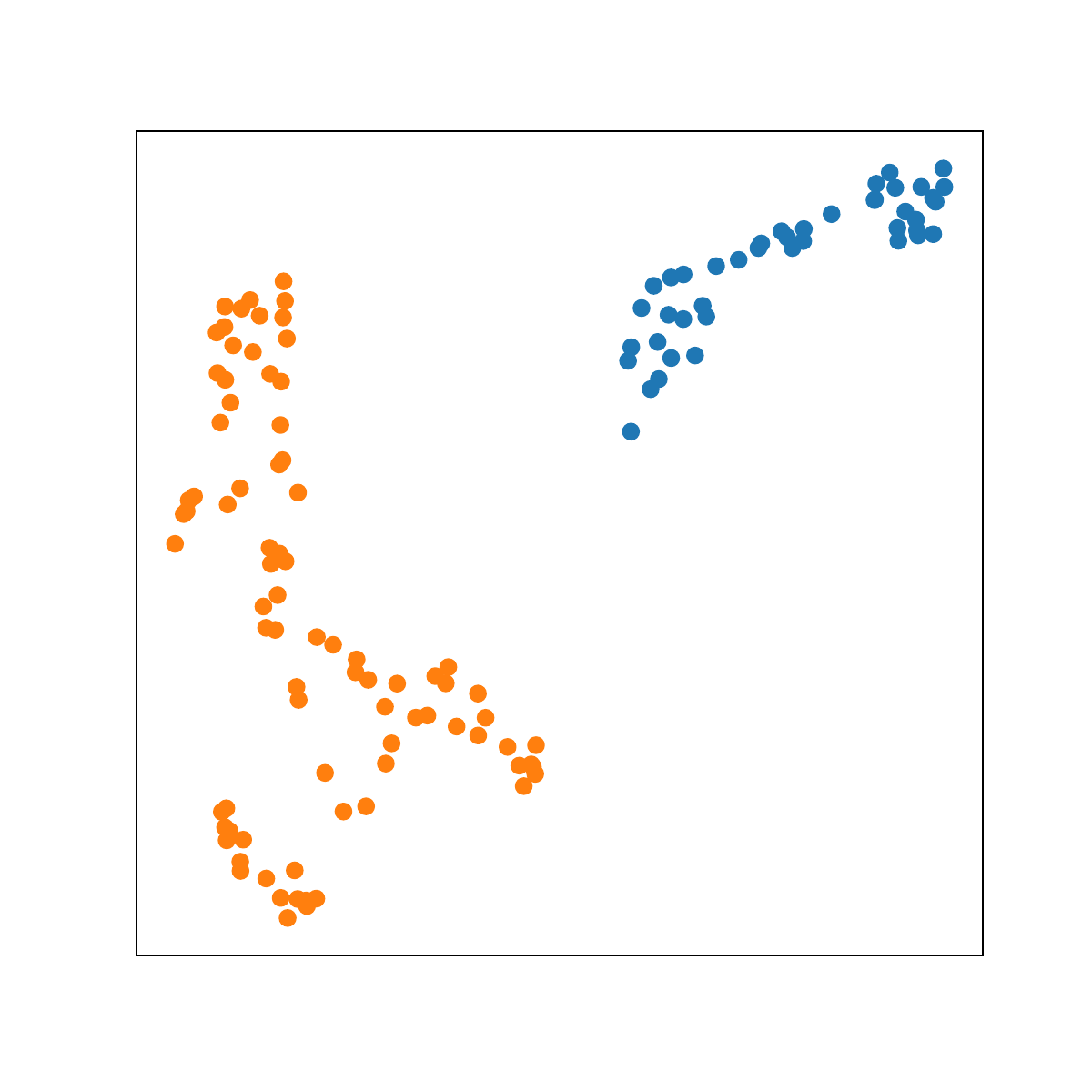}
        \caption{DBSCAN for Fr\'echet means}
    \end{subfigure}
    \caption{Dimension reduction using UMAP and clustering using DBSCAN.}
    \label{fig:dimension-reduction}
\end{figure}

\section{Data and Software Availability}

\paragraph{Data.}  The point cloud imaging datasets `Knot' and `Lock' were obtained from the shape repository Digital Shape Workbench (\url{http://visionair.ge.imati.cnr.it/}).  `Bearing' and `Motor' were obtained from the Mechanical Components Benchmark \citep{sangpil2020large}. The \texttt{Mammal} subtree of \texttt{WordNet} is obtained from \url{https://github.com/facebookresearch/poincare-embeddings}.

\paragraph{Software.} The computation of persistent homology is implemented using \texttt{GUDHI} and \texttt{Giotto-tda} \citep{gudhi:urm,tauzin2021giotto}. The computation of optimal transport is implemented using \texttt{POT} \citep{flamary2021pot}. The training for Poincar\'{e} embedding is implemented using \texttt{gensim} \citep{rehurek2011gensim}. The code for all numerical experiments in this paper can be found at \url{https://github.com/YueqiCao/PD-subsample}.

\section{Further Discussion and Directions for Future Research}

Our work inspires several future directions of research, which we now list.

\paragraph{Variance estimation for Fréchet means.} Currently there exists no convergence rate estimation of variance for Fréchet means. \cite{gouic-2019-fast} proposed a geometric condition for the fast convergence of empirical Fréchet means in Alexandrov spaces of nonnegative curvature, which may be a starting point to study the convergence and derive corresponding variance rates, for example, on $(\mathcal{D}_2,\mathrm{W}_2)$ which is known to be a nonnegative curved Alexandrov space \citep{turner-2014-frechet}.
    
\paragraph{Combining mean persistence measures and Fréchet means.} In the application real large point cloud data, we saw that mean persistence measure and Fréchet means both have merits and drawbacks, and in some sense, they compensate each other. An interesting possibility would be to construct a new representation of persistent homology that combines the advantages of both methods.
   
\paragraph{Subsampling for other data types.} In this work, we have derived theoretical results of subsampling methods under the assumption that the input data takes the form of Euclidean point clouds.  In practice subsampling methods may also be adapted to other data types. As we have seen experimentally, our proposed subsampling approach also performs well for general finite metric spaces and weighted graphs. An open direction of research is to derive similar convergence analysis results to what we have found and determine whether our results can be generalized to finite metric spaces and graphs. Additionally, from the practical perspective, it would also be worthwhile to explore other types of data such as images and explore the general applicability of subsampling and computing persistent homology in these settings.    
    
\paragraph{Applications to machine learning.} Persistent homology has been applied in many machine learning settings, including as graph neural networks \citep{zhao2020persistence}, graph classification \citep{carriere2020perslay}, and deep neural networks \citep{rieck2018neural}. An interesting direction of study would be to determine how to use subsampling methods may be used to enhance the training and construction of neural networks. 

\end{document}